\documentclass[11pt, reqno]{amsart}
\usepackage{amsfonts, amssymb, amscd, amsrefs}
\usepackage{graphicx}
\usepackage{hyperref}
\usepackage{slashed}
\usepackage{fullpage}
\usepackage{float}
\usepackage{enumitem}
\usepackage[table]{xcolor}
\usepackage{colortbl}
\usepackage{bbold}
\usepackage{placeins}
\usepackage{afterpage}

\usepackage{multirow}
\usepackage{booktabs}
\usepackage{tabularx}
\usepackage{longtable, array}

\usepackage{algorithm}
\usepackage{algpseudocode}

\usepackage{makecell}

\usepackage{amsmath, amssymb, amsthm}
\usepackage{tikz}
\usetikzlibrary{
  positioning,
  calc,
  arrows.meta,
  decorations.pathreplacing,
  shapes.geometric,
  fit
}

\tikzset{
arrow/.style={->, >=stealth, thick},
every picture/.style={ transform shape, scale=0.85, every node/.style={ scale=0.85, inner sep=2pt, outer sep=1pt } },
box/.style={ rectangle, draw, minimum width=3cm, minimum height=1cm, text centered, align=center },
every text node part/.style={ font=\sffamily } }

\usepackage{pgf}
\usepackage[pdf]{graphviz}
\usepackage{uniquecounter}

\definecolor{codebg}{RGB}{245, 245, 230} % very light yellow

\usepackage[useregional]{datetime2}
\DTMlangsetup[en-US]{zone=eastern,mapzone}

\theoremstyle{plain}
\newtheorem{theorem}{Theorem}[section]
\newtheorem{lemma}[theorem]{Lemma}
\newtheorem{proposition}[theorem]{Proposition}

\theoremstyle{definition}
\newtheorem{definition}[theorem]{Definition}

\theoremstyle{remark}
\newtheorem{remark}[theorem]{Remark}

\newcommand{\cmark}{\textcolor{green!60!black}{\ensuremath{\checkmark}}}
\newcommand{\xmark}{\textcolor{red!70!black}{\ensuremath{\times}}}

\newcommand{\Weak}{\textcolor{red!70!black}{Weak}}              % most red
\newcommand{\Limited}{\textcolor{red!50!orange}{Limited}}       % red-orange
\newcommand{\Partial}{\textcolor{orange!80!black}{Partial}}     % orange
\newcommand{\Moderate}{\textcolor{yellow!60!black}{Moderate}}   % yellow
\newcommand{\Strong}{\textcolor{green!60!black}{Strong}}        % green
\newcommand{\None}{\textcolor{gray}{None}}                      % neutral

\newcommand{\NaN}{\textnormal{\texttt{NaN}}}
\newcommand{\df}{\textnormal{\texttt{df}}}
\newcommand{\dtype}{\textnormal{\texttt{dtype}}}
\input{./helpers_root/dev_scripts_helpers/documentation/latex_abbrevs.sty}

\usepackage{listings}

\newcommand{\pythonstyle}{\lstset{ language=Python, basicstyle=\ttm, morekeywords={self}, % Add keywords here
keywordstyle=\ttb
\color{deepblue}
, emph={MyClass,__init__}, % Custom highlighting
emphstyle=\ttb
\color{deepred}
, % Custom highlighting style
stringstyle=
\color{deepgreen}
, frame=tb, % Any extra options here
showstringspaces=false }}

\lstnewenvironment{python}[1][]{ \pythonstyle \lstset{#1} }{}

\newcommand{\pythoninline}[1]{{\pythonstyle\lstinline!#1!}}

\usepackage{listings}
\usepackage{xcolor}

\definecolor{codegreen}{rgb}{0,0.6,0}
\definecolor{codegray}{rgb}{0.5,0.5,0.5}
\definecolor{codepurple}{rgb}{0.58,0,0.82}
\definecolor{backcolour}{rgb}{0.95,0.95,0.92}
\definecolor{redbackground}{rgb}{1,0.9,0.9}
\definecolor{greenbackground}{rgb}{0.9,1,0.9}

\lstdefinestyle{mystyle}{ backgroundcolor=
\color{backcolour}
, commentstyle=
\color{codegreen}
, keywordstyle=
\color{magenta}
, numberstyle=\tiny
\color{codegray}
, stringstyle=
\color{codepurple}
, basicstyle=\ttfamily\footnotesize, breakatwhitespace=false, breaklines=true, captionpos=b,
keepspaces=true, numbers=left, numbersep=5pt, showspaces=false, showstringspaces=false,
showtabs=false, tabsize=2 }

\lstdefinestyle{redstyle}{ backgroundcolor=
\color{redbackground}
, commentstyle=
\color{codegreen}
, keywordstyle=
\color{magenta}
, numberstyle=\tiny
\color{codegray}
, stringstyle=
\color{codepurple}
, basicstyle=\ttfamily\small, breakatwhitespace=false, breaklines=true, captionpos=b,
keepspaces=true, showspaces=false, showstringspaces=false,
showtabs=false, tabsize=2 }

\lstdefinestyle{greenstyle}{ backgroundcolor=
\color{greenbackground}
, commentstyle=
\color{codegreen}
, keywordstyle=
\color{magenta}
, numberstyle=\tiny
\color{codegray}
, stringstyle=
\color{codepurple}
, basicstyle=\ttfamily\small, breakatwhitespace=false, breaklines=true, captionpos=b,
keepspaces=true, showspaces=false, showstringspaces=false,
showtabs=false, tabsize=2 }

\lstdefinestyle{codebgstyle}{ backgroundcolor=
\color{codebg}
, commentstyle=
\color{codegreen}
, keywordstyle=
\color{magenta}
, numberstyle=\tiny
\color{codegray}
, stringstyle=
\color{codepurple}
, basicstyle=\ttfamily\small, breakatwhitespace=false, breaklines=true, captionpos=b,
keepspaces=true, numbers=left, numbersep=5pt, showspaces=false, showstringspaces=false,
showtabs=false, tabsize=2 }

\makeatletter
\newcommand{\printfnsymbol}[1]{%
\textsuperscript{\@fnsymbol{#1}}%
}
\makeatother

\makeatletter
\def\input@path{{papers/Causify_DataFlow_A_Causal_Simulator/}}
\makeatother

\begin{document}
  \title{DataFlow: a framework for high-performance machine learning stream computing}

  \author{Giacinto Paolo Saggese$^{*}$}
  \author{Paul Smith$^{*}$}
  \thanks{$^{*}$ Authors listed alphabetically.}

  \thanks{With contributions from Shayan Ghasemnezhad, Danil Iachmenev, Tamara
  Jordania, Sonaal Kant, Samarth KaPatel, Grigorii Pomazkin, Sameep Pote, Juraj Smeriga,
  Daniil Tikhomirov, Nina Trubacheva, and Vladimir Yakovenko, }

  \maketitle

  Version: \today\ \DTMcurrenttime\ UTC

  \setcounter{tocdepth}{1}
  \tableofcontents

  \section{Introduction}

DataFlow is a framework to build, test, and deploy high-performance streaming
computing systems based on machine learning and artificial intelligence.

The goal of DataFlow is to increase the productivity of data scientists by empowering
them to design and deploy systems with minimal or no intervention from data engineers
and devops support.

Guiding desiderata in the design of DataFlow include:
\begin{enumerate}
  \item Support rapid and flexible prototyping with the standard Python/Jupyter/data
    science tools

  \item Process both batch and streaming data in exactly the same way

  \item Avoid software rewrites in going from prototype to production

  \item Make it easy to replay stream events in testing and debugging

  \item Specify system parameters through config

  \item Scale gracefully to large data sets and dense compute
\end{enumerate}

These design principles are embodied in the many design features of DataFlow,
which include:
\begin{enumerate}
  \item \textbf{Computation as a direct acyclic graph}. DataFlow represents models
    as direct acyclic graphs (DAG), which is a natural form for dataflow and
    reactive models typically found in real-time systems. Procedural statements
    are also allowed inside nodes.

  \item \textbf{Time series processing}. All DataFlow components (such as
    data store, compute engine, deployment) handle time series processing in a
    native way. Each time series can be univariate or multivariate (e.g., panel
    data) represented in a data frame format.

  \item \textbf{Support for both batch and streaming modes}. The framework allows
    running a model both in batch and streaming mode, without any change in
    the model representation. The same compute graph can be executed feeding data
    in one shot or in chunks (as in historical/batch mode), or as data is
    generated (as in streaming mode). DataFlow guarantees that the model execution
    is the same independently on how data is fed, as long as the model is
    strictly causal. Testing frameworks are provided to compare batch/streaming
    results so that any causality issues may be detected early in the development
    process.

  \item \textbf{Precise handling of time}. All components automatically track the
    knowledge time of when the data is available at both their input and output.
    This allows one to easily catch future peeking bugs, where a system is non-causal
    and uses data available in the future.

  \item \textbf{Observability and debuggability}. Because of the ability to capture
    and replay the execution of any subset of nodes, it is possible to easily
    observe and debug the behavior of a complex system.

  \item \textbf{Tiling}. DataFlow's framework allows streaming data with different
    tiling styles (e.g., across time, across features, and both), to minimize
    the amount of working memory needed for a given computation, increasing the
    chances of caching computation.

  \item \textbf{Incremental computation and caching}. Because the dependencies
    between nodes are explicitly tracked by DataFlow, only nodes that see a
    change of inputs or in the implementation code need to be recomputed, while
    the redundant computation can be automatically cached.

  \item \textbf{Maximum parallelism}. Because the computation is expressed as a
    DAG, the DataFlow execution scheduler can extract the maximum amount of
    parallelism and execute multiple nodes in parallel in a distributed
    fashion, minimizing latency and maximizing throughput of a computation.

  \item \textbf{Automatic vectorization}. DataFlow DAG nodes can apply a computation
    to a cross-section of features relying on numpy and Pandas vectorization.

  \item \textbf{Support for train/prediction mode}. A DAG can be run in `fit' mode
    to learn some parameters, which are stored by the relevant DAG nodes, and then
    run in `predict` mode to use the learned parameters to make predictions.
    This mimics the Sklearn semantic. There is no limitation to the number of evaluation
    phases that can be created (e.g., train, validation, prediction, save state,
    load state). Many different learning styles are supported from different
    types of runners (e.g., in-sample-only, in-sample vs out-of-sample, rolling
    learning, cross-validation).

  \item \textbf{Model serialization}. A fit DAG can be serialized to disk and then
    materialized for prediction in production.

  \item \textbf{Configured through a hierarchical configuration}. Each parameter
    in a DataFlow system is controlled by a corresponding value in a configuration.
    In other words, the config space is homeomorphic with the space of DataFlow
    systems: each config corresponds to a unique DataFlow system, and vice versa,
    each DataFlow sytem is completely represented by a Config. A configuration
    is represented as a nested dictionary following the same structure of the DAG
    to make it easy to navigate its structure. This makes it easy to create an
    ensemble of DAGs sweeping a multitude of parameters to explore the design
    space.

  \item \textbf{Deployment and monitoring}. A DataFlow system can be deployed
    as a Docker container. Even the development system is run as a Docker
    container, supporting the development and testing of systems on both cloud
    (e.g., AWS) and local desktop. Airflow is used to schedule and monitor
    long-running DataFlow systems.

  \item \textbf{Python and Jupyter friendly}. The framework is completely implemented
    in high-performance Python. It supports natively `asyncio' to overlap
    computation and I/O. The same DAG can be run in a Jupyter notebook for research
    and experimentation or in a production script, without any change in code.

  \item \textbf{Python data science stack support}. Data science libraries (such
    as Pandas, numpy, scipy, sklearn) are supported natively to describe
    computation. The framework comes with a library of pre-built nodes for many
    ML/AI applications.
\end{enumerate}

% ================================================================================
\subsection{The fallacy of finite data in data science}

Traditional data science workflows operate under a fundamental assumption that
data is finite and complete. The canonical workflow reads a fixed dataset as
input, constructs a model through training procedures, and generates output
predictions. This paradigm conceptualizes data as a static, bounded entity of
known and finite size that can be loaded entirely into memory or processed in a
single pass.

This assumption represents a significant departure from real-world data
generation processes. Most production systems encounter \emph{unbounded} data
that arrives incrementally over time as a continuous stream. Time-series data
sources---including sensor measurements, financial transactions, and system
logs---generate observations continuously. Such data streams exhibit no natural
termination point and grow indefinitely over operational lifetimes.

The standard approach to reconciling unbounded data with finite-data processing
frameworks partitions the continuous data stream into discrete chunks
(typically referred to as ``batches'') and applies modeling and transformation
logic to each chunk independently. This discretization strategy, while
pragmatic, introduces several categories of technical challenges:

\begin{enumerate}
  \item \textbf{Batch boundary artifacts}. Discretizing continuous streams into
    finite batches creates artificial discontinuities at batch boundaries.
    Operations that require temporal context across these boundaries---such as
    rolling statistics, cumulative aggregations, or stateful
    transformations---may yield results that depend on the specific
    partitioning scheme employed. This dependence on batch structure can lead
    to discrepancies between batch-mode development results and streaming-mode
    production behavior, complicating model validation and deployment
    verification.

  \item \textbf{Causality violations}. Batch processing frameworks typically
    provide simultaneous access to all observations within a batch. This
    temporal indiscriminacy facilitates inadvertent causality violations, where
    computations at time $t$ access observations from time $t' > t$. Such
    violations manifest as data leakage during model development and result in
    systematic prediction failures when deployed in real-time environments
    where future observations are genuinely unavailable. Detection of these
    causality errors during development requires explicit validation
    infrastructure that many batch-oriented frameworks do not provide.

  \item \textbf{Development-production implementation divergence}. Batch
    processing frameworks often lack native support for streaming execution
    semantics. Consequently, models prototyped using batch-oriented tools
    (e.g., pandas, scikit-learn) frequently require substantial reimplementation
    for production deployment in streaming environments. This translation
    process introduces both development overhead and the risk of semantic
    discrepancies between research prototypes and production implementations,
    potentially leading to unexpected behavioral differences in deployed
    systems.

  \item \textbf{Limited reproducibility of production failures}. Production
    failures in real-time systems often depend on precise temporal ordering and
    timing of data arrivals. Batch-based development frameworks typically
    abstract away these temporal details, making faithful reproduction of
    production failures in development environments challenging. Without the
    ability to replay exact event sequences with accurate timing semantics,
    systematic debugging of production issues becomes significantly more
    difficult.
\end{enumerate}

DataFlow addresses these challenges through a unified computational model
that treats data as unbounded streams and enforces strict causality via
point-in-time idempotency guarantees (see Section~\ref{subsec:context-windows}).
This design unifies batch and streaming execution semantics within a single
framework, enabling models developed in batch mode to execute identically in
streaming production environments without code modification.

% ================================================================================
\subsection{Definition of stream computing}

In the computer science literature, several terms (such as event/data stream processing,
graph computing, dataflow computing, reactive computing) are used to describe what
in this paper we refer succintly to as ``stream computing".

By stream processing we refer to a programming paradigm where streams of data
(e.g., time series or dataframes) are the objects of computation. A series of operations
(aka kernel functions) are applied to each element in the stream.

Stream computing represents a paradigm shift from traditional batch processing
and imperative languages, emphasizing real-time data handling, adaptability,
and parallel processing, making it highly effective for modern data-intensive applications.

% ================================================================================
\subsection{Core principles of stream computing}

The core principles of stream computing are:

\begin{enumerate}
  \item \textbf{Node-based architecture}. In stream and dataflow programming,
    the code is structured as a network of nodes. Each node represents a computational
    operation or a data processing function. Nodes are connected by edges that
    represent data streams.

  \item \textbf{Data-driven and reactive execution}. Execution in dataflow languages
    is data-driven, meaning that a node will process data as soon as it
    becomes available. Unlike imperative languages where the order of operations
    is predefined, in dataflow languages, the flow of data determines the order
    of execution.

  \item \textbf{Automatic parallelism}. The Dataflow programming paradigm naturally
    lends itself to parallel execution. Since nodes operate independently, they
    can process different data elements simultaneously, exploiting concurrent
    processing capabilities of modern hardware.

  \item \textbf{Continuous data streams}. Streams represent a continuous flow of
    data rather than discrete batches. Nodes in the network continuously
    receive, process, and output data, making them ideal for real-time data processing.

  \item \textbf{State management}. Nodes can be stateful or stateless.
    Stateful nodes retain information about previously processed data, enabling
    complex operations like windowing, aggregation, or pattern detection over time.

  \item \textbf{Compute intensity}. Stream processing is characterized by a high
    number of arithmetic operations per I/O and memory reference (e.g., it can
    be 50:1), since the same kernel is applied to all records and a number of
    records can be processed simultaneously. Furthermore, data is produced
    once and read only a few times.

  \item \textbf{Dynamic adaptability}. The dataflow model can dynamically adapt
    to changes in the data stream (like fluctuations in volume or velocity),
    ensuring efficient processing under varying conditions.

  \item \textbf{Scalability}. The model scales well horizontally, meaning one
    can add more nodes (or resources) to handle increased data loads without
    major architectural changes.

  \item \textbf{Event-driven processing}. Many dataflow languages support event-driven
    models where specific events in the data stream can trigger particular computational
    pathways or nodes.
\end{enumerate}

% ================================================================================
\subsection{Applications of stream computing}

Stream computing is a natural solution in a wide range of industries and
scenarios. Table~\ref{table:stream_applications} presents representative
application domains and their corresponding use cases.

\begin{table}[H]
\caption{Applications of stream computing across industries}
\centering
\small
\begin{tabular}{|l|p{9cm}|}
\hline
\textbf{Application} & \textbf{Examples} \\
\hline
Generic machine learning &
  \begin{itemize}[leftmargin=*, nosep, after=\vspace{-\baselineskip}]
    \item Graph-based computations in modern deep learning architectures
  \end{itemize} \\
\hline
Financial Services &
  \begin{itemize}[leftmargin=*, nosep, after=\vspace{-\baselineskip}]
    \item Trading: Analyzing market data in real-time to make automated trading decisions
    \item Fraud Detection: Monitoring transactions as they occur to detect and prevent fraudulent activities
    \item Risk Management: Real-time assessment of financial risks based on current market conditions and ongoing transactions
  \end{itemize} \\
\hline
Internet of Things (IoT) &
  \begin{itemize}[leftmargin=*, nosep, after=\vspace{-\baselineskip}]
    \item Smart Homes: Processing data from various home devices for automation and monitoring
    \item Industrial IoT: Real-time monitoring and control of industrial equipment and processes
    \item Smart Cities: Integrating data from traffic, public services, and environmental sensors to optimize urban management
  \end{itemize} \\
\hline
Telecommunications &
  \begin{itemize}[leftmargin=*, nosep, after=\vspace{-\baselineskip}]
    \item Network Monitoring and Optimization: Analyzing traffic patterns to optimize network performance and detect anomalies
    \item Customer Experience Management: Real-time analysis of customer data to improve service and personalize offerings
  \end{itemize} \\
\hline
Healthcare &
  \begin{itemize}[leftmargin=*, nosep, after=\vspace{-\baselineskip}]
    \item Remote Patient Monitoring: Continuous monitoring of patient vitals for timely medical intervention
    \item Real-Time Health Data Analysis: Analyzing data streams from medical devices for immediate clinical insights
  \end{itemize} \\
\hline
Retail and E-Commerce &
  \begin{itemize}[leftmargin=*, nosep, after=\vspace{-\baselineskip}]
    \item Personalized Recommendations: Real-time analysis of customer behavior to offer personalized product recommendations
    \item Supply Chain Optimization: Streamlining logistics and inventory management based on real-time data
  \end{itemize} \\
\hline
Media and Entertainment &
  \begin{itemize}[leftmargin=*, nosep, after=\vspace{-\baselineskip}]
    \item Content Optimization: Real-time analysis of viewer preferences and behavior for content recommendations
    \item Live Event Analytics: Monitoring and analyzing data from live events for audience engagement and operational efficiency
  \end{itemize} \\
\hline
Transportation and Logistics &
  \begin{itemize}[leftmargin=*, nosep, after=\vspace{-\baselineskip}]
    \item Fleet Management: Tracking and managing vehicles in real time for optimal routing and scheduling
    \item Predictive Maintenance: Analyzing data from transportation systems to predict and prevent equipment failures
  \end{itemize} \\
\hline
Energy and Utilities &
  \begin{itemize}[leftmargin=*, nosep, after=\vspace{-\baselineskip}]
    \item Smart Grid Management: Balancing supply and demand in real-time and identifying grid anomalies
    \item Renewable Energy Optimization: Optimizing the output of renewable energy sources by analyzing environmental data streams
  \end{itemize} \\
\hline
Cybersecurity &
  \begin{itemize}[leftmargin=*, nosep, after=\vspace{-\baselineskip}]
    \item Intrusion Detection Systems: Real-time monitoring of network traffic to detect and respond to cyber threats
    \item Threat Intelligence: Analyzing global cyber threat data streams for proactive security measures
  \end{itemize} \\
\hline
Environmental Monitoring &
  \begin{itemize}[leftmargin=*, nosep, after=\vspace{-\baselineskip}]
    \item Climate and Weather Analysis: Processing data from environmental sensors for weather prediction and climate research
    \item Pollution Monitoring: Real-time tracking of air and water quality
  \end{itemize} \\
\hline
Gaming &
  \begin{itemize}[leftmargin=*, nosep, after=\vspace{-\baselineskip}]
    \item In-Game Analytics: Real-time analysis of player behavior for game optimization and personalized experiences
  \end{itemize} \\
\hline
Social Media Analytics &
  \begin{itemize}[leftmargin=*, nosep, after=\vspace{-\baselineskip}]
    \item Trend Analysis: Monitoring social media streams to identify and analyze trending topics and sentiments
  \end{itemize} \\
\hline
Emergency Response &
  \begin{itemize}[leftmargin=*, nosep, after=\vspace{-\baselineskip}]
    \item Disaster Monitoring and Management: Real-time data analysis for effective response during natural or man-made disasters
  \end{itemize} \\
\hline
\end{tabular}
\label{table:stream_applications}
\end{table}

  % ################################################################################
\section{DataFlow at a Glance}

DataFlow is a computational framework designed to build, test, and deploy
artificial intelligence and machine learning models for \emph{unbounded}
tabular time series data. The framework addresses the unique challenges of
processing streaming data where temporal ordering, causality, and reproducibility
are essential requirements. DataFlow is particularly suited to domains
including financial markets, IoT sensor networks, supply chain monitoring,
and real-time analytics systems where time dependencies constitute a fundamental
aspect of the modeling problem.

% ================================================================================
\subsection{Core Abstractions}

The central data model in DataFlow is the \emph{stream dataframe}, a
time-indexed tabular structure with a potentially unbounded number of rows.
Unlike traditional batch-oriented dataframes, stream dataframes model data as
continuous sequences that arrive incrementally over time, reflecting the
reality of production systems where data generation is ongoing and
termination is not inherent to the process.

Computation in DataFlow is organized as a directed acyclic graph (DAG) of
\emph{computational nodes}. Each node receives a fixed number (possibly zero)
of stream dataframes as input and produces a fixed number of stream dataframes
as output. Nodes encapsulate transformations ranging from simple arithmetic
operations to complex machine learning models. The DAG structure makes data
dependencies explicit and enables the framework to exploit parallelism and
perform incremental computation.

% ================================================================================
\subsection{Tiling and Point-in-Time Idempotency}

A distinguishing principle of DataFlow is the requirement that computations
be \emph{tilable}---that is, their correctness must not depend on how input
data is partitioned along temporal or feature dimensions. Formally, a
computation is tilable if it can operate on contiguous windows (tiles) of the
input stream and produce outputs that are invariant to the tile boundaries,
provided that each tile contains sufficient temporal context.

This invariance is enforced through \emph{point-in-time idempotency}: for any
time point $t$, the output at $t$ depends only on a fixed-length temporal
window preceding $t$, known as the \emph{context window}. Once a tile includes
at least this context window length of historical data, the computed output
at the tile's right endpoint is guaranteed to match the output that would be
obtained from processing the entire historical stream up to that point.

Point-in-time idempotency provides a formal correctness criterion that unifies
batch and streaming execution. A model developed and validated on historical
data in batch mode can be deployed in streaming mode with the assurance that,
as long as causality is preserved, the outputs will be identical to those
produced in batch evaluation. This property is enforced by the framework and
can be automatically validated through tiling tests that partition data in
multiple ways and verify output consistency.

% ================================================================================
\subsection{Architectural Benefits}

The tiling principle and point-in-time idempotency requirement enable several
critical capabilities:

\begin{enumerate}
  \item \textbf{Unified batch and streaming execution.} The same model
    specification executes identically in batch mode (processing historical
    data in large chunks) and streaming mode (processing data as it arrives),
    eliminating the need for separate implementations and reducing
    development-production discrepancies.

  \item \textbf{Causality enforcement.} Because outputs at time $t$ depend only
    on data from times $s \leq t$, the framework prevents future-peeking
    errors that commonly arise in batch-oriented development. Violations of
    causality manifest as tile-dependent outputs and are automatically
    detectable through tiling validation tests.

  \item \textbf{Memory-bounded computation.} Tiles may be sized to fit within
    available memory resources while maintaining correctness, enabling
    processing of arbitrarily large historical datasets without requiring
    full in-memory representation.

  \item \textbf{Replay and debugging.} Real-time execution can be captured and
    replayed deterministically in development environments. The framework
    records input streams with precise knowledge timestamps, enabling
    systematic debugging of production failures by replaying exact sequences
    of events.

  \item \textbf{Parallelism and incremental computation.} The DAG structure
    makes data dependencies explicit. Nodes with no dependency relationship
    may execute in parallel, and nodes whose inputs and implementations have
    not changed may skip recomputation by retrieving cached results. The
    framework automatically exploits these opportunities for optimization.

  \item \textbf{Vectorization across features.} Stream dataframes support
    multi-level column indexing, enabling computations to be expressed in
    vectorized form across assets or features. This design leverages the
    computational efficiency of array-oriented libraries (e.g., NumPy, pandas)
    while maintaining compatibility with streaming semantics.

  \item \textbf{Model lifecycle management.} Nodes may be stateful, supporting
    distinct \texttt{fit} and \texttt{predict} phases analogous to
    scikit-learn conventions. Trained models may be serialized and deployed to
    production environments with assurance that evaluation semantics remain
    consistent between research and deployment.
\end{enumerate}

% ================================================================================
\subsection{Context Windows and Tile Composition}

A \emph{tile} is a contiguous temporal window of a stream dataframe whose
length equals or exceeds the node's context window requirement. For a node
with context window length $L$, any tile of length $\tau \geq L$ ending at
time $t$ contains sufficient information to compute the correct output at time
$t$. The context window of a DAG is determined by the context windows of its
constituent nodes and the dependencies encoded in the graph structure; the
maximal context window over all source-to-sink paths defines the DAG-level
context window.

An important consequence of point-in-time idempotency is that \emph{two tiles
suffice for correct mini-batch execution}: providing two consecutive tiles of
input data, each of length $\tau$, ensures that all outputs in the second
tile are computed correctly, even for time points near the tile boundary where
the context window spans both tiles. This property (formalized in
Theorem~\ref{thm:two-tiles-suffice}) underpins DataFlow's ability to process
historical data efficiently in batch mode while guaranteeing identical results
to streaming execution.

The framework's tile-based execution model, combined with explicit dependency
tracking and causality enforcement, addresses the fundamental tension between
the finite-data assumptions prevalent in data science tooling and the
unbounded-data reality of production systems. By treating data as unbounded
streams from the outset and imposing strict correctness criteria, DataFlow
enables models to be developed, validated, and deployed within a unified
computational model.

  % TODO: Maybe this goes in the Introduction?

% #############################################################################
\section{Challenges in time series machine learning}

Conducting machine learning on streaming time series data introduces additional
challenges beyond those encountered with machine learning on static data. These
challenges include overfitting, feature engineering, model evaluation, data
pipeline engineering. These issues are compounded by the dynamic nature of
streaming data.

In the following we list several problems in time-series machine learning
and how \emph{DataFlow} solves these problems.

% =============================================================================
\subsection{Prototype vs Production}

Data scientists typically operate under the assumption that all data is readily
available in a well-organized data frame format. Consequently, they often
develop a prototype based on assumptions about the temporal alignment of the
data. This prototype is then transformed into a production system by rewriting
the model in a more sophisticated and precise framework. This process may
involve using different programming languages or even having different teams
handle the translation. However, this approach can lead to significant issues:
  \begin{itemize}
  \item Converting the prototype in production requires time and effort
  \item The translation process may reveal bugs in the prototype.
  \item Assumptions made during the prototype phase might not align with real-world
    conditions.
  \item Discrepancies between the two model can result in additional work to
    implement and maintain two separate models for comparison.
  \end{itemize}

DataFlow addresses this issue by modeling systems as directed acyclic graphs
(DAGs), which naturally align with the dataflow and reactive models commonly
used in real-time systems. Each node within the graph consists of procedural
statements, similar to how a data scientist would design a non-streaming
system.

DataFlow enables the execution of a model in both batch and streaming modes
without requiring any modifications to the model code. In batch Mode, the graph
can be executed by processing data all at once or in segments, suitable for
historical or batch processing. In streaming mode, the graph can also be
executed as data is presented to the model, supporting real-time data processing.

% =============================================================================
\subsection{Frequency of Model Operation}

The required frequency of a model's operation often becomes clear only after
deployment. Adjustments may be necessary to balance time and memory
requirements with latency and throughput, which require changing the production
system implementation, with further waste of engineering effort.

DataFlow enables the same model description to operate at various frequencies
by simply adjusting the size of the computation tile. This flexibility
eliminates the need for any model modifications and allows models to be
always run at the optimal frequency requested by the application.

% =============================================================================
\subsection{Non-stationarity time series}

While the assumption of stationarity is useful, it typically only strictly
holds in theoretical fields such as physics and mathematics. In practical,
real-world applications, this assumption is rarely valid. Data scientists often
refer to data drift as an anomaly to explain poor performance on out-of-sample
data. However, in reality, data drift is the standard rather than the
exception.

All DataFlow components, including the data store, compute engine, and
deployment, are designed to natively handle time series processing. Each time
series can be either univariate or multivariate (such as panel data) and is
represented in a data frame format. DataFlow addresses non-stationarity by
enabling models to learn and predict continuously over time. This is achieved
using a specified look-back period or a weighting scheme for samples. These
parameters are treated as hyperparameters of the system, which can be tuned
like any other hyperparameters.

% =============================================================================
\subsection{Non-causal Bugs}

A common and challenging problem occurs when data scientists make incorrect
assumptions about data timestamps. This issue is often called "future peeking"
because the model inadvertently uses future information. A model is developed,
validated, and fine-tuned based on these incorrect assumptions, which are only
identified as errors after the system is deployed in production. This happens
because the data scientist lacks an early, independent evaluation to identify
the presence of non-causality.

Figure~\ref{fig:future_peeking} illustrates a concrete example of this bug
pattern.

\begin{figure}[ht]
\centering

\textbf{Code Comparison:}

\vspace{0.2cm}

\begin{minipage}{0.48\textwidth}
\textbf{WRONG: Future peeking}
\begin{lstlisting}[language=Python, style=redstyle]
# Uses tomorrow's price!
df['signal'] = df['price'].shift(-1)

# At time t, this accesses
# price at time t+1
\end{lstlisting}
\end{minipage}
\hfill
\begin{minipage}{0.48\textwidth}
\textbf{CORRECT: Causal computation}
\begin{lstlisting}[language=Python, style=greenstyle]
# Uses yesterday's price
df['signal'] = df['price'].shift(1)

# At time t, this accesses
# price at time t-1
\end{lstlisting}
\end{minipage}

\vspace{0.5cm}

\textbf{Timeline Visualization:}

\vspace{0.2cm}

\begin{tikzpicture}[scale=1.1]
  % Time axis
  \draw[thick,->] (0,0) -- (12,0) node[right] {\textbf{Time}};

  % Time points
  \foreach \i/\t in {2/$t_{-2}$, 4/$t_{-1}$, 6/$t_0$, 8/$t_1$, 10/$t_2$} {
    \draw[thick] (\i, -0.1) -- (\i, 0.1);
    \node[below] at (\i, -0.1) {\small \t};
  }

  % Knowledge time (available data)
  \fill[green!20] (0,0.5) rectangle (6,1.5);
  \node[align=center] at (3, 1) {\textbf{Knowledge Time} \\ Data available at $t_0$};

  % Event time (future data)
  \fill[red!20] (6,0.5) rectangle (11,1.5);
  \node[align=center] at (8.5, 1) {\textbf{Event Time (Future)} \\ Data not yet available};

  % Current decision point
  \draw[ultra thick, blue] (6, 0) -- (6, 2.5);
  \node[above, blue, align=center] at (6, 2.5) {\textbf{Decision Point} \\ (Now: $t_0$)};

  % Bug indicator
  \draw[->, ultra thick, red, dashed] (6, 1.8) -- (8, 1.8);
  \node[above, red, align=center] at (7, 2.1) {\textbf{Bug!} \\ Accessing $t_1$};

  % Correct access indicator
  \draw[->, ultra thick, green!60!black, dashed] (6, 0.3) -- (4, 0.3);
  \node[below, green!60!black, align=center] at (5, -0.4) {\textbf{Correct:} Access $t_{-1}$};

\end{tikzpicture}

\caption{Future-peeking bug example. The incorrect code uses \texttt{shift(-1)}, which accesses data from the future (time $t+1$ when making decision at time $t$). This violates causality because the future price is not available at decision time. The correct approach uses \texttt{shift(1)} to access historical data. In backtesting, both approaches may appear to work, but only the causal version is valid for production deployment.}
\label{fig:future_peeking}
\end{figure}
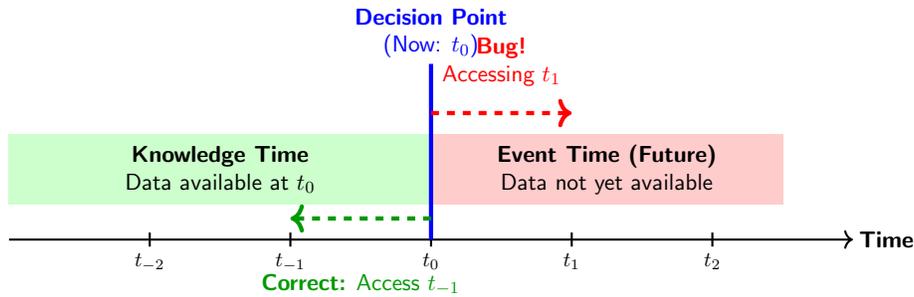

DataFlow offers precise time management. Each component automatically
monitors the time at which data becomes available at both its input and output.
This feature helps in easily identifying future peeking bugs, where a system
improperly uses data from the future, violating causality.
DataFlow ensures consistent model execution regardless of how data is fed,
provided the model adheres to strict causality. Testing frameworks are
available to compare batch and streaming results, enabling early detection of
any causality issues during the development process.

% =============================================================================
\subsection{Accurate Historical Simulation}

Implementing an accurate simulation of a system that processes time-series data
to evaluate its performance can be quite challenging. Ideally, the simulation
should replicate the exact setup that the system will use in production. For
example, to compute the profit-and-loss curve of a trading model based on
historical data, trades should be computed using only the information available
at those moments and should be simulated at the precise times they would have
occurred.

However, data scientists often create their own simplified simulation
frameworks alongside their prototypes. The various learning, validation, and
testing styles (such as in-sample-only and cross-validation) combined with
walkthrough simulations (like rolling window and expanding window) result in a
complex matrix of functionalities that need to be implemented, unit tested, and
maintained. 

DataFlow integrates these components once and for all into the framework, to
streamline the process and allow to running detailed simulation in the design phase.
DataFlow supports many different learning styles from different types of
runners (e.g., in-sample-only, in-sample vs out-of-sample, rolling learning,
cross-validation) together with 

% =============================================================================
\subsection{Debugging production systems}

The importance of comparing production systems with simulations is highlighted
by the following typical activities:

\begin{itemize}
  \item \textbf{Quality Assessment}: To evaluate the assumptions made during the
  design phase, ensuring that models perform consistently with both real-time and
  historical data. This process is often called "reconciliation between research
  and production models."
    
  \item \textbf{Model Evaluation}: To assess how models respond to changes in
    real-world conditions. For example, understanding the impact of data arriving
    one second later than expected.

  \item \textbf{Debugging}: Production systems occasionally fail and require offline
    debugging. To troubleshoot production models by extracting values at internal
    nodes to identify and resolve failures.
\end{itemize}

In many engineering setups, there is no systematic approach to conducting these
analyses. As a result, data scientists and engineers often rely on cumbersome
and time-consuming ad-hoc methods to compare models.

DataFlow solves the problem of observability and debuggability of models by
easily allowing to capture and replay the execution of any subset of nodes. In this way,
it is possible to easily observe and debug the behavior of a complex system.
This comes naturally from the fact that research and production systems are the
same, from the accurate timing semantic of the simulation kernel.

% =============================================================================
\subsection{Model performance}

Performance of research and production systems need to be tuned to accomplish
various tradeoff (e.g., fit in memory, maximize throughput, minimize latency).

DataFlow addresses the model performance and its tradeoffs with several techniques
including:

\begin{itemize}

    \item \textbf{Tiling}. DataFlow's framework allows streaming data with
    different tiling styles (e.g., across time, across features, and both), to
    minimize the amount of working memory needed for a given computation,
    increasing the chances of caching computation.

    \item \textbf{Incremental computation and caching}. Because the
    dependencies between nodes are explicitly tracked by DataFlow, only nodes
    that see a change of inputs or in the implementation code need to be
    recomputed, while the redundant computation can be automatically cached.

    \item \textbf{Maximum parallelism}. Because the computation is expressed as
    a DAG, the DataFlow execution scheduler can extract the maximum amount of
    parallelism and execute multiple nodes in parallel in a distributed
    fashion, minimizing latency and maximizing throughput of a computation.

    \item \textbf{Automatic vectorization}. DataFlow nodes can use all native
    vectorization approaches available in numpy and Pandas.
\end{itemize}

% =============================================================================
\subsection{Performing parameters analyis}

Tracking and sweeping parameter is a common challenge in machine learning
projects:
\begin{itemize}
    \item During the research phase, data scientists perform numerous
    simulations to explore the parameter space. It is crucial to
    systematically specify and track these parameter sweeps

    \item Once the model is finalized, the model parameters must be fixed and
    these parameters should be deployed alongside the production system
\end{itemize}

In a DataFlow system it is easy to generate variations of DAGs using a
declarative approach to facilitate the adjustment of multiple parameters to
comprehensively explore the design space.

\begin{itemize}
  \item Each parameter is governed by a specific value within a configuration.
    This implies that the configuration space is equivalent to the space of
    DataFlow systems: each configuration uniquely defines a DataFlow system,
    and each DataFlow system is completely described by a configuration

  \item A configuration is organized as a nested dictionary, reflecting the
    structure of the Directed Acyclic Graph (DAG). This organization enables
    straightforward navigation through its structure
\end{itemize}

% =============================================================================
\subsection{Challenges in time series MLOps}

The complexity of Machine Learning Operations (MLOps) arises from managing the
full lifecycle of ML models in production. This includes not just training and
evaluation, but also deployment, monitoring, and governance.

DataFlow provides solutions to MLOps challenges fully integrated in the
framework.

\begin{itemize}
  \item \textbf{Model Serialization}. Once a Directed Acyclic Graph (DAG) is
  fitted, it can be saved to disk. This serialized model can later be loaded
  and used for making predictions in a production environment

  \item \textbf{Deployment and Monitoring}. Any DataFlow system is deployable
  as a Docker container. This includes the development system, which also
  operates within a Docker container. This setup facilitates the development and
  testing of systems on both cloud platforms (such as AWS) and local machines.
  Airflow is natively utilized for scheduling and monitoring long-running
  DataFlow systems.
\end{itemize}

  % ################################################################################
\section{Semantics}

% ================================================================================
\subsection{Nodes and DAGs}

% ................................................................................
\begin{definition}[Time index]
  We fix a totally ordered discrete time index set $(\mathbb{T},\leq)$, which for
  concreteness one may take as $\mathbb{Z}$ or any subset thereof.
\end{definition}

% ................................................................................
\begin{definition}[Time intervals]
  We write $[s,t] = \{ u \in \mathbb{T} : s \leq u \leq t \}$ for \emph{finite
  time intervals}.
\end{definition}

% ................................................................................
\begin{definition}[Dataframe as unit of computation]
The basic unit of computation of each node is a ``dataframe". Each node takes
multiple dataframes through its inputs, and emits one or more dataframes as
outputs.

In mathematical terms, a dataframe can be described as a two-dimensional (or
more, as described below) labeled data structure, similar to an array but with
more flexible features.

A dataframe $\df$ can be represented as:

\[
  \df = \left[
  \begin{array}{cccc}
    a_{11} & a_{12} & \cdots & a_{1n} \\
    a_{21} & a_{22} & \cdots & a_{2n} \\
    \vdots & \vdots & \ddots & \vdots \\
    a_{m1} & a_{m2} & \cdots & a_{mn} \\
  \end{array}
  \right]
\]

where:
\begin{itemize}
  \item $m$ is the number of rows (observations).

  \item $n$ is the number of columns (variables).

  \item $a_{ij}$ represents the element of the Dataframe in the $i$-th row and
    $j$-th column.
\end{itemize}
\end{definition}

% ................................................................................
\begin{remark}
Some characteristics of dataframes are:
\begin{enumerate}
  \item Labeled axes:
    \begin{itemize}
      \item Rows and columns are labeled, typically with strings, but labels can
        be of any hashable type.

      \item Rows are often referred to as indices and columns as column headers.
    \end{itemize}

  \item Heterogeneous data types:
    \begin{itemize}
      \item Each column $j$ can have a distinct data type, denoted as
        $\dtype_{j}$

      \item Common data types include integers, floats, strings, and datetime
        objects.
    \end{itemize}

  \item Mutable size:
    \begin{itemize}
      \item Rows and columns can be added or removed, meaning that the size of
        $\df$ is mutable.

      \item This adds to the flexibility as compared to fixed-size arrays.
    \end{itemize}

  \item Alignment and operations:
    \begin{itemize}
      \item Dataframes support alignment and arithmetic operations along rows and
        columns.

      \item Operations are often element-wise but can be customized with aggregation
        functions.
    \end{itemize}

  \item Missing data handling:
    \begin{itemize}
      \item Dataframes can represent missing data through $\NaN$ and $\None$ objects.

      \item Dataframes provide tools to handle, impute, or remove missing data.
    \end{itemize}

  \item Multidimensionality:
    \begin{itemize}
      \item Tensor-like objects are supported through row or column ``multi-indices".

      \item If time is the primary key, then multi-index columns can be used
        to support panel or higher-dimensional data at each timestamp.
    \end{itemize}
\end{enumerate}
\end{remark}

% ................................................................................
\begin{definition}[DataFlow data format]

% from docs/dataflow/dataflow_data_format.explanation.md

As explained in XYZ, raw data from \verb|DataPull| is stored in a ``long
format'', where the data is conditioned on the asset (e.g., \verb|full_symbol|).
\verb|DataFlow| transforms this into a multi-index wide format where the index
is a timestamp, the outermost column index is the feature, and the innermost
column index is the asset.

\begin{figure}[ht]
\centering
\begin{tabular}{c c c}
% Long Format Table
\begin{tabular}{|l|l|r|}
\hline
\multicolumn{3}{|c|}{\textbf{Long Format (Database Style)}} \\
\hline
\textbf{timestamp} & \textbf{symbol} & \textbf{close} \\
\hline
t\_1 & ADA & 2.76 \\
t\_1 & AVAX & 39.5 \\
t\_2 & ADA & 2.78 \\
t\_2 & AVAX & 39.3 \\
\hline
\end{tabular}
&
% Arrow
\begin{tikzpicture}
  \draw[->, >=stealth, line width=1.5pt] (0,0) -- (1.5,0) node[midway, above] {\texttt{pivot}};
\end{tikzpicture}
&
% Wide Format Table
\begin{tabular}{|l|r|r|}
\hline
\multicolumn{3}{|c|}{\textbf{Wide Multi-Index Format}} \\
\hline
 & \multicolumn{2}{c|}{\textbf{close}} \\
\cline{2-3}
\textbf{timestamp} & \textbf{ADA} & \textbf{AVAX} \\
\hline
t\_1 & 2.76 & 39.5 \\
t\_2 & 2.78 & 39.3 \\
\hline
\end{tabular}
\end{tabular}

\vspace{0.3cm}

The transformation is accomplished using pandas pivot operations:
\begin{lstlisting}[language=Python, basicstyle=\ttfamily\small]
# Transform from long to wide multi-index format
df_wide = df_long.pivot(
    index='timestamp',
    columns='full_symbol',
    values=['close', 'open', 'high', 'low']
)
\end{lstlisting}

\caption{Transformation from long format (database style) to wide multi-index format (DataFlow style). The wide format enables vectorized operations across assets.}
\label{fig:data_format_transformation}
\end{figure}

The reason for this convention is that typically features are computed in a uni-variate
fashion (e.g., asset by asset), and DataFlow can vectorize computation over
the assets by expressing operations in terms of the features. E.g., we can express
a feature as

\begin{lstlisting}[language=Python]
df["close", "open"].max() - df["high"]).shift(2)
\end{lstlisting}
\end{definition}

% ................................................................................
\begin{definition}[Stream dataframe]
Intuitively, a stream dataframe is a time-indexed sequence of rows that may be
arbitrarily long in the past and is typically extended online as new rows
arrive, i.e., a dataframe with a potentially unbounded number of rows.  

  Let $\mathcal{R}$ denote the set of possible dataframe rows (records).  A
  \emph{stream dataframe} is a partial function
  \[
    X : \mathbb{T}_{\leq t_{\max}} \to \mathcal{R},
  \]
  where $\mathbb{T}_{\leq t_{\max}} = \{ t \in \mathbb{T} : t \leq t_{\max} \}$
  for some $t_{\max} \in \mathbb{T}$.  The value $t_{\max}$ is the
  \emph{current time} of the stream dataframe.
\end{definition}

We may represent a stream dataframe up to time $t_0$ visually as in Table
\ref{table:stream_dataframe}.

% TODO(gp): Should we put arrow of time?

\begin{table}[ht]
    \caption{Stream dataframe up to time $t_0$}
    \centering
    \begin{tabular}[b]{|c|c|c|c|c|}
        \hline
        $\cdots$ & $\cdots$ & $\cdots$ & $\cdots$ & $\cdots$ \\
        \hline
        $t_{-n}$ & & & & \\
        \hline
        $\cdots$ & $\cdots$ & $\cdots$ & $\cdots$ & $\cdots$ \\
        \hline
        $t_{-3}$ & & & & \\
        \hline
        $t_{-2}$ & & & & \\
        \hline
        $t_{-1}$ & & & & \\
        \hline
        $t_0$ & & & & \\
        \hline
    \end{tabular}
    \label{table:stream_dataframe}
\end{table}

% ................................................................................
% TODO: Simplify definition
\begin{definition}[Computational Node]
  A \emph{computational node} $N$ with $m$ inputs and $n$ outputs is specified
  by a family of functions
  \[
    F_{N,[s,t]} :
      \prod_{i=1}^m \mathcal{R}^{[s,t]}
      \;\longrightarrow\;
      \prod_{j=1}^n \mathcal{R}^{[s,t]},
  \]
  indexed by finite intervals $[s,t] \subseteq \mathbb{T}$, $s \leq t$.  Given
  $m$ input stream dataframes $X^{(1)},\dots,X^{(m)}$ whose domains contain
  $[s,t]$, the node produces $n$ output stream dataframes
  $Y^{(1)},\dots,Y^{(n)}$ on $[s,t]$ via
  \[
    (Y^{(1)}|_{[s,t]},\dots,Y^{(n)}|_{[s,t]})
    \;=\;
    F_{N,[s,t]}\Bigl(
      X^{(1)}|_{[s,t]},\dots,X^{(m)}|_{[s,t]}
    \Bigr).
  \]
\end{definition}

The number of inputs $m$ and outputs $n$ is node-specific; either may be zero.

As in the general paradigm of dataflow computing, DataFlow represents
computation as a directed graph. Nodes represent computation and edges the
flow of data from one computational node to another. In DataFlow, each node
accepts one or more (fixed number of) tables as input and emits one or more
(fixed number of) tables as output. Nodes may either be stateless or
stateful.
A computational \emph{node} receives as input a (per-node) predefined number
(which may be zero) of stream dataframes and emits as output a predefined
number of stream dataframes. The number of outputs may be zero and may differ
from the number of inputs.

A computation node has:
\begin{itemize}
  \item a fixed number of inputs

  \item a fixed number of outputs

  \item a unique node id (aka \verb|nid|)

  \item a (optional) state
\end{itemize}

Inputs and outputs to a computational nodes are tables, represented in the
current implementation as \verb|Pandas| dataframes. A node uses the inputs to compute
the output (e.g., using \verb|Pandas| and \verb|Sklearn| libraries). A node
can execute in multiple ``phases'', referred to through the corresponding methods
called on the DAG (e.g., \verb|fit|, \verb|predict|, \verb|save_state|, \verb|load_state|).

A node stores an output value for each output and method name.

% TODO: Add tikz pic

% ................................................................................
\begin{remark}[Examples of Computational Node in Financial Trading Systems]

Examples of operations that may be performed by nodes include:

\begin{itemize}
  \item Loading data (e.g., market or alternative data)

  \item Resampling data bars (e.g., OHLCV data, tick data in finance)

  \item Computing rolling average (e.g., TWAP/VWAP, volatility of returns)

  \item Adjusting returns by volatility

  \item Applying FIR filters to signals

  \item Performing per-feature operations, each requiring multiple features

  \item Performing cross-sectional operations (e.g., factor residualization, Gaussian
    ranking)

  \item Learning/applying a machine learning model (e.g., using sklearn)

  \item Applying custom (user-written) functions to data
\end{itemize}

Further examples include nodes that maintain relevant trading state, or that interact
with an external environment:

\begin{itemize}
  \item Updating and processing current positions

  \item Performing portfolio optimization

  \item Generating trading orders

  \item Submitting orders to an API
\end{itemize}
\end{remark}

% ................................................................................
\begin{definition}[Causal Computation]
A distinguishing feature of DataFlow is how time series are handled. Within
each node, computation must be \emph{causal}.

To introduce this notion, we consider the case where the (row) index of the
tabular data is a time-based index. For convenience, we assume that the index
is sorted according to the natural time ordering. We note that while data
usually has multiple timestamps associated with it (e.g., event time,
timestamps for various stages of processing, a final ``knowledge timestamp" for
the system), for the purposes of DataFlow, a single notion of time is chosen as
primary key.
\end{definition}

% ................................................................................
\begin{definition}[Knowledge Time]
Next, we posit the existence of a ``simulation clock", against which the data
timestamps may be compared. In the case of real-time processing, the simulation
clock will coincide with the system clock. In the case of simulation, the
simulation clock is entirely independent of the system clock. In DataFlow, the
simulation clock has initial and terminal times (except in real-time
processing, where there is no terminal time) and advances according to a
schedule. If, at any point in simulation time, a DataFlow node computation only
depends upon data with timestamps earlier than the simulation time, then the
computation is said to be causal.
\end{definition}

% ................................................................................
\begin{remark}[Non-causal computations]
Of course, any real-time system must be causal (any computation it performs necessarily
only utilizes data available to it at the time of the computation). In many applications,
the correctness of a computation performed at a certain point in simulation
time is dependent upon having the complete set of data up to that point. While
DataFlow relegates that function to DataPull in the case of real-time pipelines,
violations of causal computation, either through human error or through data delays,
are detectable in DataFlow through its testing and replay framework.

% Example of non-causal systems
Note that there are some time series processing methods that prima facie are
not causal. An example of this is fixed-interval Kalman smoothing, which, to
calculate the smoothed data point at a particular point in time, requires data
from the adjacent future time interval of fixed length. Provided the dependence
upon future data is bounded in time, such techniques may be included in a
causal framework through the proper choice of primary key timestamp. In the
case of Kalman smoothing, this choice of timestamp would have the apparent
effect of yielding smoothed data points with a delay equal to the
fixed-interval of time required for the smoothing.
\end{remark}

% ................................................................................
\begin{definition}[Incremental Computation]
While latency-sensitive real-time systems are expected to perform computation
incrementally, i.e., with every data update, an important observation to make
is that, in many cases, causal computation need not be performed incrementally.
For example, calculating point-to-point percentage change in a time series
(such as calculating returns from prices) is causal, but may be vectorized over
a batch. In practice, this means that simulation or batch computation need not
be executed in the same way that a real-time system is.
\end{definition}

% ................................................................................
\begin{definition}[DataFlow DAG]
  A \emph{DataFlow DAG} is a directed acyclic graph
  $G = (V,E)$ where:
  \begin{itemize}
    \item each vertex $v \in V$ is a computational node,
    \item each edge $e = (u,v) \in E$ connects an output stream of $u$ to an
          input stream of $v$.
  \end{itemize}
  A subset of incoming edges to $G$ is designated as \emph{graph inputs}, and a
  subset of outgoing edges as \emph{graph outputs}.
\end{definition}

Nodes are organized by the user into a directed acyclic graph (DAG).

% rendered_images:begin
% ```graphviz
% digraph StreamDataflow {
%     rankdir=LR;
%     node [shape=box, style="rounded,filled", fontname="Helvetica", fillcolor="#A6E7F4"];
% 
%     // Example nodes
%     input1 [label="Stream in_0", shape=ellipse, fillcolor="#C6A6F4"];
%     input2 [label="Stream in_1", shape=ellipse, fillcolor="#C6A6F4"];
%     node1 [label="Node 1"];
%     node2 [label="Node 2"];
%     node3 [label="Node 3"];
%     output1 [label="Stream out_0", shape=ellipse, fillcolor="#C6A6F4"];
%     output2 [label="Stream out_1", shape=ellipse, fillcolor="#C6A6F4"];
% 
%     // DAG structure
%     input1 -> node1;
%     input2 -> node1;
%     node1 -> node2;
%     node1 -> node3;
%     node2 -> output1;
%     node3 -> output2;
% }
% ```
% rendered_images:end
% render_images:begin
\begin{figure}[!ht]
  \includegraphics[width=\linewidth]{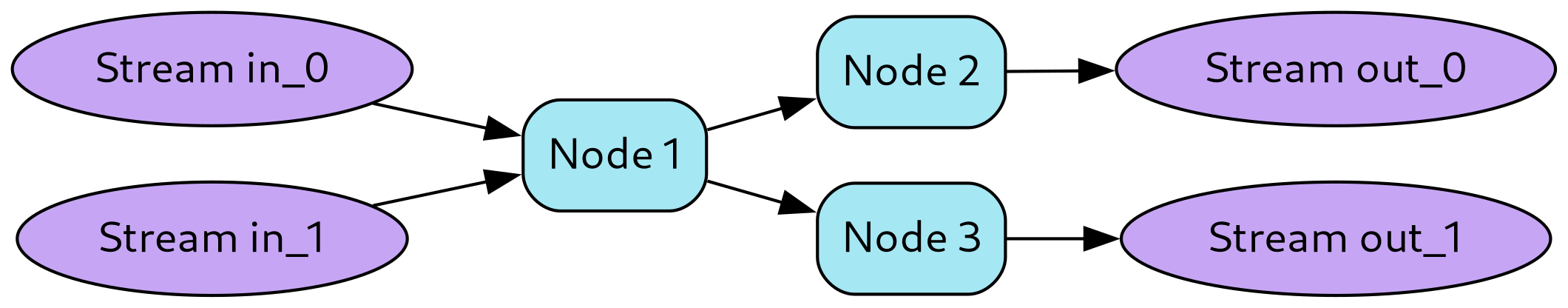}
\end{figure}
% render_images:end

% ................................................................................
\begin{definition}[DataFlow DAG]
  A \emph{DataFlow DAG} is a directed acyclic graph
  $G = (V,E)$ where:
  \begin{itemize}
    \item each vertex $v \in V$ is a computational node,
    \item each edge $e = (u,v) \in E$ connects an output stream of $u$ to an
          input stream of $v$.
  \end{itemize}
  A subset of incoming edges to $G$ is designated as \emph{graph inputs}, and a
  subset of outgoing edges as \emph{graph outputs}.
\end{definition}

Nodes are organized by the user into a directed acyclic graph (DAG).

% ================================================================================
\subsection{Context windows and point-in-time idempotency}
\label{subsec:context-windows}

We next formalize the temporal consistency condition imposed on each node.

\begin{definition}[Context window and point-in-time idempotency]
  Let $N$ be a node with associated maps $F_{N,[s,t]}$.  Fix $L \in \mathbb{N}$.
  \begin{enumerate}
    \item We say that $L$ is a \emph{context window length} for $N$ if the
      following holds: for all $t \in \mathbb{T}$ and all intervals
      $[s_1,t]$, $[s_2,t]$ satisfying
      \[
        t - s_1 + 1 \;\geq\; L,
        \qquad
        t - s_2 + 1 \;\geq\; L,
      \]
      for all input stream dataframes $X^{(1)},\dots,X^{(m)}$ we have
      \[
        (Y^{(1)}_{(1)},\dots,Y^{(n)}_{(1)})
        =
        F_{N,[s_1,t]}(X^{(1)}|_{[s_1,t]},\dots,X^{(m)}|_{[s_1,t]}),
      \]
      \[
        (Y^{(1)}_{(2)},\dots,Y^{(n)}_{(2)})
        =
        F_{N,[s_2,t]}(X^{(1)}|_{[s_2,t]},\dots,X^{(m)}|_{[s_2,t]}),
      \]
      implies
      \[
        Y^{(j)}_{(1)}(t) \;=\; Y^{(j)}_{(2)}(t)
        \quad\text{for all } j \in \{1,\dots,n\}.
      \]
    \item If such an $L$ exists, we say that $N$ is \emph{point-in-time
      idempotent}.  The minimal such $L$ (if it exists) is called the
      \emph{minimal context window length} of $N$ and denoted $L_N$.
  \end{enumerate}
\end{definition}

Intuitively, once the incoming window is at least $L$ steps long, the
latest-in-time output at $t$ is independent of how far back in time the window
starts.  This is exactly the informal notion of \emph{point-in-time
idempotency}.

Figure~\ref{fig:pit_idempotency} illustrates this concept with a concrete example
using a moving average with context window $L=3$.

\begin{figure}[ht]
\centering
\begin{tikzpicture}[scale=0.8]
  % Left panel - Window length 5
  \node[anchor=west] at (0, 4.5) {\textbf{Window length = 5}};

  % Timeline boxes for left panel
  \foreach \i/\val in {0/10, 1/12, 2/11, 3/13, 4/14} {
    \pgfmathsetmacro{\xpos}{\i * 1.2}
    \ifnum\i<2
      \draw[thick] (\xpos, 3) rectangle +(1, 0.8);
      \node at (\xpos + 0.5, 3.4) {$\val$};
      \node[below] at (\xpos + 0.5, 3) {\tiny $t_{-\pgfmathparse{int(4-\i)}\pgfmathresult}$};
    \else
      \draw[thick, fill=blue!20] (\xpos, 3) rectangle +(1, 0.8);
      \node at (\xpos + 0.5, 3.4) {$\val$};
      \node[below] at (\xpos + 0.5, 3) {\tiny $t_{-\pgfmathparse{int(4-\i)}\pgfmathresult}$};
    \fi
  }

  % Context window label
  \draw[<->, blue, thick] (2.4, 2.5) -- (6, 2.5);
  \node[blue, below] at (4.2, 2.5) {\small Context window $L=3$};

  % Arrow and output
  \draw[->, thick] (3, 1.8) -- (3, 1.2);
  \node[fill=orange!30, draw, thick] at (3, 0.6) {Output: $\frac{11+13+14}{3} = 12.67$};
  \node[below] at (3, 0.3) {$t_0$};

  % Right panel - Window length 10
  \node[anchor=west] at (8, 4.5) {\textbf{Window length = 10}};

  % Timeline boxes for right panel (showing more history)
  \foreach \i/\val in {0/8, 1/9, 2/10, 3/11, 4/12, 5/10, 6/12, 7/11, 8/13, 9/14} {
    \pgfmathsetmacro{\xpos}{8 + \i * 0.7}
    \ifnum\i<7
      \draw[thin] (\xpos, 3) rectangle +(0.6, 0.8);
      \node[font=\tiny] at (\xpos + 0.3, 3.4) {$\val$};
    \else
      \draw[thick, fill=blue!20] (\xpos, 3) rectangle +(0.6, 0.8);
      \node[font=\tiny] at (\xpos + 0.3, 3.4) {$\val$};
    \fi
  }

  % Context window label
  \draw[<->, blue, thick] (12.9, 2.5) -- (15, 2.5);
  \node[blue, below] at (13.95, 2.5) {\small Context $L=3$};

  % Arrow and output
  \draw[->, thick] (13.5, 1.8) -- (13.5, 1.2);
  \node[fill=orange!30, draw, thick] at (13.5, 0.6) {Output: $\frac{11+13+14}{3} = 12.67$};
  \node[below] at (13.5, 0.3) {$t_0$};

  % Equality indicator
  \node[font=\Large] at (7, 0.6) {$=$};

\end{tikzpicture}
\caption{Point-in-time idempotency demonstration. Despite different input window lengths (5 vs.\ 10 time steps), both computations produce identical output at $t_0$ because they depend only on the context window of length $L=3$ (shown in blue). The moving average uses only the last 3 values $\{11, 13, 14\}$ in both cases.}
\label{fig:pit_idempotency}
\end{figure}
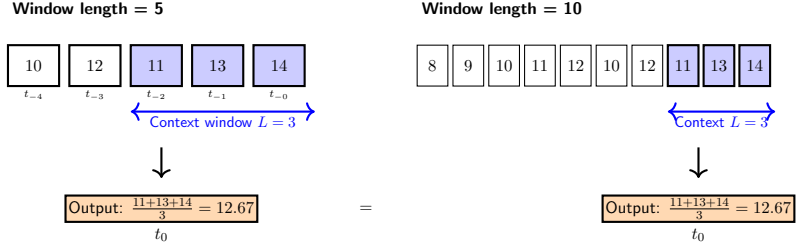

\begin{remark}[Tilability as a computational property]
  Point-in-time idempotency is closely related to the property of
  \emph{tilability}: a computation is tilable if its output does not depend on
  how input data is partitioned along the temporal axis, provided sufficient
  historical context is maintained. More precisely, for tiles $X_1$ on
  $[s_1,t_1]$ and $X_2$ on $[s_2,t_2]$ with $t_1 < s_2$, if we define the
  concatenation $X_1 \cup X_2$ as the stream covering $[s_1,t_2]$, then a
  computation $F$ satisfies tilability if
  \[
    F(X_1 \cup X_2)|_{[t_1-L+1,t_2]}
    =
    F(X_1)|_{[t_1-L+1,t_1]} \cup F(X_2)|_{[t_2-L+1,t_2]},
  \]
  where $L$ is the context window length and the restriction to output intervals
  ensures we only compare outputs that have sufficient context. Point-in-time
  idempotency guarantees that this equality holds at each time point
  independently, which in turn ensures tilability. This property enables a node
  to process data in arbitrary temporal chunks while producing consistent
  results, a fundamental requirement for unified batch-streaming execution.
\end{remark}

Another notable time series element of the DataFlow approach to computation is
the notion of a \emph{lookback period}. If computation at a certain point in
simulation time, say $t_{1}$, does not require timestamped data with
timestamps earlier than $t_{0}< t_{1}$, then the lookback period at time $t_{1}$
is $t_{1}- t_{0}> 0$. In many cases, this lookback period is independent of the
time $t_{1}$, in which case we may refer to the quantity as \emph{the} lookback
period. More generally, we define the lookback period to be the supremum of lookback
periods over all times $t_{1}$. The lookback period places an effective lower
bound on the amount of data required by a node to ensure computational
correctness and has implications for how a DataFlow graph may be executed. Certain
operations, e.g., exponentially weighted moving averages, have an effectively infinite
lookback period. Through careful state management and bookkeeping, such operations
may also be handled in DataFlow.

\begin{lemma}[Local dependence on trailing window]
  \label{lem:local-dependence}
  Let $N$ be a point-in-time idempotent node with context window length $L$.
  Then for each output index $j \in \{1,\dots,n\}$ there exists a function
  \[
    g_j :
      \prod_{i=1}^m \mathcal{R}^{L}
      \;\longrightarrow\;
      \mathcal{R}
  \]
  such that for any time $t \in \mathbb{T}$ and any interval $[s,t]$ with
  $t-s+1 \geq L$, if applying $N$ on $[s,t]$ yields outputs
  $(Y^{(1)},\dots,Y^{(n)})$, then
  \[
    Y^{(j)}(t)
    \;=\;
    g_j\Bigl(
      (X^{(1)}(t-L+1),\dots,X^{(1)}(t)),\dots,
      (X^{(m)}(t-L+1),\dots,X^{(m)}(t))
    \Bigr).
  \]
\end{lemma}

\begin{proof}
  Fix $j$.  For any tuple
  \[
    \bigl(
      (r^{(1)}_1,\dots,r^{(1)}_L),\dots,
      (r^{(m)}_1,\dots,r^{(m)}_L)
    \bigr)
    \in \prod_{i=1}^m \mathcal{R}^{L},
  \]
  choose any $t \in \mathbb{T}$ and interval $[s,t]$ with $t-s+1=L$.  Define
  input streams on $[s,t]$ by
  \[
    X^{(i)}(u) = r^{(i)}_{u-s+1},\quad u \in [s,t],\; i=1,\dots,m.
  \]
  Apply $N$ on $[s,t]$ to obtain outputs
  $(Y^{(1)},\dots,Y^{(n)}) = F_{N,[s,t]}(X^{(1)},\dots,X^{(m)})$ and set
  \[
    g_j\bigl( (r^{(1)}_1,\dots,r^{(1)}_L),\dots,(r^{(m)}_1,\dots,r^{(m)}_L) \bigr)
    \;:=\;
    Y^{(j)}(t).
  \]

  We must show $g_j$ is well-defined.  Consider any other choice of interval
  $[s',t']$ with $t'-s'+1=L$ and streams $X'^{(i)}$ that encode the same
  $L$-tuples.  Extend the original and alternative intervals arbitrarily
  backwards to $[s'',t]$ and $[s''',t']$, with $s'' \leq s$ and $s''' \leq s'$,
  and extend the streams accordingly.  By point-in-time idempotency with window
  $L$, the output at time $t$ (resp.\ $t'$) is invariant under the choice of
  how far back the window starts, provided its length is at least $L$.  Thus
  every such construction yields the same $Y^{(j)}(t)$, so $g_j$ is
  well-defined.

  Now take arbitrary input streams $X^{(1)},\dots,X^{(m)}$ and an interval
  $[s,t]$ with $t-s+1 \geq L$.  Let
  \[
    r^{(i)}_k := X^{(i)}(t-L+k),\quad k=1,\dots,L,\; i=1,\dots,m,
  \]
  and construct streams as above on an interval of length $L$ ending at $t$.
  Point-in-time idempotency implies that $Y^{(j)}(t)$ is the same as if we had
  started the window at $t-L+1$, hence it equals $g_j$ applied to the trailing
  $L$ tuples, as claimed.
\end{proof}

Lemma~\ref{lem:local-dependence} formalizes the intuitive claim that, once the
context window is ``filled,'' the output at time $t$ depends only on the last
$L$ time points of each input.

\begin{remark}[Temporal tiling in practice]
  \label{rem:temporal-tiling-examples}
  The tilability property has direct implications for how dataframe operations
  may be partitioned in time. We illustrate with two examples:

  \textbf{Example 1: Stateless operations.}
  Consider a node implementing element-wise addition. Let
  \[
    X_1 =
    \begin{array}{c|cc}
      \text{timestamp} & a & b \\
      \hline
      t_1 & 10 & 20 \\
      t_2 & 15 & 25
    \end{array},
    \quad
    X_2 =
    \begin{array}{c|cc}
      \text{timestamp} & a & b \\
      \hline
      t_3 & 12 & 22 \\
      t_4 & 18 & 28
    \end{array}.
  \]
  The operation $F(X) = X + X$ has context window $L=1$ (no memory). Processing
  the full stream $X_1 \cup X_2$ yields the same result as processing each tile
  independently and concatenating:
  \[
    F(X_1 \cup X_2) = F(X_1) \cup F(X_2).
  \]

  \textbf{Example 2: Operations with memory.}
  Consider a differencing operation $F(X)_t = X_t - X_{t-1}$, which has context
  window $L=2$. To correctly process tile $X_2$ starting at $t_3$, we must
  include the final observation from the previous tile:
  \[
    \tilde{X}_2 =
    \begin{array}{c|cc}
      \text{timestamp} & a & b \\
      \hline
      t_2 & 15 & 25 \\
      t_3 & 12 & 22 \\
      t_4 & 18 & 28
    \end{array}.
  \]
  Applying $F$ to the extended tile $\tilde{X}_2$ and restricting output to
  $[t_3,t_4]$ produces the same result as if the entire stream had been
  processed at once. This requirement---extending tiles with sufficient
  context---is central to DataFlow's tiling mechanism.

  More generally, operations with finite memory (such as rolling averages,
  exponential moving averages with recursive formulations, or FIR filters) can
  be made tilable by ensuring that each tile includes at least $L$ preceding
  observations, where $L$ is the node's context window.
\end{remark}

% =================================================================
\subsection{Stream dataframes and tiles}

We now formalize \emph{tiles}, which are contiguous windows of the stream
dataframe that are at least as long as the context window.

\begin{definition}[Tile]
  Let $X$ be a stream dataframe and let $L \in \mathbb{N}$.  A \emph{tile} of
  length $\tau \in \mathbb{N}$ ending at time $t$ is the restriction
  $X|_{[t-\tau+1,t]}$ for some $\tau \geq L$.  The interval $[t-\tau+1,t]$ is
  the \emph{temporal extent} of the tile and $t$ its \emph{right endpoint}.
\end{definition}

In Table \ref{table:tile} we depict a stream dataframe tile of length $4$
ending at time point $t_0$, with context window of length $3$.  We use blue
text coloring to indicate the indices of inputs that are included in the
context window (i.e., the last $L=3$ points).

\begin{table}[ht]
    \caption{A tile of length 4}
    \centering
    \begin{tabular}[b]{|c|c|c|c|c|}
        \hline
        $t_{-3}$ & & & & \\
        \hline
        $\textcolor{blue}{t_{-2}}$ & & & & \\
        \hline
        $\textcolor{blue}{t_{-1}}$ & & & & \\
        \hline
        $\textcolor{blue}{t_0}$ & & & & \\
        \hline
    \end{tabular}
    \label{table:tile}
\end{table}

\begin{definition}[Ideal pointwise semantics]
  Let $N$ be a point-in-time idempotent node with context window length $L$,
  and let $X^{(1)},\dots,X^{(m)}$ be input streams.  For any $t \in \mathbb{T}$
  and any $s \leq t$ with $t-s+1 \geq L$, let
  \[
    (Y^{(1)}_{[s,t]},\dots,Y^{(n)}_{[s,t]})
    =
    F_{N,[s,t]}(X^{(1)}|_{[s,t]},\dots,X^{(m)}|_{[s,t]}).
  \]
  By point-in-time idempotency, $Y^{(j)}_{[s,t]}(t)$ is independent of $s$, as
  long as $t-s+1 \geq L$.  We define the \emph{ideal output} of $N$ at time $t$
  by
  \[
    Y^{(j)}(t) := Y^{(j)}_{[s,t]}(t)
    \quad\text{for any } s \leq t-L+1.
  \]
\end{definition}

We now show that any tile whose length is at least the context window length
produces the correct (ideal) output at its right endpoint.

\begin{theorem}[Single-step tile correctness (streaming mode)]
  \label{thm:single-step-tile-correctness}
  Let $N$ be a point-in-time idempotent node with context window length $L$.
  Fix any tile length $\tau \geq L$, and let $X^{(1)},\dots,X^{(m)}$ be input
  streams.  For a time $t$ such that $[t-\tau+1,t]$ is contained in the domain
  of each input, define
  \[
    (Y^{(1)}_{[t-\tau+1,t]},\dots,Y^{(n)}_{[t-\tau+1,t]})
    =
    F_{N,[t-\tau+1,t]}(X^{(1)}|_{[t-\tau+1,t]},\dots,X^{(m)}|_{[t-\tau+1,t]}).
  \]
  Then for all $j \in \{1,\dots,n\}$,
  \[
    Y^{(j)}_{[t-\tau+1,t]}(t) \;=\; Y^{(j)}(t),
  \]
  i.e., the output at the right endpoint $t$ computed from the tile coincides
  with the ideal node output at $t$.
\end{theorem}

\begin{proof}
  By definition of the ideal semantics, for any $s \leq t-L+1$ we have
  \[
    Y^{(j)}(t) = Y^{(j)}_{[s,t]}(t).
  \]
  Since $\tau \geq L$, we have $t-(t-\tau+1)+1 = \tau \geq L$, so
  $t-\tau+1 \leq t-L+1$.  Thus $s := t-\tau+1$ is admissible, and
  \[
    Y^{(j)}(t)
    \;=\;
    Y^{(j)}_{[t-\tau+1,t]}(t),
  \]
  which is the desired claim.
\end{proof}

Point-in-time idempotency therefore guarantees correctness of the $t_0$ output
for any tile of length at least the context window (e.g., length $3$ or $4$ in
Table~\ref{table:tile}).  We illustrate this in
Table~\ref{table:stream_processing}, where the text in blue indicates the
context window inputs, and the text in orange indicates the corresponding
point-in-time idempotent output.

\begin{table}[ht]
    \caption{Stream processing of input context window}
    \centering
    \begin{tabular}{c c c}
        \begin{tabular}[b]{|c|c|c|c|c|}
            \hline
            $t_{-3}$ & & & & \\
            \hline
            $\textcolor{blue}{t_{-2}}$ & & & & \\
            \hline
            $\textcolor{blue}{t_{-1}}$ & & & & \\
            \hline
            $\textcolor{blue}{t_0}$ & & & & \\
            \hline
        \end{tabular}
        &
        \begin{tabular}[b]{c}
            \\
            $\rightarrow$ \\
            \\
        \end{tabular}
        &
        \begin{tabular}[b]{|c|c|c|c|c|}
            \hline
            $\textcolor{orange}{t_0}$ & & & & \\
            \hline
        \end{tabular}
    \end{tabular}
    \label{table:stream_processing}
\end{table}

When operating in streaming mode, the context window ``slides'' as time
advances.  In Table \ref{table:time_t1}, we show how the context window
advances as the next time point $t_1$ arrives, producing a new point-in-time
idempotent output at $t_1$.

\begin{table}[ht]
    \caption{Stream processing at next time step}
    \centering
    \begin{tabular}{c c c}
        \begin{tabular}[b]{|c|c|c|c|c|}
            \hline
            $t_{-3}$ & & & & \\
            \hline
            $t_{-2}$ & & & & \\
            \hline
            $\textcolor{blue}{t_{-1}}$ & & & & \\
            \hline
            $\textcolor{blue}{t_0}$ & & & & \\
            \hline
            $\textcolor{blue}{t_{1}}$ & & & & \\
            \hline
        \end{tabular}
        &
        \begin{tabular}[b]{c}
            \\
            $\rightarrow$ \\
            \\
        \end{tabular}
        &
        \begin{tabular}[b]{|c|c|c|c|c|}
            \hline
            $t_0$ & & & & \\
            \hline
            $\textcolor{orange}{t_{1}}$ & & & & \\
            \hline
        \end{tabular}
    \end{tabular}
    \label{table:time_t1}
\end{table}

\subsection{Mini-batch execution and tile-level idempotency}

When it is important to generate point-in-time outputs for times in the past
(e.g., in backtesting), it is often more computationally efficient to operate
in \emph{mini-batch} mode, where outputs are generated for multiple time points
simultaneously.  We formalize this using consecutive tiles and show that two
tiles suffice to generate an idempotent output tile.

\begin{definition}[Two-tile input window]
  Fix a tile length $\tau \in \mathbb{N}$ and a time $t \in \mathbb{T}$.  The
  \emph{two-tile input window} ending at $t$ is
  \[
    I^{(2)}_t := [t-2\tau+1, t].
  \]
  The subintervals $[t-2\tau+1, t-\tau]$ and $[t-\tau+1, t]$ are called the
  \emph{first} and \emph{second} tiles, respectively.
\end{definition}

\begin{definition}[Mini-batch output tile]
  Let $N$ be a node, $\tau \in \mathbb{N}$ a tile length, and
  $X^{(1)},\dots,X^{(m)}$ input streams.  For a time $t$ such that
  $I^{(2)}_t = [t-2\tau+1,t]$ is contained in the domain of each input, define
  \[
    (Y^{(1)}_{I^{(2)}_t},\dots,Y^{(n)}_{I^{(2)}_t})
    =
    F_{N,I^{(2)}_t}(X^{(1)}|_{I^{(2)}_t},\dots,X^{(m)}|_{I^{(2)}_t}),
  \]
  and call the collection
  \[
    \bigl( Y^{(j)}_{I^{(2)}_t}(u) \bigr)_{u \in [t-\tau+1,t],\; j=1,\dots,n}
  \]
  the \emph{mini-batch output tile} corresponding to the second tile.
\end{definition}

\begin{theorem}[Two tiles suffice for tile-level idempotency]
  \label{thm:two-tiles-suffice}
  Let $N$ be a point-in-time idempotent node with context window length $L$,
  and let $\tau \geq L$.  Let $X^{(1)},\dots,X^{(m)}$ be input streams defined
  on a prefix $\mathbb{T}_{\leq t_{\max}}$, and fix $t \leq t_{\max}$ such that
  $[t-2\tau+1,t] \subseteq \mathbb{T}_{\leq t_{\max}}$.  Then for every
  $u \in [t-\tau+1,t]$ and every $j \in \{1,\dots,n\}$,
  \[
    Y^{(j)}_{I^{(2)}_t}(u) \;=\; Y^{(j)}(u),
  \]
  i.e., the mini-batch output tile over the second tile coincides with the
  ideal pointwise outputs.  Moreover, these outputs are invariant under
  extending the input window further into the past.
\end{theorem}

\begin{proof}
  Fix $u \in [t-\tau+1,t]$ and $j$.  By
  Lemma~\ref{lem:local-dependence}, there exists a function $g_j$ such that the
  value of the $j$-th output at time $u$ depends only on the restriction of the
  inputs to
  \[
    J_u := [u-L+1, u].
  \]
  We first check $J_u \subseteq I^{(2)}_t$.  The right endpoint satisfies
  $u \leq t$ by assumption.  For the left endpoint, we use
  $u \geq t-\tau+1$ and $L \leq \tau$:
  \[
    u - L + 1
    \;\geq\;
    (t-\tau+1) - L + 1
    \;=\;
    t - (\tau + L - 2)
    \;\geq\;
    t - (2\tau-2)
    \;=\;
    t-2\tau+2
    \;\geq\;
    t-2\tau+1.
  \]
  Thus $J_u \subseteq [t-2\tau+1,t] = I^{(2)}_t$.

  Consequently, when we apply $N$ to $I^{(2)}_t$, the value
  $Y^{(j)}_{I^{(2)}_t}(u)$ is determined entirely by the inputs on $J_u$ and
  therefore coincides with the value obtained by applying $N$ to any interval
  ending at $u$ and containing $J_u$ with length at least $L$.  By the
  definition of the ideal semantics, this value is exactly $Y^{(j)}(u)$.
  This proves the first part.

  For invariance under further extension, consider any $s' \leq t-2\tau+1$ and
  the extended window $[s',t]$.  The interval $J_u$ remains contained in
  $[s',t]$, and point-in-time idempotency with window $L$ ensures that any
  earlier history prior to $u-L+1$ does not affect the output at $u$.  Hence
  the value at $u$ computed from $[s',t]$ equals that computed from $I^{(2)}_t$,
  establishing the desired invariance.
\end{proof}

Theorem~\ref{thm:two-tiles-suffice} justifies the mini-batch schematic shown
in Figure~\ref{fig:two_tile_execution} and Table~\ref{table:minibatch}.

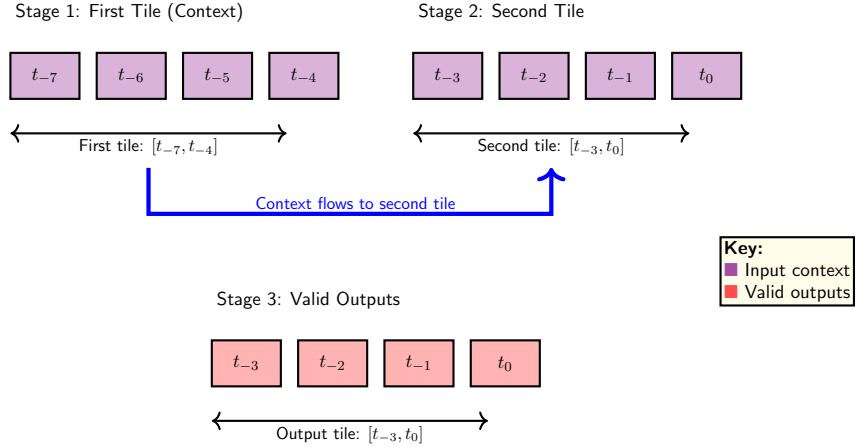
\begin{figure}[ht]
\centering
\begin{tikzpicture}[scale=0.9]
  % Stage 1: First Tile (Context)
  \node[anchor=west, font=\bfseries] at (0, 5.5) {Stage 1: First Tile (Context)};
  \foreach \i/\t in {0/$t_{-7}$, 1/$t_{-6}$, 2/$t_{-5}$, 3/$t_{-4}$} {
    \pgfmathsetmacro{\xpos}{\i * 1.5}
    \draw[thick, fill=violet!30] (\xpos, 4) rectangle +(1.2, 0.8);
    \node at (\xpos + 0.6, 4.4) {\t};
  }
  \draw[<->, thick] (0, 3.4) -- (4.8, 3.4);
  \node[below] at (2.4, 3.4) {\small First tile: $[t_{-7}, t_{-4}]$};

  % Stage 2: Second Tile (Computation)
  \node[anchor=west, font=\bfseries] at (7, 5.5) {Stage 2: Second Tile};
  \foreach \i/\t in {0/$t_{-3}$, 1/$t_{-2}$, 2/$t_{-1}$, 3/$t_0$} {
    \pgfmathsetmacro{\xpos}{7 + \i * 1.5}
    \draw[thick, fill=violet!30] (\xpos, 4) rectangle +(1.2, 0.8);
    \node at (\xpos + 0.6, 4.4) {\t};
  }
  \draw[<->, thick] (7, 3.4) -- (11.8, 3.4);
  \node[below] at (9.4, 3.4) {\small Second tile: $[t_{-3}, t_0]$};

  % Arrow showing flow
  \draw[->, ultra thick, blue] (2.4, 2.8) -- (2.4, 2) -- (9.4, 2) -- (9.4, 2.8);
  \node[above, blue] at (6, 2) {\small Context flows to second tile};

  % Stage 3: Output Tile
  \node[anchor=west, font=\bfseries] at (3.5, 0.5) {Stage 3: Valid Outputs};
  \foreach \i/\t in {0/$t_{-3}$, 1/$t_{-2}$, 2/$t_{-1}$, 3/$t_0$} {
    \pgfmathsetmacro{\xpos}{3.5 + \i * 1.5}
    \draw[thick, fill=red!30] (\xpos, -1) rectangle +(1.2, 0.8);
    \node at (\xpos + 0.6, -0.6) {\t};
  }
  \draw[<->, thick] (3.5, -1.6) -- (8.3, -1.6);
  \node[below] at (5.9, -1.6) {\small Output tile: $[t_{-3}, t_0]$};

  % Annotation box
  \node[draw, thick, align=left, fill=yellow!10] at (13.5, 1) {
    \textbf{Key:} \\
    \textcolor{violet!70}{$\blacksquare$} Input context \\
    \textcolor{red!70}{$\blacksquare$} Valid outputs
  };

\end{tikzpicture}
\caption{Two-tile execution mechanism. The first tile provides context for boundary computations in the second tile. Outputs in the second tile (red) are valid because they have access to sufficient context from both tiles. For nodes with context window $L$, outputs at times like $t_{-3}$ depend on inputs spanning both tiles.}
\label{fig:two_tile_execution}
\end{figure}

In Figure~\ref{fig:two_tile_execution} and Table~\ref{table:minibatch}, we use
purple (violet) to indicate the input context required for tile-level idempotency,
and red to indicate the idempotent outputs associated with the tile-level context.

\begin{table}[ht]
    \caption{Temporal mini-batch processing}
    \centering
    \begin{tabular}{c c c}
        \begin{tabular}[b]{c}
            \begin{tabular}[b]{|c|c|c|c|c|}
                \hline
                $t_{-7}$ & & & & \\
                \hline
                $t_{-6}$ & & & & \\
                \hline
                $\textcolor{violet}{t_{-5}}$ & & & & \\
                \hline
                $\textcolor{violet}{t_{-4}}$ & & & & \\
                \hline
            \end{tabular} \\
            \begin{tabular}[b]{|c|c|c|c|c|}
                \hline
                $\textcolor{violet}{t_{-3}}$ & & & & \\
                \hline
                $\textcolor{violet}{t_{-2}}$ & & & & \\
                \hline
                $\textcolor{violet}{t_{-1}}$ & & & & \\
                \hline
                $\textcolor{violet}{t_0}$ & & & & \\
                \hline
            \end{tabular}
        \end{tabular}
        &
        \begin{tabular}[b]{c}
            \\
            $\rightarrow$ \\
            \\
            \\
            \\
        \end{tabular}
        &
        \begin{tabular}[b]{|c|c|c|c|c|}
            \hline
            $\textcolor{red}{t_{-3}}$ & & & & \\
            \hline
            $\textcolor{red}{t_{-2}}$ & & & & \\
            \hline
            $\textcolor{red}{t_{-1}}$ & & & & \\
            \hline
            $\textcolor{red}{t_0}$ & & & & \\
            \hline
        \end{tabular}
    \end{tabular}
    \label{table:minibatch}
\end{table}

Note that outputs $t_{-2}$ and $t_{-3}$ in Table~\ref{table:minibatch} have
context windows spanning both tiles, as predicted by the proof of
Theorem~\ref{thm:two-tiles-suffice}.

\begin{remark}[Detecting incorrect tiled computations]
  \label{rem:tiling-validation}
  The tilability property provides a practical method for validating the
  correctness of DataFlow implementations. Given a node (or DAG) purported to
  have context window length $L$, one may perform the following validation:
  \begin{enumerate}
    \item Execute the node on the full historical stream to obtain reference
      outputs $Y^{\text{ref}}$.
    \item Partition the same historical data into tiles of varying lengths
      $\tau_1, \tau_2, \dots \geq L$ with different temporal boundaries.
    \item Execute the node on each tiled partition (using two-tile windows as
      per Theorem~\ref{thm:two-tiles-suffice}) to obtain outputs
      $Y^{(\tau_1)}, Y^{(\tau_2)}, \dots$
    \item Verify that $Y^{(\tau_k)} = Y^{\text{ref}}$ for all $k$ up to
      numerical precision.
  \end{enumerate}
  Discrepancies between tiled and untiled outputs indicate one of three failure
  modes: (1) the stated context window $L$ is insufficient; (2) the node
  implementation inadvertently accesses future data (violates causality); or
  (3) the node exhibits nondeterministic behavior or untracked dependencies.

  This validation protocol is particularly valuable for detecting
  \emph{future-peeking bugs}, where a computation inadvertently uses
  information from time $t' > t$ when producing output at time $t$. Such errors
  are notoriously difficult to detect in batch-mode development but manifest
  immediately as tiling inconsistencies. DataFlow's testing framework automates
  this validation, enabling early detection of causality violations during the
  development cycle.
\end{remark}

\begin{remark}[Tile size vs.\ DAG context window]
  Tile size $\tau$ is a user-selected parameter whose optimal value depends on
  use case, data requirements, and computational resources.  For example,
  backtesting on a single machine may suggest a tile size dictated by memory
  constraints, while a near real-time system whose input frequency exceeds the
  decision-making frequency may suggest a small tile size.  Latency-critical
  real-time applications suggest the minimum tile size, i.e., the DAG-level
  context window length.  Crucially, $\tau$ is tunable and separate from the
  computational logic: the same DataFlow DAG may be used for both
  latency-critical real-time applications and large-scale historical backtests,
  with correctness guaranteed whenever $\tau$ is chosen to be at least the DAG
  context window.
\end{remark}

% =============================================================================
\subsection{Columnar grouping}

Up to this point we have discussed tiling purely in terms of time indexing.
DataFlow also supports columnar tiling, which allows the separation of
computationally independent columns into separate tiles. This form of
independence is easiest to interpret when operating with tiles with multilevel
column indexing.

Formally, let $\mathcal{U}$ denote a finite index set of \emph{entities}
(e.g., customers), and let $\mathcal{F}$ denote a finite index set of
\emph{features} (e.g., ``time spent'', ``dollars spent'').  The set of
columns is then the Cartesian product
\[
  \mathcal{C} := \mathcal{U} \times \mathcal{F},
\]
and a stream dataframe row at time $t$ may be written as
\[
  X_t = \bigl( X_t(u,f) \bigr)_{(u,f)\in\mathcal{C}}.
\]
A \emph{columnar grouping} is a partition
\[
  \mathcal{P} = \{ C_1,\dots,C_K \}
  \quad\text{of } \mathcal{C},
\]
where each $C_k \subseteq \mathcal{C}$ is a group of columns that we may tile
and process jointly.

\begin{definition}[Column-restricted time tile]
  Let $[s,t]$ be a time interval and let $C \subseteq \mathcal{C}$ be a column
  group.  The \emph{column-restricted time tile} of $X$ over $[s,t]$ and $C$ is
  the restriction
  \[
    X|_{[s,t]\times C}
    :=
    \bigl( X_u(c) \bigr)_{u \in [s,t],\, c \in C}.
  \]
\end{definition}

For example, suppose there is a column index level called
``customer'', and to each customer there are associated fields
``time spent'' and ``dollars spent''. We depict such a streaming dataframe in
Table \ref{table:multicol}.

\begin{table}[ht]
    \caption{Streaming dataframe with columnar grouping}
    \centering
    \begin{tabular}[b]{|c|c|c|c|c|}
        \hline
        \multirow{2}{*}{} &
          \multicolumn{2}{c}{A} &
          \multicolumn{2}{c|}{B} \\
          & time & \$ & time & \$ \\
        \hline
        $\cdots$ & $\cdots$ & $\cdots$ & $\cdots$ & $\cdots$ \\
        \hline
        $t_{-n}$ & & & & \\
        \hline
        $\cdots$ & $\cdots$ & $\cdots$ & $\cdots$ & $\cdots$ \\
        \hline
        $t_{-3}$ & & & & \\
        \hline
        $t_{-2}$ & & & & \\
        \hline
        $t_{-1}$ & & & & \\
        \hline
        $t_0$ & & & & \\
        \hline
    \end{tabular}
    \label{table:multicol}
\end{table}

In this example, $\mathcal{U} = \{A,B\}$ and
$\mathcal{F} = \{\text{time}, \$\}$, so
$\mathcal{C} = \{(A,\text{time}),(A,\$),(B,\text{time}),(B,\$)\}$.
Table~\ref{table:multicol} depicts the restriction of the stream dataframe to
these columns over time.

Suppose that one wishes to aggregate over a sliding time window the
``time spent'' column across all users, and that, separately, one wishes to
identify the user who spends the greatest dollar amount in a given time window.
In this contrived example, the ``time'' and ``dollar'' computations are
independent, which allows one to tile according to column group (in addition to
time).

Formally, let $\mathcal{U} = \{A,B,\dots\}$ and define feature-specific column
groups
\[
  C_{\text{time}} := \mathcal{U} \times \{\text{time}\},
  \qquad
  C_{\$} := \mathcal{U} \times \{\$\}.
\]
Let $[t-3,t_0]$ be a time interval of length $4$.  Consider two functionals:
\begin{align*}
  G_{\text{time}} &: X|_{[t-3,t_0]\times C_{\text{time}}}
    \;\longmapsto\; \text{(aggregate time spent over users and time)}, \\
  G_{\$} &: X|_{[t-3,t_0]\times C_{\$}}
    \;\longmapsto\; \text{(argmax user by dollars spent over window)}.
\end{align*}
We say that these functionals are \emph{column-disjoint} because
$G_{\text{time}}$ depends only on $C_{\text{time}}$ and $G_{\$}$ depends only
on $C_{\$}$.

\begin{definition}[Column-separable computation]
  Let $\mathcal{P} = \{ C_1,\dots,C_K \}$ be a partition of $\mathcal{C}$.
  A node-level map on a time interval,
  \[
    F_{[s,t]} : \mathcal{R}^{[s,t]\times\mathcal{C}}
    \;\longrightarrow\;
    \mathcal{O},
  \]
  is \emph{column-separable with respect to $\mathcal{P}$} if there exist maps
  \[
    F^{(k)}_{[s,t]} : \mathcal{R}^{[s,t]\times C_k}
      \;\longrightarrow\; \mathcal{O}_k,
      \quad k=1,\dots,K,
  \]
  and a combining map $\Phi : \mathcal{O}_1 \times \cdots \times \mathcal{O}_K
  \to \mathcal{O}$ such that for all inputs $X$ on $[s,t]\times\mathcal{C}$,
  \[
    F_{[s,t]}(X)
    =
    \Phi\Bigl(
      F^{(1)}_{[s,t]}(X|_{[s,t]\times C_1}),\dots,
      F^{(K)}_{[s,t]}(X|_{[s,t]\times C_K})
    \Bigr).
  \]
\end{definition}

In this terminology, the ``time spent'' aggregation and ``dollars spent''
argmax are column-separable with respect to
$\mathcal{P} = \{C_{\text{time}}, C_{\$}\}$, with $K=2$ and two independent
subcomputations $F^{(1)}$ and $F^{(2)}$.

\begin{proposition}[Parallel column-group evaluation]
  \label{prop:column-parallel}
  Let $\mathcal{P} = \{ C_1,\dots,C_K \}$ be a partition of $\mathcal{C}$ and
  let $F_{[s,t]}$ be column-separable with respect to $\mathcal{P}$.  Then for
  any time interval $[s,t]$ and any stream dataframe $X$ on
  $[s,t]\times\mathcal{C}$, the value $F_{[s,t]}(X)$ can be computed by
  evaluating $F^{(k)}_{[s,t]}$ independently on each column-restricted tile
  $X|_{[s,t]\times C_k}$ and combining the results via $\Phi$.
\end{proposition}

\begin{proof}
  This is immediate from the definition of column-separability: by assumption,
  for all $X$,
  \[
    F_{[s,t]}(X)
    =
    \Phi\Bigl(
      F^{(1)}_{[s,t]}(X|_{[s,t]\times C_1}),\dots,
      F^{(K)}_{[s,t]}(X|_{[s,t]\times C_K})
    \Bigr),
  \]
  and each argument of $\Phi$ depends only on the restriction of $X$ to the
  corresponding group $C_k$.  Thus the $K$ subcomputations can be carried out
  independently and in parallel on the respective column-restricted tiles.
\end{proof}

\begin{remark}[Cross-sectional tiling and computational independence]
  \label{rem:cross-sectional-tiling}
  Proposition~\ref{prop:column-parallel} establishes that computations may be
  partitioned not only along the temporal axis (as in temporal tiling) but also
  along the feature/entity axis when column-separability holds. This enables
  \emph{cross-sectional tiling}, where the dataframe is partitioned into
  column-disjoint subsets that may be processed independently.

  Cross-sectional tiling provides several benefits:
  \begin{itemize}
    \item \textbf{Horizontal scalability.} When processing thousands of
      entities (e.g., financial instruments, customers, sensors), column groups
      may be distributed across multiple machines, enabling data-parallel
      execution that scales linearly with the number of entities.

    \item \textbf{Memory efficiency.} Large dataframes with many columns may
      exceed available memory. Column-restricted tiles allow processing subsets
      that fit in memory, with results combined via the map $\Phi$.

    \item \textbf{Heterogeneous processing.} Different column groups may be
      assigned to processors with different characteristics (e.g., CPU vs.\ GPU)
      based on their computational requirements.
  \end{itemize}

  Importantly, temporal and cross-sectional tiling may be \emph{composed}: a
  dataframe may be partitioned along both time and column dimensions
  simultaneously, yielding a two-dimensional tiling that exploits independence
  along both axes. The resulting tiles are characterized by a time interval
  $[s,t]$ and a column group $C_k$, and correctness follows from the conjunction
  of temporal point-in-time idempotency and column-separability.

  Not all operations are column-separable. Cross-sectional operations (e.g.,
  normalizing features to have zero mean across entities at each time point, or
  computing entity rankings) require access to all columns at once and cannot be
  column-tiled. Such operations impose constraints on the column-grouping
  strategy and may require materialization of full cross-sections at designated
  synchronization points in the DAG.
\end{remark}

An example of the two such tiles of context window length $4$ that would be
generated ending at $t_0$ are as in Table \ref{table:multicol_tiling}.  The
left tile corresponds to the group $C_{\text{time}}$, and the right to
$C_{\$}$.

\begin{table}[ht]
    \caption{Tiles with column groups}
    \centering
    \begin{tabular}{c c}
        \begin{tabular}[b]{|c|c|c|}
            \hline
            \multirow{2}{*}{} &
              \multicolumn{2}{c|}{time} \\
              & A & B \\
            \hline
            $t_{-3}$ & & \\
            \hline
            $t_{-2}$ & & \\
            \hline
            $t_{-1}$ & & \\
            \hline
            $t_0$ & & \\
            \hline
        \end{tabular}
        &
        \begin{tabular}[b]{|c|c|c|}
            \hline
            \multirow{2}{*}{} &
              \multicolumn{2}{c|}{\$} \\
              & A & B \\
            \hline
            $t_{-3}$ & & \\
            \hline
            $t_{-2}$ & & \\
            \hline
            $t_{-1}$ & & \\
            \hline
            $t_0$ & & \\
            \hline
        \end{tabular}
    \end{tabular}
    \label{table:multicol_tiling}
\end{table}

Consider an alternative scenario where, instead of aggregating features
cross-sectionally across customers, one instead wished to perform a
per-customer data operation. Such a scenario would permit an alternative
column grouping into tiles as shown in Table \ref{table:multicol_tiling2}.

Formally, suppose we have a node whose computation is \emph{per-customer
separable}: there exists a map
\[
  H_{[s,t]} : \mathcal{R}^{[s,t]\times\mathcal{F}} \longrightarrow \mathcal{O}
\]
such that for each customer $u \in \mathcal{U}$ and any input $X$,
\[
  \text{(output for customer $u$ over $[s,t]$)}
  \;=\;
  H_{[s,t]}\bigl( X|_{[s,t]\times(\{u\}\times\mathcal{F})} \bigr).
\]
Define column groups
\[
  C_u := \{u\} \times \mathcal{F},\quad u \in \mathcal{U},
\]
so that $\mathcal{P}_{\text{cust}} := \{ C_u : u \in \mathcal{U} \}$ is a
partition of $\mathcal{C}$ by customer.  In this case, all groups share the
same computation $H_{[s,t]}$ up to the label $u$, so the resulting tiles are
\emph{semantically identical} in the sense that each tile is processed by the
same function.

An example of the resulting column grouping for $\mathcal{U} = \{A,B\}$ and
$\mathcal{F} = \{\text{time},\$\}$ is shown in
Table~\ref{table:multicol_tiling2}.

\begin{table}[ht]
    \caption{Tiles with alternative column groups}
    \centering
    \begin{tabular}{c c}
        \begin{tabular}[b]{|c|c|c|}
            \hline
            \multirow{2}{*}{} &
              \multicolumn{2}{c|}{A} \\
              & time & \$ \\
            \hline
            $t_{-3}$ & & \\
            \hline
            $t_{-2}$ & & \\
            \hline
            $t_{-1}$ & & \\
            \hline
            $t_0$ & & \\
            \hline
        \end{tabular}
        &
        \begin{tabular}[b]{|c|c|c|}
            \hline
            \multirow{2}{*}{} &
              \multicolumn{2}{c|}{B} \\
              & time & \$ \\
            \hline
            $t_{-3}$ & & \\
            \hline
            $t_{-2}$ & & \\
            \hline
            $t_{-1}$ & & \\
            \hline
            $t_0$ & & \\
            \hline
        \end{tabular}
    \end{tabular}
    \label{table:multicol_tiling2}
\end{table}

\begin{definition}[Semantically identical vs.\ different tiles]
  Let $\mathcal{P} = \{ C_1,\dots,C_K \}$ be a column partition and let
  $F^{(k)}_{[s,t]}$ denote the node computation restricted to tile
  $[s,t]\times C_k$ (in the sense of column-separability).
  \begin{itemize}
    \item We say that tiles $C_k$ and $C_\ell$ are \emph{semantically
      identical} if there exists a relabeling (e.g., of customers) under which
      $F^{(k)}_{[s,t]}$ and $F^{(\ell)}_{[s,t]}$ coincide for all inputs.
    \item Otherwise we say that the tiles are \emph{semantically different}.
  \end{itemize}
\end{definition}

Note that the tiling in Table \ref{table:multicol_tiling} produces semantically
different tiles (one tile is processed by $G_{\text{time}}$, the other by
$G_{\$}$), whereas the tiling in Table \ref{table:multicol_tiling2} produces
semantically identical tiles (each tile is processed by the same per-customer
map $H_{[s,t]}$). Both cases support parallelism, but in different ways. The
former lends itself to parallelism across multiple DAG nodes (or distinct
subcomputations), whereas the latter lends itself to intra-node parallelism
(e.g., parallelizing the same computation across many customers).

\begin{remark}[Working with DataFlow structures]
A user can work with DataFlow at 4 levels of abstraction:

\begin{enumerate}
  \item Pandas long-format (non multi-index) dataframes and for-loops
    \begin{itemize}
      \item We can do a group-by or filter by \verb|full_symbol|

      \item Apply the transformation on each resulting dataframe

      \item Merge the data back into a single dataframe with the long-format
    \end{itemize}

  \item Pandas multiindex dataframes
    \begin{itemize}
      \item The data is in the DataFlow native format

      \item We can apply the transformation in a vectorized way

      \item This approach is best for performance and with compatibility with DataFlow
        point of view

      \item An alternative approach is to express multi-index transformations in
        terms of approach 1 (i.e., single asset transformations and then
        concatenation). This approach is functionally equivalent to a multi-index
        transformation, but typically slow and memory inefficient
    \end{itemize}

  \item DataFlow nodes
    \begin{itemize}
      \item A DataFlow node implements certain transformations on dataframes according
        to the DataFlow convention and interfaces

      \item Nodes operate on the multi-index representation by typically calling
        functions from level 2 above
    \end{itemize}

  \item DAG
    \begin{itemize}
      \item A series of transformations in terms of DataFlow nodes
    \end{itemize}
\end{enumerate}

%An example ./amp/dataflow/notebooks/gallery_dataflow_example.ipynb
%TODO(gp): Fix this reference.
\end{remark}

% =============================================================================
\subsection{Closure properties and common operators}

The point-in-time idempotency property composes well under standard dataframe
operations. The following lemmas establish context window requirements for
commonly used transformations.

\begin{lemma}[Stateless and FIR operators]
  \label{lem:fir}
  A pointwise operator $Y_t = f(X_t)$ is point-in-time idempotent with context
  window $w=1$. Any finite impulse response (FIR) operator with horizon $h$ is
  point-in-time idempotent with context window $w=h$.
\end{lemma}

\begin{proof}
  A pointwise operator depends only on the current time point, so $L=1$
  suffices. An FIR operator with horizon $h$ computes
  \[
    Y_t = g(X_{t-h+1}, X_{t-h+2}, \dots, X_t)
  \]
  for some function $g$, so the output at time $t$ depends only on the last $h$
  inputs, establishing $L=h$.
\end{proof}

\begin{lemma}[Column-group-separable nodes]
  \label{lem:col-groups}
  Let the columns decompose into disjoint groups $C_1,\dots,C_m$ and suppose
  the node acts independently on each group with contexts $w_1,\dots,w_m$.
  Then the node is point-in-time idempotent with $w = \max_j w_j$.
\end{lemma}

\begin{proof}
  For each group $C_j$, the output at time $t$ restricted to columns in $C_j$
  depends only on inputs in $C_j$ over the window $[t-w_j+1,t]$. Since the
  computations are independent across groups, the overall output at time $t$
  depends on the union of these windows, which is
  $[t-\max_j w_j+1,t]$. Thus $w = \max_j w_j$ suffices.
\end{proof}

\begin{lemma}[Joins with temporal alignment]
  \label{lem:join-delta}
  Suppose two input streams are joined on time with left-closed, right-exact
  alignment and bounded misalignment $\Delta \in \mathbb{N}$ (e.g.,
  last-observation-carried-forward up to $\Delta$ time steps). If the input
  nodes are point-in-time idempotent with context windows $w_1$ and $w_2$,
  then the join node is point-in-time idempotent with
  $w = \max(w_1, w_2) + \Delta$.
\end{lemma}

\begin{proof}
  The join output at time $t$ may depend on input values from either stream up
  to $\Delta$ time steps earlier (due to the alignment tolerance). For each
  input stream $i$, we require access to $[t-w_i+1,t]$ to compute its
  contribution. Accounting for the misalignment, we need access to
  $[t-w_i-\Delta+1,t]$. Taking the maximum over both streams yields
  $w = \max(w_1, w_2) + \Delta$.
\end{proof}

% =============================================================================
\subsection{Approximating infinite impulse response operations}

While finite impulse response (FIR) operations naturally satisfy point-in-time
idempotency with bounded context windows, infinite impulse response (IIR)
operations (such as exponential moving averages) technically require access to
the entire history. In practice, IIR operations can be approximated by FIR
operations with bounded truncation error.

\begin{theorem}[FIR approximation error for EWMA]
  \label{thm:ewma}
  Let $y_{t+1} = (1-\lambda)x_t + \lambda y_t$ with $0 < \lambda < 1$ define
  an exponentially weighted moving average (EWMA). Let $\tilde{y}_t^{(h)}$ be
  the truncated FIR approximation using the last $h$ inputs:
  \[
    \tilde{y}_t^{(h)}
    = (1-\lambda) \sum_{j=0}^{h-1} \lambda^j x_{t-j}.
  \]
  Then the approximation error satisfies
  \[
    \bigl|y_t - \tilde{y}_t^{(h)}\bigr|
    \;\le\;
    \lambda^h \cdot \max_{s \le t} |y_s|.
  \]
  Hence choosing $h \ge \frac{\ln(1/\varepsilon)}{\ln(1/\lambda)}$ ensures
  error $\le \varepsilon$ relative to the maximum historical value.
\end{theorem}

\begin{proof}
  The exact EWMA can be written as
  \[
    y_t = (1-\lambda) \sum_{j=0}^{\infty} \lambda^j x_{t-j}.
  \]
  The truncation error is
  \[
    |y_t - \tilde{y}_t^{(h)}|
    = \left|(1-\lambda) \sum_{j=h}^{\infty} \lambda^j x_{t-j}\right|
    \le (1-\lambda) \sum_{j=h}^{\infty} \lambda^j \max_{s\le t}|x_s|.
  \]
  The geometric series sums to
  \[
    (1-\lambda) \sum_{j=h}^{\infty} \lambda^j
    = (1-\lambda) \cdot \frac{\lambda^h}{1-\lambda}
    = \lambda^h.
  \]
  Thus $|y_t - \tilde{y}_t^{(h)}| \le \lambda^h \max_{s\le t}|x_s|$. Since
  $|x_s| \le |y_s|$ for EWMA, the bound follows. Solving
  $\lambda^h = \varepsilon$ gives $h = \frac{\ln(1/\varepsilon)}{\ln(1/\lambda)}$.
\end{proof}

\begin{remark}[Approximate point-in-time idempotency]
  Theorem~\ref{thm:ewma} justifies treating IIR operations as approximately
  point-in-time idempotent by selecting a truncation horizon $h$ that achieves
  a target error tolerance $\varepsilon$. For example, with $\lambda = 0.9$ and
  $\varepsilon = 10^{-6}$, we require $h \ge \frac{\ln(10^6)}{\ln(10/9)} \approx 131$
  time steps. In practice, context windows of 100--200 observations suffice for
  common EWMA applications in financial time series.
\end{remark}

% =============================================================================
\subsection{Caching and reconciliation guarantees}

The deterministic nature of point-in-time idempotent computations enables
automatic caching and ensures consistency between research and production
executions.

\begin{theorem}[Cache correctness]
  \label{thm:cache}
  Let $N$ be a deterministic node that is point-in-time idempotent with context
  window $w$. Suppose the node's output depends only on (i) the last $w$ input
  rows, (ii) a configuration parameter $c$, and (iii) an implementation hash
  $h$. Then memoizing outputs by key $(X_{[t-w+1,t]}, c, h)$ reproduces
  exactly the output at time $t$ across different executions.
\end{theorem}

\begin{proof}
  By point-in-time idempotency (Definition in subsection~\ref{subsec:context-windows}),
  the output at time $t$ is a function $F$ of $X_{[t-w+1,t]}$ alone, independent
  of how far back the input window extends. Since the node is deterministic and
  depends only on $(X_{[t-w+1,t]}, c, h)$, any two executions with identical
  values for this triple must produce identical outputs. Hence caching by this
  key is sound.
\end{proof}

\begin{theorem}[Research--production reconciliation]
  \label{thm:recon}
  Assume: (i) all nodes in a DAG $G$ are point-in-time idempotent with correct
  context windows; (ii) node outputs are cached by input window, configuration,
  and code hash; and (iii) mini-batch execution uses tile size $L \ge w(G)$.
  Then research (batch or mini-batch) and production (streaming) executions
  produce identical outputs on overlapping time indices.
\end{theorem}

\begin{proof}
  By Theorem~\ref{thm:two-tiles-suffice}, mini-batch execution with
  $L \ge w(G)$ produces outputs identical to streaming execution at every time
  point in the output tile. By Theorem~\ref{thm:cache}, cached computations are
  deterministic and match direct execution. Combining these results, any
  research execution using mini-batch tiling and caching produces the same
  outputs as production streaming execution, provided all nodes are correctly
  specified with their context windows.
\end{proof}

\begin{remark}[Practical implications for model validation]
  Theorem~\ref{thm:recon} provides a formal guarantee that models developed and
  validated in batch mode (using historical data) will produce identical
  predictions when deployed in streaming mode (processing real-time data),
  provided the model specification is correct. This eliminates a major source
  of discrepancy between research prototypes and production systems. In
  practice, automated testing frameworks verify this property by executing the
  same model in both batch and streaming modes on identical input data and
  asserting bitwise equality of outputs.
\end{remark}

% =============================================================================
\subsection{Knowledge time and causality enforcement}

In real-time systems, data arrives with inherent delays between event time (when
an event occurs) and knowledge time (when the system becomes aware of the event).
DataFlow enforces causality by tracking knowledge time and preventing access to
future data.

\begin{definition}[Knowledge time and causal execution]
  Each row of a stream dataframe has a \emph{knowledge time} $k(u)$ representing
  the wall-clock time at which the system becomes aware of the observation at
  logical time $u$. An execution is \emph{causal} if decisions at wall-clock
  step $s$ depend only on rows with $k(\cdot) \le s$.
\end{definition}

\begin{definition}[Embargo]
  An \emph{embargo} of $\Delta \in \mathbb{N}$ time steps emits the output for
  logical time $t$ no earlier than wall-clock time $t + \Delta$. This provides
  a safety margin to account for late-arriving data.
\end{definition}

\begin{theorem}[Causal safety with embargo]
  \label{thm:embargo}
  If all nodes are point-in-time idempotent and the system enforces causality
  with embargo $\Delta$ such that all late arrivals satisfy
  $k(u) \le u + \Delta$, then the embargoed execution is observationally
  equivalent to the ideal execution with no late data.
\end{theorem}

\begin{proof}
  By assumption, all data with logical time up to $t$ arrives by wall-clock
  time $t + \Delta$. The embargo ensures that outputs for logical time $t$ are
  not emitted before wall-clock time $t + \Delta$. Thus, at emission time, all
  relevant input data is available, and the computation proceeds identically to
  the ideal case where data arrives instantaneously at logical time. Point-in-time
  idempotency guarantees that the output at $t$ depends only on the context
  window $[t-w+1,t]$, independent of when this data actually arrived (as long
  as it arrived before output emission).
\end{proof}

\begin{proposition}[Detection of future-peeking]
  \label{prop:detection}
  If a node uses data from time $t+\delta$ (where $\delta > 0$) when producing
  outputs at time $t$, then there exists an arrival schedule where streaming
  and two-tile mini-batch execution (with $L \ge w(G)$) produce different
  outputs, thus detecting the causality violation.
\end{proposition}

\begin{proof}
  Consider an arrival schedule where data at time $t+\delta$ arrives late, after
  the streaming system has already produced the output at time $t$. In streaming
  mode, the node cannot access the unavailable future data. In mini-batch mode
  with post-facto data, the node has access to all data including $t+\delta$.
  If the node's output at $t$ depends on data from $t+\delta$, the streaming
  and batch outputs will differ, revealing the future-peeking violation.
\end{proof}

\begin{remark}[Probabilistic detection via randomized tiling]
  If future-peeking violations occur with frequency $\rho > 0$ over random tile
  placements, then $n$ independent tiling tests detect at least one violation
  with probability $1 - (1-\rho)^n$. This provides a practical method for
  validating causal correctness: partition historical data into tiles with
  random boundaries, execute in both batch and streaming modes, and verify
  output equality. Discrepancies indicate either causality violations or
  incorrect context window specifications.
\end{remark}

% =============================================================================
\subsection{Batch and streaming execution modes}
\label{subsec:batch-streaming-modes}

The tilability property established in the preceding sections enables DataFlow
to execute the same computational specification in fundamentally different
operational modes while guaranteeing output consistency. This section
formalizes the relationship between batch and streaming execution and
characterizes their performance tradeoffs.

% ................................................................................
\begin{definition}[Batch execution mode]
  In \emph{batch mode}, a node (or DAG) processes historical data over an
  interval $[t_{\min}, t_{\max}]$ by partitioning it into $K$ tiles of
  length $\tau \geq L$:
  \[
    [t_{\min}, t_{\max}]
    = [t_{\min}, t_1] \cup [t_1+1, t_2] \cup \cdots \cup [t_{K-1}+1, t_{\max}],
  \]
  where each tile satisfies $t_k - t_{k-1} \geq \tau$ and tiles may overlap to
  provide sufficient context windows. Execution proceeds by applying the
  computational node(s) to each tile sequentially or in parallel, producing
  outputs that are subsequently concatenated.
\end{definition}

% ................................................................................
\begin{definition}[Streaming execution mode]
  In \emph{streaming mode}, data arrives incrementally over time. At each time
  $t$, the node receives input data on the window $[t-L+1, t]$ (or possibly a
  longer prefix) and produces output at time $t$. The execution advances the
  time index as new observations become available, processing tiles of minimal
  size $\tau = L$ (the context window).
\end{definition}

% ................................................................................
\begin{proposition}[Batch-streaming output equivalence]
  \label{prop:batch-streaming-equiv}
  Let $N$ be a point-in-time idempotent node with context window $L$, and let
  $X^{(1)},\dots,X^{(m)}$ be input streams defined on $[t_{\min}, t_{\max}]$.
  Then for any time $t \in [t_{\min}+L-1, t_{\max}]$:
  \begin{enumerate}
    \item The output $Y(t)$ computed in streaming mode (using tiles of length
      $\tau = L$) coincides with the output computed in batch mode (using any
      tiling with $\tau \geq L$).
    \item Both outputs coincide with the ideal pointwise output defined in
      Section~\ref{subsec:context-windows}.
  \end{enumerate}
\end{proposition}

\begin{proof}
  This follows directly from Theorems~\ref{thm:single-step-tile-correctness}
  and~\ref{thm:two-tiles-suffice}. In streaming mode, at time $t$ the node
  receives input on $[t-L+1,t]$ and produces output $Y_{\text{stream}}(t)$.  By
  Theorem~\ref{thm:single-step-tile-correctness}, this equals the ideal output
  $Y(t)$.

  In batch mode, the time point $t$ lies within some tile or at its boundary.
  By Theorem~\ref{thm:two-tiles-suffice}, as long as the tile (and its
  predecessor, if necessary) provides at least $L$ preceding observations, the
  output at $t$ equals the ideal output $Y(t)$. Hence
  $Y_{\text{stream}}(t) = Y(t) = Y_{\text{batch}}(t)$.
\end{proof}

Proposition~\ref{prop:batch-streaming-equiv} establishes the central guarantee
of DataFlow: a model developed and validated using batch execution on
historical data will produce \emph{identical outputs} when deployed in
streaming mode, provided the model is correctly specified (point-in-time
idempotent with known context window $L$) and causality is preserved.

% ................................................................................
\begin{remark}[Operational tradeoffs between batch and streaming modes]
  \label{rem:batch-streaming-tradeoffs}
  While batch and streaming modes produce equivalent outputs, they exhibit
  distinct performance characteristics:

  \textbf{Batch mode advantages:}
  \begin{itemize}
    \item \emph{Throughput.} Processing large tiles amortizes scheduling
      overhead and enables aggressive vectorization across time, increasing
      data throughput.
    \item \emph{Parallelism.} Independent tiles may be executed concurrently
      across multiple processors or machines.
    \item \emph{Optimization.} Compilers and runtime systems can apply
      optimizations (loop fusion, memory layout transformations) over large
      contiguous data blocks.
  \end{itemize}

  \textbf{Streaming mode advantages:}
  \begin{itemize}
    \item \emph{Latency.} Minimal tile sizes ($\tau = L$) reduce time-to-output
      for individual predictions, critical for low-latency applications.
    \item \emph{Memory footprint.} Processing small tiles reduces working
      memory requirements, enabling execution on resource-constrained devices.
    \item \emph{Causality validation.} Operating with real-time data arrival
      patterns exposes future-peeking errors that may be masked in batch mode.
  \end{itemize}

  \textbf{Hybrid execution:}
  DataFlow supports intermediate tile sizes $L \leq \tau < (t_{\max} - t_{\min})$,
  enabling exploration of the throughput-latency-memory tradeoff space. For
  example, near-real-time systems may batch observations arriving within short
  intervals (e.g., one second) to exploit vectorization while maintaining
  acceptable latency. The optimal tile size $\tau^*$ depends on application
  requirements, data characteristics, and computational resources, and is a
  tunable parameter independent of the model specification.
\end{remark}

% ................................................................................
\begin{remark}[Benefits of unified batch-streaming semantics]
  The equivalence established in Proposition~\ref{prop:batch-streaming-equiv}
  addresses several challenges in time series machine learning:
  \begin{enumerate}
    \item \textbf{Development-production consistency.} Researchers develop and
      validate models using batch execution on historical data (for speed), then
      deploy the identical model specification in streaming mode (for real-time
      predictions), eliminating the need for model reimplementation and the
      associated risk of semantic discrepancies.

    \item \textbf{Debugging production failures.} When a production streaming
      system exhibits unexpected behavior, the exact sequence of inputs can be
      captured and replayed in batch mode (with arbitrarily small tile sizes to
      match streaming execution precisely), enabling systematic debugging in
      controlled environments.

    \item \textbf{Flexible reprocessing.} Historical periods may be reprocessed
      with updated models using efficient batch execution, while forward-going
      predictions continue in streaming mode, without requiring separate
      execution engines.

    \item \textbf{Testing and validation.} The same model specification can be
      subjected to both backtesting (batch execution over historical periods)
      and paper trading (streaming execution with simulated real-time data
      arrival), providing comprehensive validation across execution modalities.
  \end{enumerate}
  These capabilities stem directly from the tilability property and
  point-in-time idempotency requirements imposed on DataFlow computations.
\end{remark}

  \section{DAGs}

% =============================================================================
\subsection{Nodes}

Nodes represent computation that is point-in-time idempotent on tiles of a
streaming dataframe. Tiles have a primary index representing time. In practice,
multiple time points are associated with an element of data, such as event
time, knowledge time, processing time, etc. One notion of time is chosen as
primary, while other notions may remain available as features.

Formally, recall that a (multi-column) stream dataframe is a map
\[
  X : \mathbb{T}_{\le t_{\max}} \times \mathcal{C} \to \mathcal{V},
\]
where $\mathbb{T}_{\le t_{\max}} \subseteq \mathbb{T}$ is a discrete time
prefix, $\mathcal{C}$ is a finite column index set, and $\mathcal{V}$ is a
value space.  We distinguish a \emph{primary time} component $t \in \mathbb{T}$
and may include any secondary time notions (event time, knowledge time,
processing time, etc.) as columns in $\mathcal{C}$.

A tile of length $\tau$ ending at time $t$ is a restriction
$X|_{[t-\tau+1,t]\times \mathcal{C}}$.  As in previous sections, each node is
equipped with a context window length $L \in \mathbb{N}$ and is required to be
point-in-time idempotent with respect to this primary time index.

\begin{definition}[Admissible node function]
  Fix a context window length $L \in \mathbb{N}$.  A (stateless) node function
  on tiles is a family of maps
  \[
    F_{[s,t]} :
      \bigl(\mathcal{V}^{[s,t]\times\mathcal{C}}\bigr)^m
      \longrightarrow
      \bigl(\mathcal{V}^{[s,t]\times\mathcal{C}'}\bigr)^n,
      \quad s \le t,
  \]
  where $m$ is the number of input dataframes, $n$ is the number of output
  dataframes, and $\mathcal{C}'$ is the output column index set.  We say that
  $F$ is \emph{admissible} as a DataFlow node if:
  \begin{enumerate}
    \item (Prefix restriction) $F$ is defined for every finite time interval
      $[s,t]$ contained in the domains of the inputs.
    \item (Point-in-time idempotency) There exists $L$ such that for all
      $t \in \mathbb{T}$, all intervals $[s_1,t]$, $[s_2,t]$ with
      $t-s_1+1 \ge L$, $t-s_2+1 \ge L$, and all inputs
      $X^{(1)},\dots,X^{(m)}$,
      the outputs at time $t$ coincide:
      \[
        F_{[s_1,t]}(X^{(1)}|_{[s_1,t]\times\mathcal{C}},\dots)
        \Big|_{\{t\}\times\mathcal{C}'}
        =
        F_{[s_2,t]}(X^{(1)}|_{[s_2,t]\times\mathcal{C}},\dots)
        \Big|_{\{t\}\times\mathcal{C}'}.
      \]
  \end{enumerate}
  Any such $F$ induces a pointwise, primary-time semantics that is independent
  of the length of the input window beyond $L$.
\end{definition}

A key feature of DataFlow is that any user-defined function that is
point-in-time idempotent is admissible in a computational node. This
flexibility makes it easy to convert pre-existing code into a DataFlow
pipeline. It also enables users to write and deploy pipelines using the
computational abstractions natural for the problem at hand. In other words,
user-defined functions need not be decomposed into computational primitives in
order to be executed in DataFlow.

Formally, let $\mathsf{UF}$ be the class of user-defined functions
\[
  f : \bigl(\mathcal{V}^{[s,t]\times\mathcal{C}}\bigr)^m
      \longrightarrow
      \bigl(\mathcal{V}^{[s,t]\times\mathcal{C}'}\bigr)^n
\]
that satisfy the point-in-time idempotency condition for some context window
length $L_f$.  Then every $f \in \mathsf{UF}$ defines an admissible node in the
sense above, with node context window $L_f$.

Nodes may have state and may call out to external resources. For example, a
machine learning node may operate in a ``fit'' mode, where it learns from tiles
that are processed and stores a trained model upon completion. When operating
in a ``predict'' mode, the node loads into memory or calls out via an API the
trained model and emits the output of the model as applied to the tiles seen
during prediction.

We can formalize such behavior as a \emph{stateful} node.

\begin{definition}[Stateful node with modes]
  A stateful node is given by:
  \begin{itemize}
    \item a state space $\mathcal{S}$,
    \item a set of modes $\mathcal{M}$ (e.g.\ $\mathcal{M} = \{\mathrm{fit},
          \mathrm{predict}\}$),
    \item for each mode $m \in \mathcal{M}$ and time interval $[s,t]$,
      a transition map
      \[
        F^{(m)}_{[s,t]} :
          \bigl(\mathcal{V}^{[s,t]\times\mathcal{C}}\bigr)^m
          \times \mathcal{S}
          \longrightarrow
          \bigl(\mathcal{V}^{[s,t]\times\mathcal{C}'}\bigr)^n
          \times \mathcal{S},
      \]
      mapping inputs and current state to outputs and next state.
  \end{itemize}
  We say that $F^{(m)}$ is \emph{point-in-time idempotent in the data
  arguments} if, for each fixed state $s \in \mathcal{S}$, the map
  $(X^{(1)},\dots,X^{(m)}) \mapsto F^{(m)}_{[s,t]}(X^{(1)},\dots,X^{(m)},s)$ is
  an admissible node in the sense above, with some context window length
  $L_{m}$ independent of $s$.
\end{definition}

In ``fit'' mode, the node applies $F^{(\mathrm{fit})}$, updating its internal
state (e.g.\ model parameters).  In ``predict'' mode, the node applies
$F^{(\mathrm{predict})}$ using the frozen state produced during fitting.  As
long as both transition maps are point-in-time idempotent in the data
arguments, the node remains admissible in DataFlow.

Among the simplest nodes are the single-source single-ouput (SISO) nodes, which
accept one streaming dataframe as input and emit one streaming dataframe as
output.

\begin{definition}[SISO, source, and sink nodes]
  Let a node have $m$ input streams and $n$ output streams.
  \begin{itemize}
    \item The node is \emph{SISO} if $m = n = 1$.
    \item The node is a \emph{source} if $m = 0$ and $n \ge 1$.
    \item The node is a \emph{sink} if $m \ge 1$ and $n = 0$.
  \end{itemize}
\end{definition}

Source nodes are those that do not receive input from another node; examples
include those that call out to external databases to generate output and those
that generate synthetic data for testing or verification purposes. Sink nodes
are those that do not emit output; examples include nodes that serialize
results to disk and those that simply call out to an external service. An
example of a node with multiple inputs is a join node, and an example of a node
with multiple outputs is a splitter node that splits columns of its input into
multiple sets across multiple streaming dataframes.

% =============================================================================
\subsection{DAGs}

In DataFlow, computational nodes are organized into a directed acyclic graph
(DAG). In this sense, DataFlow may be interpreted as a dataflow programming
paradigm, as the data (streaming dataframes) flows between nodes (computational
operations). The DAG structure supports parallelism at the coarsest level. In
particular, if neither node $A$ nor node $B$ is an ancestor of the other, then
the two nodes may perform computation in parallel.

Formally, a DataFlow DAG is a directed acyclic graph
$G = (V,E)$ whose vertices $v \in V$ are nodes (as above) and whose edges
$(u,v) \in E$ connect outputs of $u$ to inputs of $v$.  We write
$u \prec v$ if there exists a directed path from $u$ to $v$; this induces a
partial order on $V$.  Two nodes $A,B \in V$ are \emph{concurrent} if neither
$A \prec B$ nor $B \prec A$ holds, in which case they may be scheduled in
parallel in any topological ordering.

A simple DAG with 5 nodes, including one source and one sink, is depicted in
Figure \ref{fig:simple_dag}. Note that nodes $C_1$ and $C_2$ may be computed in
parallel, as they have no dependency relationship.

% rendered_images:begin
%   ```graphviz
%   digraph SimpleDAG {
%       rankdir=LR;
%       nodesep=0.6;
%       ranksep=0.8;
% 
%       node [shape=box, style="rounded,filled", fontname="Helvetica", fontsize=12, penwidth=1.4];
% 
%       Source [label="Source\n(Load Data)", fillcolor="#C6A6F4"];
%       C1 [label="C_1\n(Compute\\nFeature A)", fillcolor="#A6E7F4"];
%       C2 [label="C_2\n(Compute\\nFeature B)", fillcolor="#A6E7F4"];
%       C3 [label="C_3\n(Merge\\nFeatures)", fillcolor="#F4E6A6"];
%       Sink [label="Sink\\n(Output)", fillcolor="#C6F4A6"];
% 
%       Source -> C1 [penwidth=1.5];
%       Source -> C2 [penwidth=1.5];
%       C1 -> C3 [penwidth=1.5];
%       C2 -> C3 [penwidth=1.5];
%       C3 -> Sink [penwidth=1.5];
% 
%       // Highlight parallel nodes with same rank
%       {rank=same; C1; C2}
%   }
%   ```
%   caption=Simple DAG with concurrent execution. Nodes $C_1$ and $C_2$ can execute in parallel since neither depends on the other.
%   label=fig:simple_dag
% rendered_images:end
% render_images:begin
\begin{figure}[!ht]
  \includegraphics[width=\linewidth]{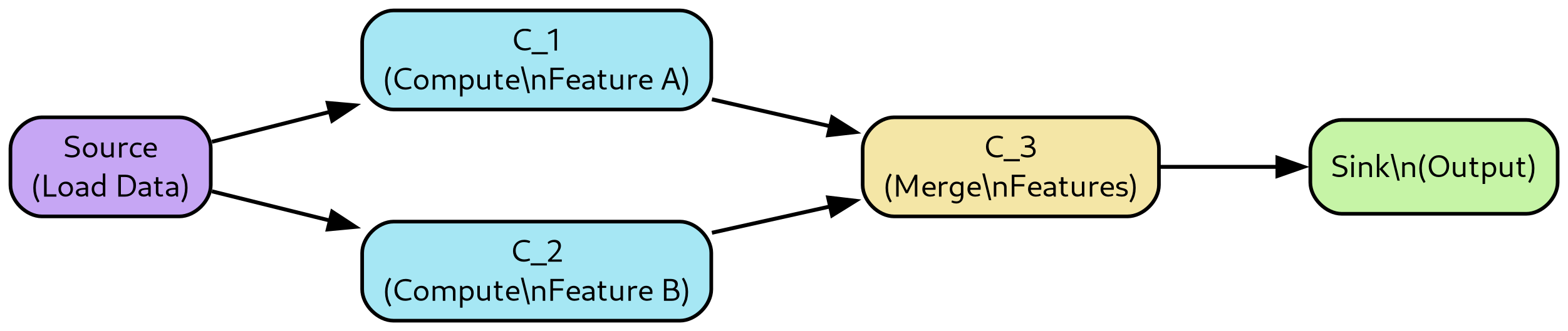}
  \caption{Simple DAG with concurrent execution. Nodes $C_1$ and $C_2$ can execute in parallel since neither depends on the other.}
  \label{fig:simple_dag}
\end{figure}
% render_images:end

To determine the context window requirement of a DAG, we first formalize the
context window of a path.

Let each node $v \in V$ be point-in-time idempotent with a (minimal) context
window length $w_v \in \mathbb{N}$.  Consider a path
$p_{i_0,i_1,\dots,i_n} = (v_{i_0},v_{i_1},\dots,v_{i_n})$ from a source node
$v_{i_0}$ to a sink node $v_{i_n}$.

\begin{definition}[Path context window]
  The \emph{context window} $w(p_{i_0,\dots,i_n})$ of a path
  $p_{i_0,\dots,i_n}$ is defined as the minimal integer $L$ such that the
  outputs at time $t$ of $v_{i_n}$ depend only on the inputs at the source
  node $v_{i_0}$ on the time interval $[t-L+1,t]$, for all $t$.  Equivalently:
  if two executions of the DAG agree on all source-input stream values on
  $[t-L+1,t]$, then they produce identical outputs at time $t$ at the sink
  along that path.
\end{definition}

The following proposition shows that the formula used in DataFlow indeed
computes this quantity.

\begin{proposition}[Closed form for path context window]
  \label{prop:path-window}
  Let $p_{i_0,i_1,\dots,i_n}$ be a source-to-sink path in $G$, and let
  $w_{i_k}$ be the context window of node $v_{i_k}$ for $k=0,\dots,n$.  Then
  the path context window satisfies
  \[
    w(p_{i_0, i_1, \ldots, i_n})
    = 1 + \sum_{k = 0}^n (w_{i_k} - 1)
    = 1 - n + \sum_{k = 0}^n w_{i_k}.
  \]
\end{proposition}

\begin{proof}
  We argue inductively on the length $n$ of the path.

  \emph{Base case $n=0$.}  The path consists of a single node $v_{i_0}$ with
  context window $w_{i_0}$.  By definition of context window, the outputs at
  time $t$ depend only on inputs on $[t-w_{i_0}+1,t]$, and this dependence is
  minimal.  Thus $w(p_{i_0}) = w_{i_0}$, which agrees with the formula
  $1 + (w_{i_0}-1)$.

  \emph{Inductive step.}  Suppose the formula holds for all paths of length
  $n-1$, and consider a path
  $p_{i_0,\dots,i_{n-1},i_n}$.  Let $p'$ denote its prefix
  $(v_{i_0},\dots,v_{i_{n-1}})$, and write $w(p')$ for the context window of
  $p'$.  By the inductive hypothesis,
  \[
    w(p') = 1 + \sum_{k=0}^{n-1} (w_{i_k}-1).
  \]

  Consider the output of $v_{i_n}$ at time $t$.  This depends on the outputs of
  $v_{i_{n-1}}$ on the time interval $[t-w_{i_n}+1,t]$.  For each
  $u \in [t-w_{i_n}+1,t]$, the output of $v_{i_{n-1}}$ at time $u$ depends only
  on the inputs at $v_{i_0}$ on the interval
  \[
    [u - w(p') + 1,\,u],
  \]
  by definition of $w(p')$.  The earliest time in this union, as $u$ ranges
  over $[t-w_{i_n}+1,t]$, is obtained by taking $u = t - w_{i_n} + 1$ and then
  its left endpoint:
  \[
    t - w_{i_n} + 1 - w(p') + 1
    = t - \bigl( w_{i_n} + w(p') - 2 \bigr).
  \]
  Hence the outputs of $v_{i_n}$ at time $t$ depend only on inputs at $v_{i_0}$
  on the interval
  \[
    \Bigl[ t - \bigl( w_{i_n} + w(p') - 2 \bigr),\, t \Bigr],
  \]
  which has length
  \[
    1 + \bigl( w_{i_n} + w(p') - 2 \bigr)
    = w_{i_n} + w(p') - 1.
  \]
  Thus
  \[
    w(p_{i_0,\dots,i_n})
    = w_{i_n} + w(p') - 1.
  \]
  Substituting the inductive expression for $w(p')$ gives
  \[
    w(p_{i_0,\dots,i_n})
    = w_{i_n} + \Bigl( 1 + \sum_{k=0}^{n-1} (w_{i_k}-1) \Bigr) - 1
    = 1 + \sum_{k=0}^n (w_{i_k}-1),
  \]
  as desired.
\end{proof}

Thus the context window $w$ of path $p_{i_0, i_1, \ldots, i_n}$ is given by
\begin{align*}
    w(p_{i_0, i_1, \ldots, i_n}) &= 1 + \sum_{k = 0}^n (w_{i_k} - 1) \\
    &= 1 - n + \sum_{k = 0}^n w_{i_k}.
\end{align*}
Note that the context window of a path of nodes each of context window $1$ is
itself $1$. On the other hand, if $p$ is a path of $n$ context window-$2$
nodes, then the context window of $p$ is $n + 1$, in accordance with the
formula.

This is due to the fact that, in proceeding backward along the path, the
context window requirement for the current node is the union of the context
window requirements of the output node; at each step, an additional time point
enters into the context window requirement.

The following worked example illustrates the path context window calculation
with concrete values.

\begin{figure}[ht]
\centering
\textbf{Example: Path Context Window Calculation}

\vspace{0.3cm}

Consider a linear path $v_1 \to v_2 \to v_3$ with context windows $w_1=2$, $w_2=3$, $w_3=2$.

\vspace{0.2cm}

\begin{minipage}{0.45\textwidth}
\centering
\begin{tabular}{|l|c|c|}
\hline
\textbf{Node} & \textbf{Context} & \textbf{Cumulative} \\
\hline
$v_1$ & $w_1 = 2$ & $2$ \\
\hline
$v_2$ & $w_2 = 3$ & $2 + (3-1) = 4$ \\
\hline
$v_3$ & $w_3 = 2$ & $4 + (2-1) = 5$ \\
\hline
\end{tabular}
\end{minipage}
\hfill
\begin{minipage}{0.5\textwidth}
\begin{tikzpicture}[scale=0.9]
  % Nodes
  \node[draw, circle, minimum size=1cm] (v1) at (0, 0) {$v_1$};
  \node[draw, circle, minimum size=1cm] (v2) at (2.5, 0) {$v_2$};
  \node[draw, circle, minimum size=1cm] (v3) at (5, 0) {$v_3$};

  % Arrows
  \draw[->, thick] (v1) -- (v2);
  \draw[->, thick] (v2) -- (v3);

  % Context annotations
  \node[above, font=\small] at (0, 0.7) {$w=2$};
  \node[above, font=\small] at (2.5, 0.7) {$w=3$};
  \node[above, font=\small] at (5, 0.7) {$w=2$};
\end{tikzpicture}
\end{minipage}

\vspace{0.4cm}

\textbf{Timeline showing data dependencies:}

\vspace{0.2cm}

\begin{tikzpicture}[scale=1.0]
  % Time axis
  \draw[thick,->] (0,0) -- (12,0) node[right] {time};

  % Time points
  \foreach \i/\t in {0/$t_{-4}$, 2/$t_{-3}$, 4/$t_{-2}$, 6/$t_{-1}$, 8/$t_0$} {
    \draw[thick] (\i, -0.1) -- (\i, 0.1);
    \node[below] at (\i, -0.2) {\small \t};
  }

  % v1 context window
  \fill[red!20] (0, 0.5) rectangle (4, 1.2);
  \node at (2, 0.85) {\small $v_1$ needs 2 steps};

  % v2 adds more context
  \fill[orange!20] (0, 1.4) rectangle (6, 2.1);
  \node at (3, 1.75) {\small $v_2$ needs 4 steps total};

  % v3 adds final context
  \fill[green!20] (0, 2.3) rectangle (8, 3.0);
  \node at (4, 2.65) {\small $v_3$ needs 5 steps total};

  % Output point
  \draw[ultra thick, blue] (8, 0) -- (8, 3.2);
  \node[above, blue] at (8, 3.2) {Output};

\end{tikzpicture}

\vspace{0.2cm}

The path context window is $w(v_1 \to v_2 \to v_3) = 1 + (2-1) + (3-1) + (2-1) = 5$.
To produce output at time $t_0$, the path requires input data back to time $t_{-4}$.

%\caption{Worked example of path context window calculation. Each node adds $(w_i - 1)$ to the cumulative context requirement. The timeline shows how context accumulates: $v_1$ needs 2 time steps, then $v_2$ extends this to 4, and finally $v_3$ extends to 5 time steps total.}
%\label{fig:path_context_example}
\end{figure}

Figure~\ref{fig:linear_chain_context} illustrates this accumulation with a
concrete example of a three-node linear chain.

% rendered_images:begin
% ```graphviz
% digraph LinearChainContext {
%     rankdir=LR;
%     nodesep=0.8;
%     ranksep=1.0;
% 
%     node [shape=box, style="rounded,filled", fontname="Helvetica", fontsize=12, penwidth=1.4];
% 
%     v1 [label="v_1\\nw=2", fillcolor="#F4A6A6"];
%     v2 [label="v_2\\nw=3", fillcolor="#F4C6A6"];
%     v3 [label="v_3\\nw=2", fillcolor="#F4E6A6"];
% 
%     v1 -> v2 [label="  ", penwidth=1.5];
%     v2 -> v3 [label="  ", penwidth=1.5];
% 
%     // Add annotations showing cumulative context
%     label1 [shape=plaintext, label="Path context:\\nw(v_1) = 2", fontsize=10];
%     label2 [shape=plaintext, label="Path context:\\nw(v_1→v_2) = 2+(3-1) = 4", fontsize=10];
%     label3 [shape=plaintext, label="Path context:\\nw(v_1→v_2→v_3) = 4+(2-1) = 5", fontsize=10];
% 
%     v1 -> label1 [style=invis];
%     v2 -> label2 [style=invis];
%     v3 -> label3 [style=invis];
% 
%     {rank=same; v1; label1}
%     {rank=same; v2; label2}
%     {rank=same; v3; label3}
% }
% ```
% caption=Linear chain showing context window accumulation. Each node has its local context window $w$, and the path context window accumulates according to the formula $w(p) = 1 + \sum (w_i - 1)$.
% label=fig:linear_chain_context
% rendered_images:end
% render_images:begin
\begin{figure}[!ht]
  \includegraphics[width=\linewidth]{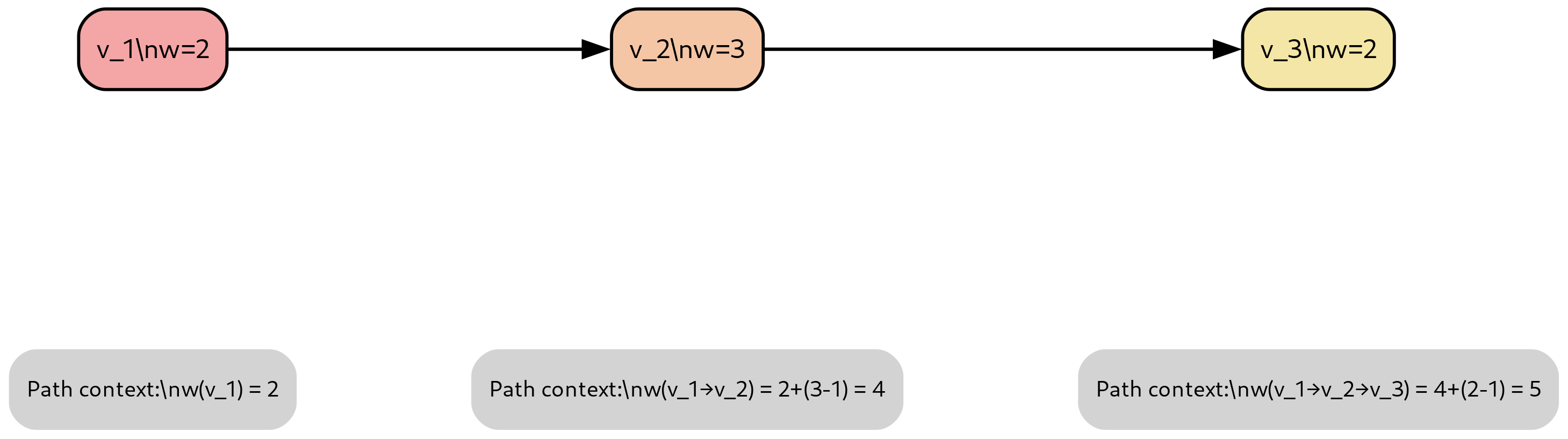}
  \caption{Linear chain showing context window accumulation. Each node has its local context window $w$, and the path context window accumulates according to the formula $w(p) = 1 + \sum (w_i - 1)$.}
  \label{fig:linear_chain_context}
\end{figure}
% render_images:end

The context window $w$ of a DAG $G$ is the maximum of the context windows of
all paths in $G$ connecting a source node to a sink node:
\begin{equation}
    w(G) = \max_{p \textrm{ a source-to-sink path in } G} w(p).
\end{equation}

\begin{proposition}[Graph-level context window]
  \label{prop:graph-window}
  Suppose every node $v \in V$ in a DAG $G$ is point-in-time idempotent with
  context window $w_v$.  Then:
  \begin{enumerate}
    \item For each source-to-sink path $p$ and each time $t$, the outputs at
      the sink node along $p$ at time $t$ depend only on the values of the
      source inputs on $[t-w(p)+1,t]$.
    \item The graph context window $w(G)$ is the minimal integer $L$ such that
      for every sink node and time $t$, the sink outputs at $t$ depend only on
      the source inputs on $[t-L+1,t]$.
  \end{enumerate}
\end{proposition}

\begin{proof}
  (1) is a direct consequence of Proposition~\ref{prop:path-window}, applied to
  each path separately.  (2) follows because any influence from sources to a
  given sink must travel along some source-to-sink path.  The largest window
  required among those paths is thus both sufficient (no path requires more
  than $w(G)$ time steps) and necessary (for any smaller $L<w(G)$, there exists
  a path whose sink output at time $t$ can be changed by altering source
  inputs earlier than $t-L+1$).
\end{proof}

Recall that the graph context window size is a lower bound for the minimum tile
size. In many practical cases, conservative upper bounds for the graph context
window size may be produced, which already may suffice for tiling purposes.
Whether the estimated bounds suffice will depend upon the data, computational
resources, and operational mode.

Formally, let $\tau$ be a chosen tile length.  We require
\[
  \tau \;\ge\; w(G)
\]
to ensure that each tile contains at least as much history as the worst-case
graph context window.  Under this condition, we can view the entire DAG
restricted to its sources and sinks as a single (composite) node with context
window $w(G)$.

Figure~\ref{fig:financial_pipeline} illustrates a realistic financial trading
pipeline implemented as a DataFlow DAG, demonstrating the complexity of
production systems and the importance of context window management.

\begin{figure}[!ht]
  \includegraphics[width=\linewidth]{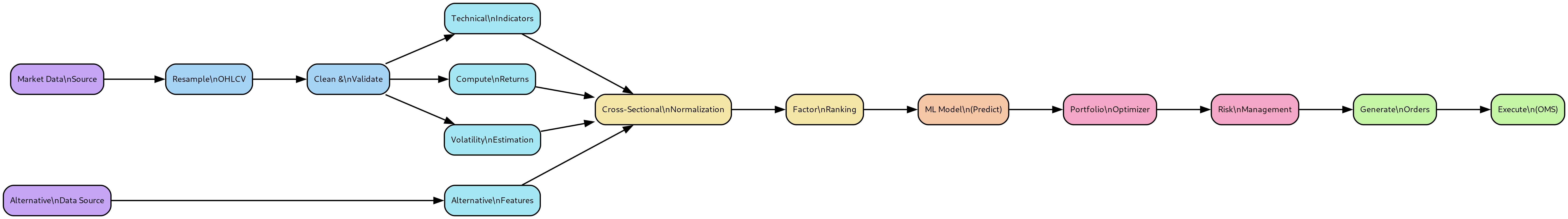}
  \caption{Complex financial trading pipeline DAG. The system processes market data through feature engineering, machine learning prediction, portfolio optimization, and order execution. Cross-sectional nodes (Normalize, Rank) require access to all assets simultaneously, constraining spatial tiling strategies. The critical path determines minimum system latency.}
  \label{fig:financial_pipeline}
\end{figure}
% render_images:end

In any case, once the tile size is determined, achieving tile-level idempotency
is straightforward: two input tiles suffice to achieve tile-level idempotency
at the graph level (see Table \ref{table:minibatch} for a pictoral
representation). This is a consequence of how we have defined the graph context
window and minimum tile size.

\begin{theorem}[Two tiles suffice at the graph level]
  Suppose every node in $G$ is point-in-time idempotent with context window
  $w_v$, and let $w(G)$ be the resulting graph context window.  Let $\tau$ be a
  tile length satisfying $\tau \ge w(G)$.  Then for any time $t$, if we provide
  the DAG with two consecutive tiles of source inputs covering
  $[t-2\tau+1,t]$, the outputs produced at all sink nodes on the second tile
  $[t-\tau+1,t]$ coincide with the ideal streaming outputs at those times, and
  they are invariant under any further extension of the input window into the
  past.
\end{theorem}

\begin{proof}
  Fix a sink node, a time $u \in [t-\tau+1,t]$, and consider any source-to-sink
  path $p$ that reaches this sink.  By Proposition~\ref{prop:graph-window}, the
  output at time $u$ along $p$ depends only on the source inputs on
  $[u-w(p)+1,u]$.  Since $w(p) \le w(G)$ and $\tau \ge w(G)$, we have
  \[
    [u-w(G)+1,u] \subseteq [u-\tau+1,u]
    \subseteq [t-2\tau+1,t],
  \]
  so $[u-w(p)+1,u] \subseteq [t-2\tau+1,t]$.  Thus the two-tile window
  $[t-2\tau+1,t]$ contains all source-input times that can influence the output
  at $u$ along $p$.

  Since this holds for every source-to-sink path to the sink, the output at
  time $u$ at the sink is determined entirely by the restriction of the sources
  to the interval $[t-2\tau+1,t]$, and hence coincides with the ideal streaming
  output (which is defined as the limit over longer and longer prefixes).  The
  invariance under further extension into the past follows from the same
  argument: any additional history prior to $t-2\tau+1$ lies strictly before
  $u-w(G)+1$ and thus cannot affect the output at time $u$.
\end{proof}

\begin{proposition}[Necessity of $\tau \ge w(G)$]
  \label{prop:tau-necessity}
  If $\tau < w(G)$, there exist DAGs and input sequences for which two-tile
  computation on the output tile $I = [t-\tau+1,t]$ disagrees with streaming
  execution.
\end{proposition}

\begin{proof}
  We construct a counterexample. Let $G$ be a linear chain of nodes
  $v_1 \to v_2 \to \cdots \to v_n$, each with context window $w_i \ge 2$,
  chosen such that $w(G) = 1 + \sum_{i=1}^n (w_i - 1) > \tau$.

  Consider an input stream where all values are zero except for a single value
  of $1$ at time $t_0 = t - w(G) + 1$. In streaming execution with full
  history, this input propagates through the chain and affects the output at
  time $t$. Specifically, by the definition of the graph context window, the
  output at time $t$ depends on inputs back to time $t - w(G) + 1 = t_0$.

  Now consider two-tile execution with tile length $\tau < w(G)$. The two tiles
  cover $[t-2\tau+1, t]$. Since $2\tau < 2w(G)$, we have
  $t-2\tau+1 > t-2w(G)+1$. In particular, if we choose the input spike at
  $t_0 = t-w(G)+1$, then $t_0 > t-2\tau+1$ for $\tau < w(G)$ sufficiently
  small. However, the two-tile window $[t-2\tau+1,t]$ does not extend back to
  $t_0$, so the spike is not visible to the two-tile computation. The two-tile
  output at time $t$ will be zero, while the streaming output is nonzero,
  demonstrating the disagreement.

  More formally, choose $\tau = \lfloor w(G)/2 \rfloor$ and set the spike at
  $t_0 = t - w(G) + 1$. Then $t-2\tau+1 = t - 2\lfloor w(G)/2 \rfloor + 1
  \ge t - w(G) + 2 > t_0$, so the spike falls outside the two-tile window,
  establishing the claim.
\end{proof}

\begin{remark}[Practical implications of tile size selection]
  Proposition~\ref{prop:tau-necessity} establishes that $\tau \ge w(G)$ is not
  merely sufficient but also necessary for correctness. In practice, this means
  that the tile size must be chosen based on the DAG's context window. Choosing
  $\tau < w(G)$ leads to subtle errors where outputs appear correct for most
  time points but silently fail when context dependencies span tile boundaries.
  DataFlow's testing framework detects such violations by comparing tiled and
  untiled execution across multiple random tile boundaries.
\end{remark}

% =============================================================================
\subsection{Separation of DAG topology and DAG configuration}

In DataFlow, the DAG topology is separated from the configuration of the
nodes of the DAG, which we refer to as ``DAG configuration''. For example, a
machine learning node that fits and utilizes an autoregressive model may expose
hyperparameters such as the order of the model (i.e., the number of lags to
include), the technique used to learn parameters, and the model training
period. A more pedestrian use case, common in practice, is to allow the user to
define admissible node-level input and output column names of dataframes.

We formalize this separation as follows.

\begin{definition}[DAG topology and configuration]
  A \emph{DAG topology} is a directed acyclic graph $G = (V,E)$ together with,
  for each vertex $v \in V$,
  \begin{itemize}
    \item its input and output arities $(m_v,n_v)$, and
    \item its abstract node type (e.g.\ SISO, join, splitter, ML node), which
          fixes the input/output schema and the admissibility constraints
          (such as point-in-time idempotency and schema compatibility).
  \end{itemize}
  A \emph{DAG configuration} is a map
  \[
    \kappa : V \longrightarrow \mathsf{Conf},
  \]
  assigning to each node $v$ a configuration object $\kappa(v)$ that
  instantiates its type (e.g.\ concrete hyperparameters, column names, or
  pointers to user-defined functions).
\end{definition}

A more abstract case is where the selection of computational operations is
specified in the DAG configuration itself. For example, a computational node
$N$ may wrap a user-defined function conforming to some constraints on input
and output schema. The wrapping may be implemented so that a pointer to the
user-defined function is specified in the node configuration, thereby allowing
the user to supply any function $f$ which conforms to the constraints upon DAG
configuration (or reconfiguration).

Formally, let $\mathsf{UF}_v$ denote the set of all user-defined functions
admissible at node $v$ (e.g.\ those that satisfy the node’s schema constraints
and point-in-time idempotency).  Then a node configuration may include a choice
$f_v \in \mathsf{UF}_v$, and the overall DAG configuration is an element of the
product space
\[
  \prod_{v \in V} \bigl( \mathsf{Conf}_v \times \mathsf{UF}_v \bigr),
\]
where $\mathsf{Conf}_v$ captures other node-specific parameters (such as lag
order or training period for a model).

\begin{remark}[Topology vs.\ configuration and context window]
  The DAG topology $(V,E)$ encodes data dependencies and hence determines the
  way in which node-level context windows $\{w_v\}_{v\in V}$ combine into the
  graph-level context window $w(G)$.  The DAG configuration $\kappa$ selects
  concrete admissible functions at each node but does not change the topology.
  As long as the configured functions preserve the same node-level context
  windows, the graph context window $w(G)$ remains unchanged when the DAG is
  reconfigured.
\end{remark}

Associated with each node is a user-defined temporal minimum \emph{context
window}, which is used to impose a consistency requirement upon the
computational node. Namely, provided the data window exceeds the context
window, the latest-in-time output of the node must not depend upon the length
of window provided. In other words, for the latest output, the node acts in an
idempotent way over any data window of the streaming dataframe that includes
the context window. We call this property \emph{point-in-time idempotency}.

% =============================================================================
\subsection{Performance bounds and scheduling}

The explicit DAG structure in DataFlow enables analysis of performance bounds
and scheduling opportunities. This section establishes fundamental limits on
latency and parallelism.

\begin{theorem}[Critical-path latency lower bound]
  \label{thm:crit-path}
  Let $G = (V,E)$ be a DAG and let each node $v \in V$ have service time
  $\tau(v) > 0$ (the time required to process a single tile). The minimum
  possible latency to produce an output at a sink node is lower bounded by
  \[
    L^* = \max_{p \text{ source-to-sink path in } G} \sum_{v \in p} \tau(v),
  \]
  the weighted length of the critical path in $G$.
\end{theorem}

\begin{proof}
  Consider any source-to-sink path $p = (v_{i_0}, v_{i_1}, \dots, v_{i_n})$
  where $v_{i_0}$ is a source and $v_{i_n}$ is a sink. By the dependency
  structure of the DAG, node $v_{i_k}$ cannot begin computation until all its
  predecessors have completed. In particular, along path $p$, node $v_{i_k}$
  cannot start until node $v_{i_{k-1}}$ has finished, which takes time at least
  $\tau(v_{i_{k-1}})$.

  Thus, the total time to compute the output at $v_{i_n}$ following path $p$
  is at least
  \[
    \sum_{k=0}^{n} \tau(v_{i_k}) = \sum_{v \in p} \tau(v).
  \]
  Since the sink output depends on at least one source-to-sink path, and
  computation must proceed sequentially along any such path, the minimum
  latency is bounded below by the maximum path length. This maximum is the
  critical path.

  Note that this bound is tight: a scheduler that processes nodes in
  topological order, with infinite parallelism across independent nodes,
  achieves exactly this latency.
\end{proof}

\begin{remark}[Implications for system design]
  Theorem~\ref{thm:crit-path} establishes that the critical path length
  $L^*$ represents a fundamental lower bound on latency, independent of
  available computational resources. Even with infinite parallelism, latency
  cannot be reduced below $L^*$. This has several practical implications:

  \begin{itemize}
    \item \textbf{DAG restructuring for latency reduction.} If a system fails to
      meet latency requirements, the only solution is to restructure the DAG to
      reduce the critical path length—either by optimizing slow nodes or by
      reorganizing dependencies to create shorter paths.

    \item \textbf{Resource allocation.} Nodes on the critical path are
      candidates for priority resource allocation (e.g., GPU acceleration,
      caching), as improvements to non-critical nodes do not reduce overall
      latency.

    \item \textbf{Load balancing.} In distributed execution, nodes on the
      critical path should be assigned to the fastest available workers to
      minimize end-to-end latency.
  \end{itemize}

  Figure~\ref{fig:critical_path_gantt} illustrates these concepts with a Gantt
  chart comparing sequential and parallel execution.
\end{remark}

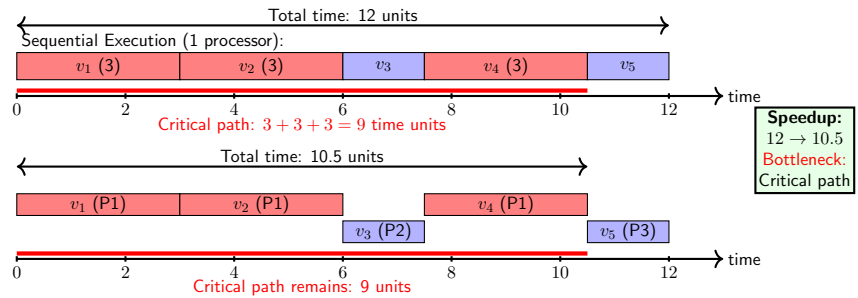
\begin{figure}[ht]
\centering
\begin{tikzpicture}[scale=0.85]
  % Title for sequential execution
  \node[anchor=west, font=\bfseries] at (0, 6) {Sequential Execution (1 processor):};

  % Time axis for sequential
  \draw[thick,->] (0, 5) -- (13, 5);
  \foreach \i in {0,2,...,12} {
    \draw[thick] (\i, 4.95) -- (\i, 5.05);
    \node[below, font=\tiny] at (\i, 4.95) {$\i$};
  }
  \node[right] at (13, 5) {time};

  % Sequential execution boxes (critical path in red)
  \draw[fill=red!50] (0, 5.3) rectangle (3, 5.8) node[midway] {$v_1$ (3)};
  \draw[fill=red!50] (3, 5.3) rectangle (6, 5.8) node[midway] {$v_2$ (3)};
  \draw[fill=blue!30] (6, 5.3) rectangle (7.5, 5.8) node[midway, font=\small] {$v_3$};
  \draw[fill=red!50] (7.5, 5.3) rectangle (10.5, 5.8) node[midway] {$v_4$ (3)};
  \draw[fill=blue!30] (10.5, 5.3) rectangle (12, 5.8) node[midway, font=\small] {$v_5$};

  % Critical path indicator
  \draw[ultra thick, red] (0, 5.1) -- (10.5, 5.1);
  \node[below, red, font=\small] at (5.25, 4.7) {Critical path: $3+3+3=9$ time units};

  % Total time
  \draw[<->, thick] (0, 6.3) -- (12, 6.3);
  \node[above, font=\small] at (6, 6.3) {Total time: 12 units};

  % Title for parallel execution
  \node[anchor=west, font=\bfseries] at (0, 3) {Parallel Execution (3 processors):};

  % Time axis for parallel
  \draw[thick,->] (0, 2) -- (13, 2);
  \foreach \i in {0,2,...,12} {
    \draw[thick] (\i, 1.95) -- (\i, 2.05);
    \node[below, font=\tiny] at (\i, 1.95) {$\i$};
  }
  \node[right] at (13, 2) {time};

  % Processor 1 (critical path)
  \draw[fill=red!50] (0, 2.8) rectangle (3, 3.2) node[midway] {$v_1$ (P1)};
  \draw[fill=red!50] (3, 2.8) rectangle (6, 3.2) node[midway] {$v_2$ (P1)};
  \draw[fill=red!50] (7.5, 2.8) rectangle (10.5, 3.2) node[midway] {$v_4$ (P1)};

  % Processor 2 (parallel work)
  \draw[fill=blue!30] (6, 2.3) rectangle (7.5, 2.7) node[midway, font=\small] {$v_3$ (P2)};

  % Processor 3 (parallel work)
  \draw[fill=blue!30] (10.5, 2.3) rectangle (12, 2.7) node[midway, font=\small] {$v_5$ (P3)};

  % Critical path indicator
  \draw[ultra thick, red] (0, 2.1) -- (10.5, 2.1);
  \node[below, red, font=\small] at (5.25, 1.7) {Critical path remains: 9 units};

  % Total time
  \draw[<->, thick] (0, 3.7) -- (10.5, 3.7);
  \node[above, font=\small] at (5.25, 3.7) {Total time: 10.5 units};

  % Speedup annotation
  \node[draw, thick, fill=green!10, align=center] at (14.5, 4) {
    \textbf{Speedup:} \\
    $12 \to 10.5$ \\
    \textcolor{red}{Bottleneck:} \\
    Critical path
  };

\end{tikzpicture}
\caption{Critical path performance comparison. Top: Sequential execution where all nodes run serially, taking 12 time units. Bottom: Parallel execution with 3 processors reduces total time to 10.5 units by executing $v_3$ and $v_5$ (blue) in parallel with critical path nodes (red). However, the critical path of 9 units remains the fundamental bottleneck. Nodes on the critical path ($v_1 \to v_2 \to v_4$) determine minimum latency regardless of parallelism.}
\label{fig:critical_path_gantt}
\end{figure}

\begin{remark}[Parallelism and DAG width]
  \label{rem:parallelism-ceiling}
  The maximum instantaneous parallelism achievable in a DAG is upper bounded by
  the width of the DAG, defined as the size of the largest antichain (set of
  pairwise incomparable nodes in the partial order induced by the DAG).

  For a DAG with $n$ nodes, total work $W = \sum_{v \in V} \tau(v)$, and
  critical path length $L^*$, standard DAG scheduling results establish:
  \begin{itemize}
    \item The optimal makespan with $P$ processors is at least
      $\max(L^*, W/P)$.
    \item Greedy topological scheduling achieves makespan at most
      $L^* + (W-L^*)/P$.
  \end{itemize}

  In DataFlow, these bounds apply both to intra-tile parallelism (executing
  multiple nodes on the same time slice) and to inter-tile parallelism
  (processing multiple tiles concurrently).
\end{remark}

\begin{definition}[DAG width and antichain]
  An \emph{antichain} in a DAG $G = (V,E)$ is a set $A \subseteq V$ of nodes
  such that for any distinct $u,v \in A$, neither $u \prec v$ nor $v \prec u$
  (i.e., they are incomparable in the partial order). The \emph{width} of $G$
  is the size of the largest antichain.
\end{definition}

\begin{proposition}[Maximum parallelism bound]
  Let $G$ have width $w$. Then at most $w$ nodes can execute in parallel at any
  given time step in a valid topological schedule.
\end{proposition}

\begin{proof}
  At any time step, the set of nodes currently executing forms an antichain
  (since if $u \prec v$, then $v$ cannot start until $u$ completes). By
  definition, no antichain can have size exceeding $w$.
\end{proof}

\begin{remark}[Amdahl's law for DAGs]
  Even if most nodes in a DAG can be parallelized, the critical path imposes a
  serial bottleneck. This is analogous to Amdahl's law: if fraction $f$ of the
  work lies on the critical path, then the speedup with $P$ processors is
  bounded by $1/(f + (1-f)/P)$, which approaches $1/f$ as $P \to \infty$. For
  latency-critical applications, reducing the critical path fraction $f$ is
  essential.
\end{remark}

  \section {Execution Engine}

% ================================================================================
\subsection{Graph computation}
% TODO(ai): Merge the concepts docs/dataflow/all.computation_as_graphs.explanation.md

% --------------------------------------------------------------------------------
\subsubsection{DataFlow model}
A DataFlow model (aka \verb|DAG|) is a direct acyclic graph composed of DataFlow
nodes

It allows one to connect, query the structure, \ldots{}

Running a method on a DAG means running that method on all its nodes in
topological order, propagating values through the DAG nodes.

TODO(Paul, Samarth): Add picture.

% --------------------------------------------------------------------------------
\subsubsection{DagConfig}
A \verb|Dag| can be built by assembling Nodes using a function representing
the connectivity of the nodes and parameters contained in a \verb|Config| (e.g.,
through a call to a builder \verb|DagBuilder.get_dag(config)|).

A DagConfig is hierarchical and contains one subconfig per DAG node. It should
only include \verb|Dag| node configuration parameters, and not information
about \verb|Dag| connectivity, which is specified in the \verb|Dag| builder part.
A \verb|Dag| can be built by assembling Nodes using a function representing the
connectivity of the nodes and parameters contained in a \verb|Config| (e.g.,
through a call to a builder \verb|DagBuilder.get_dag(config)|).

A DagConfig is hierarchical and contains one subconfig per DAG node. It should
only include \verb|Dag| node configuration parameters, and not information about
\verb|Dag| connectivity, which is specified in the \verb|Dag| builder part.

% ================================================================================
\subsection{Graph execution}

% --------------------------------------------------------------------------------
\subsubsection{Simulation kernel}

A computation graph is a directed graph where nodes represent operations or variables,
and edges represent dependencies between these operations.

For example, in a computation graph for a mathematical expression, nodes would
represent operations like addition or multiplication, while edges would indicate
the order (and grouping) of operations.

The DataFlow simulation kernel schedules nodes according to their dependencies.

% --------------------------------------------------------------------------------
\subsection{Simulation kernel details}

The most general case of simulation consists of multiple nested loops:

\begin{enumerate}
  \item \textbf{Multiple DAG computation}. The general workload contains multiple
    DAG computations, each one inferred through a \verb|Config| belonging to a
    list of \verb|Config|s describing the entire workload to execute.
    \begin{itemize}
      \item In this set-up each DAG computation is independent, although some
        pieces of computations can be common across the workload. DataFlow will
        compute and then cache the common computations automatically as part of
        the framework execution
    \end{itemize}

  \item \textbf{Learning pattern}. For each DAG computation, multiple train/predict
    loops represent different machine learning patterns (e.g., in-sample vs out-of-sample,
    cross-validation, rolling window)
    \begin{itemize}
      \item This loop accommodates the need for nodes with state to be driven to
        learn parameters and hyperparameters and then use the learned state to
        predict on unseen data (i.e., out-of-sample)
    \end{itemize}

  \item \textbf{Temporal tiling}. Each DAG computation runs over a tile representing
    an interval of time
    \begin{itemize}
      \item As explained in section XYZ, DataFlow partition the time
        dimension in multiple tiles

      \item Temporal tiles might overlap to accommodate the amount of memory needed
        by each node (see XYZ), thus each timestamp will be covered by at
        least one tile. In the case of DAG nodes with no memory, then time is
        partitioned in non-overlapping tiles.

      \item The tiling pattern over time does not affect the result as long as
        the system is properly designed (see XYZ)
    \end{itemize}

  \item \textbf{Spatial tiling}. Each temporal slice can be computed in terms of
    multiple sections across the horizontal dimension of the dataframe inputs,
    as explained in section XYZ.

    \begin{itemize}
      \item This is constrained by nodes that compute features cross-sectionally,
        which require the entire space slice to be computed at once
    \end{itemize}

  \item \textbf{Single DAG computation}. Finally a topological sorting based on
    the specific DAG connectivity is performed in order to execute nodes in the
    proper order. Each node executes over temporal and spatial tiles.
\end{enumerate}

Figure~\ref{fig:nested_loops} illustrates the nested structure of these
simulation loops.

\begin{figure}[ht]
\centering
\begin{tikzpicture}[scale=0.9,
    box/.style={draw, thick, rounded corners, minimum height=0.8cm},
    label/.style={font=\small\bfseries, anchor=west}]

  % Innermost level - Single DAG
  \node[box, fill=blue!20, minimum width=4cm] (dag) at (0, 0)
    {Single DAG Computation};
  \node[label, right] at (2.5, 0)
    {Topological execution};

  % Spatial tiling
  \node[box, fill=green!15, minimum width=6cm, minimum height=1.8cm,
        fit={(dag)}] (spatial) {};
  \node[label] at (-3.2, 1.1) {Spatial Tiling};
  \node[font=\tiny, text width=4cm, align=left, anchor=west] at (2.5, 0.6)
    {Per-asset or column-group tiles};

  % Temporal tiling
  \node[box, fill=orange!15, minimum width=8cm, minimum height=3cm,
        fit={(spatial)}] (temporal) {};
  \node[label] at (-4.2, 1.8) {Temporal Tiling};
  \node[font=\tiny, text width=4cm, align=left, anchor=west] at (2.5, 1.4)
    {Time intervals (days, months)};

  % Learning pattern
  \node[box, fill=purple!15, minimum width=10cm, minimum height=4.2cm,
        fit={(temporal)}] (learning) {};
  \node[label] at (-5.2, 2.5) {Learning Pattern};
  \node[font=\tiny, text width=4cm, align=left, anchor=west] at (2.5, 2.2)
    {Train/predict, cross-validation};

  % Multiple DAG computations
  \node[box, fill=red!15, minimum width=12cm, minimum height=5.4cm,
        fit={(learning)}] (multiple) {};
  \node[label] at (-6.2, 3.2) {Multiple DAG Computations};
  \node[font=\tiny, text width=4cm, align=left, anchor=west] at (2.5, 3.0)
    {Parameter sweeps, experiments};

  % Add arrows showing nesting
  \foreach \i/\start/\end in {1/-3/-2.5, 2/-4/-3.5, 3/-5/-4.5, 4/-6/-5.5} {
    \draw[->, thick, gray] (\start, -2.5) -- (\end, -2.5);
  }
  \node[below, font=\small, gray] at (0, -2.8) {Increasing scope $\rightarrow$};

\end{tikzpicture}

\caption{Nested loop structure of the DataFlow simulation kernel. Each outer loop encompasses multiple executions of inner loops. The innermost level performs topological execution of a single DAG over a tile. Outer levels handle tiling, learning patterns, and parameter sweeps. This hierarchy enables caching (common computations across experiments) and parallelism (independent loops can run concurrently).}
\label{fig:nested_loops}
\end{figure}
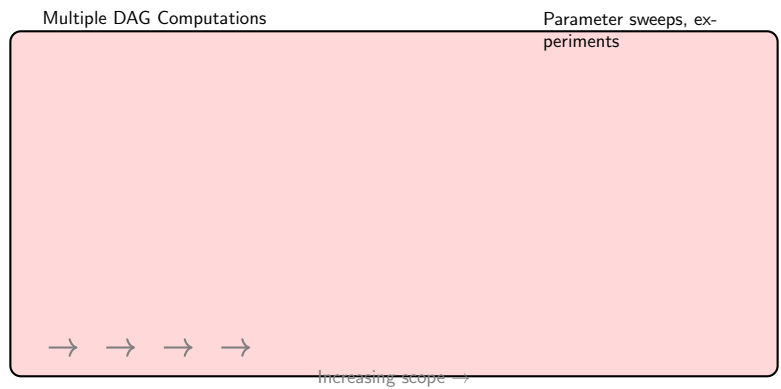

Note that it is possible to represent all the computations from the above
loops in a single ``scheduling graph'' and use this graph to schedule
executions in a global fashion.

Parallelization across CPUs comes naturally from the previous approach, since computations
that are independent in the scheduling graph can be executed in parallel, as
described in Section XYZ.

Incremental and cached computation is built-in in the scheduling algorithm
since it's possible to memoize the output by checking for a hash of all the
inputs and of the code in each node, as described in Section XYZ.

Even though each single DAG computation is required to have no loops, a System
(see XYZ) can have components introducing loops in the computation (e.g., a Portfolio
component in a trading system, where a DAG computes forecasts which are acted
upon based on the available funds). In this case, the simulation kernel needs to
enforce dependencies in the time dimension.

% --------------------------------------------------------------------------------
\subsection{Node ordering for execution}

TODO(gp, Paul): Extend this to the multiple loop.

Topological sorting is a linear ordering of the vertices of a directed graph such
that for every directed edge from vertex u to vertex v, u comes before v in
the ordering. This sorting is only possible if the graph has no directed cycles,
i.e., it must be a Directed Acyclic Graph (DAG).

\begin{lstlisting}[language=Python]
def topological_sort(graph):
  visited = set()
  post_order = []

  def dfs(node):
      if node in visited:
          return
      visited.add(node)
      for neighbor in graph.get(node, []):
          dfs(neighbor)
      post_order.append(node)

  for node in graph:
      dfs(node)

  return post_order[::-1]  # Reverse the post-order to get the topological order
\end{lstlisting}

% --------------------------------------------------------------------------------
\subsection{Heuristics for splitting computational steps into nodes}

There are degrees of freedom in splitting the work between various nodes of a
graph E.g., the same DataFlow computation can be described with several nodes
or with a single node containing all the code

The trade-off is often between several metrics:

\begin{itemize}
  \item Observability

    \begin{itemize}
      \item More nodes make it easier to:

        \begin{itemize}
          \item observe and debug intermediate the result of complex computation

          \item profile graph executions to understand performance bottlenecks
        \end{itemize}
    \end{itemize}

  \item latency/throughput

    \begin{itemize}
      \item More nodes:

        \begin{itemize}
          \item allows for better caching of computation

          \item allows for smaller incremental computation when only one part of
            the inputs change

          \item prevents optimizations performed across nodes

          \item incurs in more simulation kernel overhead for scheduling

          \item allows more parallelism between nodes being extracted and exploited
        \end{itemize}
    \end{itemize}

  \item memory consumption

    \begin{itemize}
      \item More nodes:

        \begin{itemize}
          \item allows one to partition the computation in smaller chunks requiring
            less working memory
        \end{itemize}
    \end{itemize}
\end{itemize}

A possible heuristic is to start with smaller nodes, where each node has a clear
function, and then merge nodes if this is shown to improve performance

\subsection{DataFlow System}

% --------------------------------------------------------------------------------
\subsubsection{Motivation}

While DataFlow requires that a DAG should not have cycles, general computing
systems might need to reuse the state from computation performed on past data.
E.g., in a trading system, there is often a Forecast component that can be
modeled as a DAG with no cycles and a Portfolio object that uses the forecasts
to compute the desired allocation of capital across different positions based
on the previous positions.

DataFlow supports this need by assembling multiple DAGs into a complete
\verb|System| that allows cycles.

The assumption is that DAGs are computationally expensive, while other components
mainly execute light procedural computation that requires interaction with external
objects such as databases, filesystems, sockets.

TODO(gp): Add picture

TODO(gp): Explain that System are derived from other Python objects.

% ================================================================================
\subsection{Timing semantic and clocks}

% --------------------------------------------------------------------------------
\subsubsection{Time semantics}

DataFlow components can execute in real-time or simulated mode, with different
approaches for representing the passage of time. The framework supports multiple
temporal execution modes designed to prevent future peeking while maintaining
consistency between simulation and production environments.

% --------------------------------------------------------------------------------
\subsubsection{Clock types}

The framework defines three distinct clock implementations:

\begin{enumerate}
  \item \textbf{Static clock}: A clock that remains constant during system execution.
    Future peeking is technically permissible with this clock type.

  \item \textbf{Replayed clock}: A clock that advances through historical time,
    either synchronized with wall-clock time or driven by computational events.
    The clock can be positioned in either past or future relative to actual time,
    but future peeking is prohibited to maintain simulation fidelity.

  \item \textbf{Real clock}: The wall-clock time, where data becomes available
    as generated by external systems. Future peeking is inherently impossible.
\end{enumerate}

% --------------------------------------------------------------------------------
\subsubsection{Knowledge time}

Knowledge time represents the timestamp when data becomes available to the system,
either through download or computation. Each data row is annotated with its
corresponding knowledge time. The framework enforces that data with knowledge time
exceeding the current clock time remains inaccessible, preventing inadvertent
future peeking.

% --------------------------------------------------------------------------------
\subsubsection{Timed and non-timed simulation}

\textbf{Timed simulation}. In timed simulation (also referred to as historical,
vectorized, or batch simulation), data is provided with an advancing clock that
reports the current timestamp. The system enforces that only data with knowledge
time less than or equal to the current timestamp is observable, thus preventing
future peeking. This mode typically employs either a replayed clock or static
clock depending on the specific use case.

\textbf{Non-timed simulation}. In non-timed simulation (also referred to as
event-based or reactive simulation), the clock type is static. The wall-clock
time corresponds to a timestamp equal to or greater than the latest knowledge
time in the dataset. Consequently, all data in the dataframe becomes immediately
available since each row has a knowledge time less than or equal to the wall-clock
time. In this mode, data for the entire period of interest is provided as a
single dataframe.

For example, consider a system generating predictions every 5 minutes. In non-timed
simulation, all input data are equally spaced on a 5-minute grid and indexed by
knowledge time:

\begin{lstlisting}[language=Python]
df["c"] = (df["a"] + df["b"]).shift(1)
\end{lstlisting}

% --------------------------------------------------------------------------------
\subsubsection{Real-time execution}

In real-time execution, the clock type is a real clock. For a system predicting
every 5 minutes, one forecast is generated every 5 minutes of wall-clock time,
with data arriving incrementally rather than in bulk.

% --------------------------------------------------------------------------------
\subsubsection{Replayed simulation}

In replayed simulation, data is provided in the same format and timing as in
real-time execution, but the clock type is a replayed clock. This allows the
system to simulate real-time behavior while processing historical data, facilitating
testing and validation of production systems.

% --------------------------------------------------------------------------------
\subsubsection{Synchronous and asynchronous execution modes}

\textbf{Asynchronous mode}. In asynchronous mode, multiple system components
execute concurrently. For example, the DAG may compute while orders are transmitted
to the market and other components await responses. The implementation utilizes
Python's \verb|asyncio| framework. While true asynchronous execution typically
requires multiple CPUs, under certain conditions (e.g., when I/O operations overlap
with computation), a single CPU can effectively simulate asynchronous behavior.

\textbf{Synchronous mode}. In synchronous mode, components execute sequentially.
For instance, the DAG completes its computation, then passes the resulting dataframe
to the Order Management System (OMS), which subsequently executes orders and
updates the portfolio.

The framework supports simulating the same system in either synchronous or
asynchronous mode. Synchronous execution follows a sequential pattern: the DAG
computes, passes data to the OMS, which then executes orders and updates the
portfolio. Asynchronous execution creates persistently active objects that
coordinate through mutual blocking mechanisms.

% ================================================================================
\subsection{Vectorization}
% ================================================================================

% --------------------------------------------------------------------------------
\subsubsection{Vectorization}
% --------------------------------------------------------------------------------
Vectorization is a technique for enhancing the performance of computations by
simultaneously processing multiple data elements with a single instruction, leveraging
the capabilities of modern processors (e.g., SIMD (Single Instruction,
Multiple Data) units).

% --------------------------------------------------------------------------------
\subsubsection{Vectorization in DataFlow}
% --------------------------------------------------------------------------------
Given the DataFlow format, where features are organized in a hierarchical structure,
DataFlow allows one to apply an operation to be applied across the cross-section
of a dataframe. In this way DataFlow exploits Pandas and NumPy data manipulation
and numerical computing capabilities, which are in turns built on top of low-level
libraries written in languages like C and Fortran. These languages provide efficient
implementations of vectorized operations, thus bypassing the slower execution
speed of Python loops.

% --------------------------------------------------------------------------------
\subsubsection{Example of vectorized node in DataFlow}

TODO

% ================================================================================
\subsection{Incremental, cached, and parallel execution}

% --------------------------------------------------------------------------------
\subsubsection{DataFlow and functional programming}

The DataFlow computation model shares many similarity with functional
programming:

\begin{itemize}
  \item Data immutability: data in dataframe columns is typically added or
    replaced. A node in a DataFlow graph cannot alter data in the nodes earlier
    in the graph.

  \item Pure functions: the output of a node depends only on its input values
    and it does not cause observable side effects, such as modifying a global state
    or changing the value of its inputs

  \item Lack of global state: nodes do not rely on data outside their scope, especially
    global state
\end{itemize}

% --------------------------------------------------------------------------------
\subsubsection{Incremental computation}

Only parts of a compute graph that see a change of inputs need to be
recomputed.

Incremental computation is an approach where the result of a computation is
updated in response to changes in its inputs, rather than recalculating everything
from scratch

% --------------------------------------------------------------------------------
\subsubsection{Caching}

Because of the "functional" style (no side effects) of data flow, the output
of a node is determinstic and function only of its inputs and code.

Thus the computation can be cached across runs. E.g., if many DAG simulations share
the first part of simulation, then that part will be automatically cached and
reused, without needing to be recomputed multiple times.

Figure~\ref{fig:caching_flowchart} shows the caching algorithm used by DataFlow.

\begin{figure}[ht]
\centering
\begin{tikzpicture}[scale=0.9,
    process/.style={rectangle, draw, thick, minimum width=3.5cm, minimum height=1cm,
                    text centered, rounded corners, fill=blue!10},
    decision/.style={diamond, draw, thick, aspect=2, minimum width=3cm,
                     text centered, fill=orange!20},
    io/.style={trapezium, trapezium left angle=70, trapezium right angle=110,
               draw, thick, minimum width=3cm, minimum height=0.8cm,
               text centered, fill=green!10}]

  % Start
  \node[io] (start) at (0, 0) {Node Execution Request};

  % Compute hash
  \node[process, below=1cm of start] (hash)
    {Compute Hash Key $H = h(\text{inputs}, \text{config}, \text{code})$};

  % Check cache
  \node[decision, below=1.2cm of hash] (check) {In Cache?};

  % Cache hit path
  \node[process, right=3cm of check] (hit) {Retrieve Cached Result};
  \node[io, below=0.8cm of hit] (return_hit) {Return Result};

  % Cache miss path
  \node[process, below=1.5cm of check] (miss) {Execute Node Computation};
  \node[process, below=0.8cm of miss] (store) {Store Result in Cache};
  \node[io, below=0.8cm of store] (return_miss) {Return Result};

  % Arrows
  \draw[->, thick] (start) -- (hash);
  \draw[->, thick] (hash) -- (check);
  \draw[->, thick] (check) -- node[above] {\small Yes} (hit);
  \draw[->, thick] (hit) -- (return_hit);
  \draw[->, thick] (check) -- node[right] {\small No} (miss);
  \draw[->, thick] (miss) -- (store);
  \draw[->, thick] (store) -- (return_miss);

  % Hash key components box
  \node[draw, thick, fill=yellow!10, text width=3.5cm, align=left] at (6, -1)
    {\textbf{Hash Key Components:} \\
     \small • Input data \\
     \small • Node configuration \\
     \small • Code version/hash};

\end{tikzpicture}

\caption{DataFlow caching algorithm. Before executing a node, the system computes a hash key from inputs, configuration, and code. If a matching result exists in the cache, it is returned immediately. Otherwise, the node executes and stores its result for future use. This enables automatic reuse of common computations across parameter sweeps and experiments.}
\label{fig:caching_flowchart}
\end{figure}
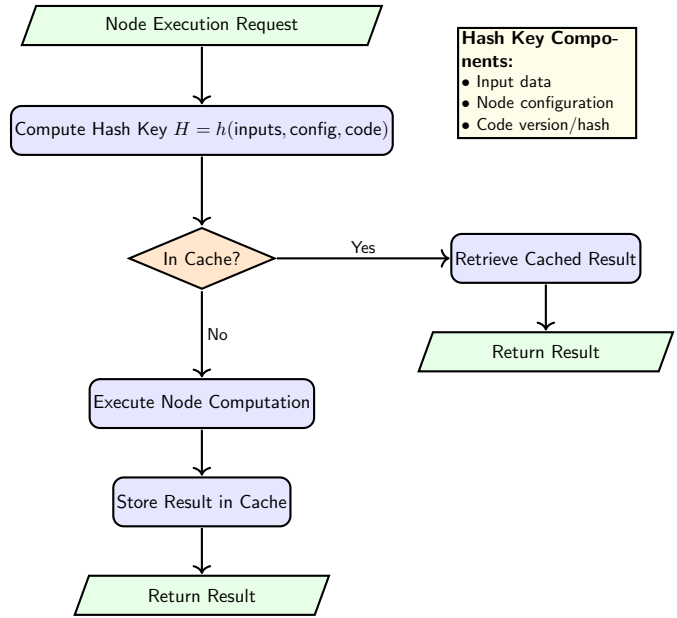

% --------------------------------------------------------------------------------
\subsubsection{Parallel execution}

Parallel and distributed execution in DataFlow is supported at two different
levels:
\begin{itemize}
  \item Across runs: given a list of \verb|Config|, each describing a different
    system, each simulation can be in parallel because they are completely
    independent.

  \item Intra runs: each DataFlow graph can be run exploiting the fact that
    nodes
\end{itemize}

In the current implementation for intra-run parallelism Kaizen flow relies on \verb|Dask|
For across-run parallelism DataFlow relies on \verb|joblib| or \verb|Dask|

Dask extends the capabilities of the Python ecosystem by providing an
efficient way to perform parallel and distributed computing.

Dask supports various forms of parallelism, including multi-threading, multi-processing,
and distributed computing. This allows it to leverage multiple cores and
machines for computation.

When working in a distributed environment, Dask distributes data and computation
across multiple nodes in a cluster, managing communication and synchronization
between nodes. It also provides resilience by re-computing lost data if a node
fails.

% ================================================================================
\subsection{Train and predict}

A DAG computation may undergo multiple evaluation phases (also referred to as
methods) to accommodate different experimental designs and validation strategies.

% --------------------------------------------------------------------------------
\subsubsection{Evaluation phases}

The framework supports several distinct phases:

\begin{itemize}
  \item \textbf{Initialization phase}: Performs computations necessary to establish
    an initial state, such as loading previously learned model parameters.

  \item \textbf{Fit phase}: Learns the state of stateful nodes using training data.

  \item \textbf{Validate phase}: Tunes hyperparameters of the system using a
    validation set. Examples include determining the optimal number of epochs
    or layers in a neural network, or the time constant of a smoothing parameter.

  \item \textbf{Predict phase}: Applies the learned state of each node to generate
    predictions on unseen data.

  \item \textbf{Load state}: Retrieves previously learned state of stateful DAG
    nodes, such as trained model weights.

  \item \textbf{Save state}: Persists the learned state of a DAG following a fit
    phase, enabling subsequent deployment to production environments.

  \item \textbf{Save results}: Stores artifacts generated during model execution,
    such as predictions from a predict phase.
\end{itemize}

The simulation kernel schedules these phases according to the type of simulation
and the dependency structure across DAG nodes. For instance, the initialization
phase can load previously learned DAG state, enabling a subsequent predict phase
without requiring a fit phase.

% --------------------------------------------------------------------------------
\subsubsection{Experimental designs}

\paragraph{In-sample evaluation}

In-sample evaluation tests the model on the same dataset used for training. While
this approach provides optimistic performance estimates, it serves as a useful
baseline. The process consists of:
\begin{enumerate}
  \item Feeding all data to the DAG in fit mode
  \item Learning parameters for each stateful node
  \item Running the DAG in predict mode on the training data
\end{enumerate}

\paragraph{Train/test evaluation}

Train/test evaluation (also known as in-sample/out-of-sample evaluation) partitions
the data into disjoint training and test sets:
\begin{enumerate}
  \item Split the data into training and test sets without temporal overlap
  \item Feed training data to the DAG in fit mode
  \item Learn parameters for each stateful node
  \item Run the DAG in predict mode on the test data
\end{enumerate}

\paragraph{Train/validate/test evaluation}

This extends the train/test approach by introducing a validation set for
hyperparameter tuning. The validation set enables selection of design parameters
such as network architecture or regularization strength before final evaluation
on the test set.

\paragraph{Cross-validation}

Cross-validation provides robust model evaluation by partitioning the dataset
into multiple subsets. For each partition:
\begin{enumerate}
  \item Use one subset as the test set and remaining data as training set
  \item Feed training data to the DAG in fit mode
  \item Learn parameters for each stateful node
  \item Run the DAG in predict mode on the test subset
\end{enumerate}
Aggregate performance across all subsets to assess overall model quality. For
time series data, this approach must respect temporal ordering to prevent future
peeking.

\paragraph{Rolling train/test evaluation}

Rolling evaluation is particularly suited for time series analysis. The approach
sequentially partitions the dataset such that each test set immediately follows
its corresponding training set in time:
\begin{enumerate}
  \item Partition the dataset into sequential train and test sets
  \item For each partition:
    \begin{itemize}
      \item Use earlier data as the training set
      \item Use immediately following data as the test set
      \item Feed training data to the DAG in fit mode
      \item Learn parameters for each stateful node
      \item Run the DAG in predict mode on the test data
    \end{itemize}
\end{enumerate}
This method simulates realistic scenarios where the model is trained on historical
data and tested on future observations, with the model continually updated as
new data becomes available.

% --------------------------------------------------------------------------------
\subsubsection{Stateful nodes}

A DAG node is stateful if it uses data to learn parameters (e.g., linear regression
coefficients, weights in a neural network, support vectors in a SVM) during
the \verb|fit| stage, that are then used in a successive \verb|predict| stage.

The state is stored inside the implementation of the node.

The state of stateful DAG node varies during a single simulation.

The following example demonstrates a stateful node implementation:

\begin{figure}[ht]
\centering
\begin{lstlisting}[language=Python, style=codebgstyle]
class MovingAverageNode(Node):
    """
    Stateful node that learns optimal window size during fit phase
    and applies it during predict phase.
    """

    def __init__(self, nid: str, window_range: tuple = (5, 50)):
        """
        Args:
            nid: Unique node identifier
            window_range: Range of window sizes to search (min, max)
        """
        super().__init__(nid)
        self.window_range = window_range
        self.optimal_window = None  # Learned state

    def fit(self, df: pd.DataFrame) -> pd.DataFrame:
        """
        Learn optimal window size from training data using
        cross-validation on a holdout metric.
        """
        best_score = float('inf')
        best_window = self.window_range[0]

        # Search for optimal window size
        for window in range(*self.window_range):
            ma = df['price'].rolling(window=window).mean()
            # Evaluate on some metric (e.g., forecast error)
            score = self._evaluate_window(df, ma)
            if score < best_score:
                best_score = score
                best_window = window

        # Store learned state
        self.optimal_window = best_window
        return self._compute_ma(df, best_window)

    def predict(self, df: pd.DataFrame) -> pd.DataFrame:
        """
        Apply learned window size to new data.
        State must be set before calling predict.
        """
        assert self.optimal_window is not None, \
            "Must call fit() before predict()"
        return self._compute_ma(df, self.optimal_window)

    def _compute_ma(self, df: pd.DataFrame, window: int) -> pd.DataFrame:
        """Compute moving average with given window."""
        df_out = df.copy()
        df_out['ma'] = df['price'].rolling(window=window).mean()
        return df_out

    def _evaluate_window(self, df: pd.DataFrame,
                        ma: pd.Series) -> float:
        """Evaluate window size quality (implementation detail)."""
        # Example: mean squared error on next-step prediction
        return ((df['price'].shift(-1) - ma) ** 2).mean()
\end{lstlisting}
\caption{Example of a stateful DataFlow node. The node learns an optimal window size during the \texttt{fit()} phase by cross-validation, stores it as internal state (\texttt{optimal\_window}), and applies the learned parameter during \texttt{predict()}. This separation enables proper in-sample/out-of-sample evaluation.}
\label{fig:stateful_node_code}
\end{figure}

% --------------------------------------------------------------------------------
\subsubsection{Stateless nodes}
A DAG node is stateless if the output is not dependent on previous \verb|fit|
stages. In other words the output of the node is only function of the current inputs
and of the node code, but not from inputs from previous tiles of inputs.

A stateless DAG node emits the same output independently from the current and
previous \verb|fit| vs \verb|predict| phases.

A stateless DAG node has no state that needs to be stored across a simulation.

% --------------------------------------------------------------------------------
\subsubsection{Loading and saving node state}

Each stateful node provides mechanisms for persisting and retrieving its internal
state on demand. The framework orchestrates the serialization and deserialization
of entire DAG states to disk.

A stateless node returns an empty state when saving and raises an assertion error
if presented with a non-empty state during loading.

The framework enables loading DAG state for subsequent analysis. For instance,
one might examine how linear model weights evolve over time in a rolling window
simulation.

% --------------------------------------------------------------------------------
\subsubsection{Batch computation and tiled execution}

DataFlow supports batch computation through tiled execution, which partitions
computation across temporal and spatial dimensions. Tiled execution provides
several advantages:

\begin{itemize}
  \item \textbf{Memory efficiency}: Large-scale simulations that exceed available
    memory can be executed by processing data in manageable tiles.

  \item \textbf{Incremental computation}: Results can be computed progressively
    and cached, avoiding redundant calculations.

  \item \textbf{Parallelization}: Independent tiles can be processed concurrently
    across multiple compute resources.
\end{itemize}

\paragraph{Temporal tiling}

Temporal tiling partitions the time dimension into intervals. Each tile represents
a specific time period (e.g., a single day or month). Tiles may overlap to accommodate
node memory requirements. For nodes without memory dependencies, time is partitioned
into non-overlapping intervals.

\paragraph{Spatial tiling}

Within each temporal slice, computation may be further divided across the horizontal
dimension of dataframes. This approach is constrained by nodes that compute
cross-sectional features, which require simultaneous access to the entire spatial
slice.

\paragraph{DAG runner implementations}

Different \verb|DagRunner| implementations support various execution patterns:

\begin{itemize}
  \item \verb|FitPredictDagRunner|: Implements separate fit and predict phases
    for in-sample and out-of-sample evaluation.

  \item \verb|RollingFitPredictDagRunner|: Supports rolling window evaluation
    with periodic retraining.

  \item \verb|RealTimeDagRunner|: Executes nodes with real-time semantics, processing
    data as it arrives.
\end{itemize}

% ================================================================================
\subsection{Backtesting and model evaluation}

Backtesting provides a framework for evaluating model performance on historical
data, supporting various levels of abstraction and fidelity to production
environments.

% --------------------------------------------------------------------------------
\subsubsection{Backtest execution modes}

A backtest consists of code configured by a single \verb|Config| object. The
framework supports multiple execution modes:

\begin{itemize}
  \item \textbf{Batch mode}: All data is available from the outset and processed
    in bulk, either as a single operation or partitioned into tiles. No clock
    advancement occurs during execution.

  \item \textbf{Streaming mode}: Data becomes available incrementally as a clock
    advances, simulating real-time operation. This mode is equivalent to processing
    tiles with temporal span matching the data arrival frequency.
\end{itemize}

% --------------------------------------------------------------------------------
\subsubsection{Research flow}

The research flow provides rapid model evaluation without portfolio management
complexity. This mode excludes position tracking, order submission, and exchange
interaction. Consequently, transaction costs and market microstructure effects
are not reflected in performance metrics. The research flow proves valuable for
assessing predictive power and conducting preliminary model comparison.

% --------------------------------------------------------------------------------
\subsubsection{Tiled backtesting}

Tiled backtesting extends the basic backtest concept by partitioning execution
across multiple dimensions:

\begin{itemize}
  \item \textbf{Asset dimension}: Each tile processes a subset of assets, potentially
    a single instrument, over the entire time period.

  \item \textbf{Temporal dimension}: Each tile processes all assets over a specific
    time interval (e.g., one day or month), closely resembling real-time system
    operation.

  \item \textbf{Hybrid tiling}: Arbitrary partitioning across both dimensions to
    optimize memory usage and computational efficiency.
\end{itemize}

The framework represents each tile as a \verb|Config| object. Source nodes support
tiling through Parquet and database backends, computation nodes handle tiling
naturally through DataFlow's streaming architecture, and sink nodes write results
using Hive-partitioned Parquet format.

% --------------------------------------------------------------------------------
\subsubsection{Configuration and reproducibility}

The framework employs hierarchical \verb|Config| objects to ensure reproducibility:

\begin{itemize}
  \item \textbf{DagConfig}: Contains node-specific parameters, excluding connectivity
    information which is specified in the \verb|DagBuilder|.

  \item \textbf{SystemConfig}: Encompasses the entire system specification,
    including market data configuration, execution parameters, and DAG configuration.

  \item \textbf{BacktestConfig}: Defines temporal boundaries, universe selection,
    trading frequency, and data lookback requirements.
\end{itemize}

Each configuration is serialized and stored alongside results, enabling precise
reproduction of experiments.

% ================================================================================
\subsection{Observability and debuggability}

% --------------------------------------------------------------------------------
\subsubsection{Running a DAG partially}
DataFlow allows one to run nodes and DAGs in a notebook during design, analysis,
and debugging phases, and in a Python script during simulation and production phases.

It is possible to run a DAG up to a certain node to iterate on its design and debug.

TODO: Add example

% --------------------------------------------------------------------------------
\subsubsection{Replaying a DAG}

Each DAG node can:

\begin{itemize}
  \item capture the stream of data presented to it during either a simulation and
    real-time execution

  \item serialize the inputs and the outputs, together with the knowledge
    timestamps

  \item play back the outputs
\end{itemize}

DataFlow allows one to describe a cut in a DAG and capture the inputs and outputs
at that interface. In this way it is possible to debug a DAG replacing all the
components before a given cut with a synthetic one replaying the observed
behavior together with the exact timing in terms of knowledge timestamps.

This allows one to easily:

\begin{itemize}
  \item capture failures in production and replay them in simulation for debugging

  \item write unit tests using observed data traces
\end{itemize}

DataFlow allows each node to automatically save all the inputs and outputs to
disk to allow replay and analysis of the behavior with high fidelity.

%% ================================================================================
%\subsection{Profiling DataFlow execution}
% TODO(Grisha)

%- Information about each node execution time, memory footprint, basic stats
%about inputs and outputs
%- Notebooks that allow you to compute some stats
%- Maybe add a box plot to give an idea

% ================================================================================
%\subsection{DataFlow and the Python data stack}
%- Describe how other libraries can be integrated in DataFlow
%- E.g., sklearn, gluonts, statsmodel, Pandas, numpy
%
%- Show how a sklearn node can be used

  %\section{Application of DataFlow to quant finance}
  % Specifics of how we use DataFlow for finance

  % ###############################################################################
\section{Comparison to Related Work}

The landscape of distributed data processing systems encompasses streaming
frameworks, batch processing engines, and machine learning platforms. While
these systems address various aspects of large-scale computation, none provide
comprehensive support for time-aware machine learning with formal guarantees of
causality, reproducibility, and unified batch-streaming semantics. This section
examines existing frameworks and positions DataFlow's contributions relative to
the state of the art.

% =============================================================================
\subsection{Distributed Stream Processing Systems}

\textbf{Apache Flink}~\cite{CaEwHaKaMaTz15} is a distributed stream processing
engine designed for stateful computations over unbounded data streams. The
system provides event-time processing semantics, sophisticated windowing
operators, and fault tolerance through distributed snapshots. Flink excels at
high-throughput real-time analytics but does not provide integrated support for
machine learning model development, historical simulation capabilities, or
formal enforcement of point-in-time causality constraints required for
financial time series applications.

\textbf{Apache Spark Structured Streaming}~\cite{ShMoAlNa19,ZhChFrShSt10}
extends the Apache Spark batch processing engine with streaming capabilities
through a micro-batch execution model. The system treats streaming data as
unbounded tables and supports time-based windowing operations. However, the
micro-batch semantics introduce latency constraints, and the framework provides
limited support for continuous time series modeling. Critically, Spark lacks
built-in mechanisms to prevent future-peeking errors or validate
context-window dependencies in machine learning pipelines.

\textbf{Google Dataflow} and \textbf{Apache Beam}~\cite{Aketal15} provide a
unified programming model for batch and streaming data processing. The Dataflow
model introduces the concepts of event time, processing time, and windowing,
along with triggers for controlling output emission. While this model addresses
many challenges in distributed stream processing, it does not provide native
abstractions for time series machine learning, model training and evaluation
workflows, or formal validation of causal correctness.

% =============================================================================
\subsection{Machine Learning and Feature Engineering Systems}

\textbf{Metaflow}, developed by Netflix, is an ML workflow orchestration
framework focused on experiment tracking, model versioning, and deployment
infrastructure. The system provides DAG-based pipeline definitions and
integrates with cloud computing resources. However, Metaflow does not provide
native support for streaming data processing, temporal context window
management, or real-time model serving, limiting its applicability to
time series machine learning applications.

\textbf{Kaskada} is a time series feature engine that introduces a
domain-specific query language for temporal feature definitions. The system
provides point-in-time correctness guarantees and compositional feature
construction. While Kaskada addresses temporal semantics, it focuses primarily
on feature computation rather than end-to-end model training and deployment. It
lacks integration with general-purpose machine learning libraries and does not
support flexible DAG-based execution with arbitrary user-defined functions.

\textbf{Tecton} and \textbf{Feast} are feature store systems designed to
maintain consistency between offline training and online serving environments.
These systems support point-in-time-correct feature joins and materialization
of precomputed features. However, feature stores address only the data layer of
ML systems; they do not provide model execution engines, training orchestration,
or comprehensive pipeline observability. Feature stores complement but do not
replace systems like DataFlow that provide end-to-end computation frameworks.

\textbf{River} is a Python library for online machine learning that provides
incremental learning algorithms designed for streaming data. The library offers
implementations of various online learning methods but does not provide a
distributed execution engine, DAG-based orchestration, or formal guarantees of
point-in-time idempotency. River targets single-machine online learning rather
than distributed time series model development and backtesting.

% =============================================================================
\subsection{Legacy Distributed Computing Systems}

\textbf{Apache Storm} is a distributed real-time computation system designed
for processing unbounded streams of data. Storm provides at-least-once or
exactly-once processing guarantees through its Trident API. However, the system
lacks native support for event-time processing, does not enforce temporal
correctness constraints, and requires substantial engineering effort to
implement machine learning pipelines with proper causality guarantees.

\textbf{MapReduce}~\cite{DeGh08} is a programming model for distributed batch
processing that influenced subsequent big data systems. While MapReduce excels
at embarrassingly parallel computations, the model does not provide abstractions
for streaming data, temporal windowing, or stateful computations. Implementing
time series machine learning with sliding context windows in MapReduce would
require significant data duplication and careful manual management of temporal
dependencies.

% =============================================================================
\subsection{DataFlow's Distinguishing Contributions}

DataFlow distinguishes itself from existing systems through the integration of
several key capabilities:
\begin{enumerate}
  \item \textbf{Point-in-time idempotency} as a formal correctness criterion,
    with automatic validation through tiling tests.
  \item \textbf{Unified batch-streaming semantics} that guarantee identical
    outputs in research backtesting and production deployment.
  \item \textbf{Native dataframe-based programming model} supporting Python's
    scientific computing ecosystem (pandas, NumPy, scikit-learn).
  \item \textbf{Context window enforcement} preventing future-peeking errors
    and enabling memory-bounded computation.
  \item \textbf{Deterministic replay} capabilities for production debugging and
    systematic testing.
  \item \textbf{Automatic caching and incremental computation} based on
    explicit DAG dependencies.
\end{enumerate}

While systems like Apache Flink and Apache Beam provide sophisticated stream
processing primitives, they do not integrate these capabilities with machine
learning model development and validation workflows. Conversely, systems like
Metaflow and feature stores address ML engineering concerns but lack native
support for streaming execution and temporal correctness validation. DataFlow
is specifically designed to support the complete time series machine learning
lifecycle—from exploratory research to production deployment—with formal
correctness guarantees, observability, and reproducibility as first-class design
principles.

\begin{table}[ht]
\caption{Comparison of streaming and ML frameworks}
\centering
\small
\begin{tabular}{|l|c|c|c|c|c|c|c|}
\hline
\textbf{Framework} & \textbf{Stream} & \textbf{Batch} &
\shortstack{\textbf{ML}\\\textbf{Integration}} &
\shortstack{\textbf{Time}\\\textbf{Semantics}} &
\shortstack{\textbf{Causal}\\\textbf{Safety}} &
\textbf{Backtesting} &
\shortstack{\textbf{DAG}\\\textbf{Support}} \\
\hline
\textbf{DataFlow}        & \cmark & \cmark & \cmark   & \Strong   & \cmark & \cmark   & \cmark \\
\textbf{Apache Flink}    & \cmark & \cmark & \Limited & \Strong   & \xmark & \xmark   & \cmark \\
\textbf{Apache Spark}    & \cmark & \cmark & \Partial & \Moderate & \xmark & \Partial & \cmark \\
\textbf{Google Dataflow} & \cmark & \cmark & \Limited & \Strong   & \xmark & \xmark   & \cmark \\
\textbf{Metaflow}        & \xmark & \cmark & \cmark   & \Weak     & \xmark & \Partial & \cmark \\
\textbf{Kaskada}         & \cmark & \cmark & \Limited & \Strong   & \cmark & \Partial & \xmark \\
\textbf{Tecton/Feast}    & \cmark & \cmark & \Partial & \Strong   & \cmark & \xmark   & \xmark \\
\textbf{River}           & \cmark & \xmark & \cmark   & \Weak     & \xmark & \xmark   & \xmark \\
\textbf{Apache Storm}    & \cmark & \xmark & \Limited & \Weak     & \xmark & \xmark   & \cmark \\
\textbf{MapReduce}       & \xmark & \cmark & \Limited & \None     & \xmark & \xmark   & \xmark \\
\hline
\end{tabular}
\label{table:framework_comparison}
\end{table}

% ================================================================================
\subsection{Additional Related Systems}

Several additional systems address specific aspects of distributed stream
processing and machine learning but remain complementary to rather than
competitive with DataFlow's comprehensive approach.

\textbf{Apache Kafka Streams} is a client library for building stream
processing applications on top of Apache Kafka. The library provides stateful
stream processing with exactly-once semantics but focuses on Kafka-centric
architectures and does not provide integrated support for machine learning
workflows or point-in-time idempotency guarantees.

\textbf{Apache Samza} is a distributed stream processing framework originally
developed at LinkedIn. Samza integrates tightly with Apache Kafka for messaging
and provides stateful stream processing capabilities. However, like Kafka
Streams, it does not address the specific requirements of time series machine
learning or provide formal temporal correctness guarantees.

\textbf{Dask} is a parallel computing library in Python that extends NumPy,
pandas, and scikit-learn to distributed environments. While Dask supports
dataframe operations and integrates well with the Python scientific computing
ecosystem, it does not provide native streaming semantics, event-time
processing, or the temporal correctness guarantees required for time series
machine learning applications. Dask is well-suited for batch processing of
large datasets but does not address real-time streaming or causality validation.

\textbf{MillWheel}~\cite{Aketal13} is Google's internal framework for
fault-tolerant stream processing that inspired Apache Beam and Google Cloud
Dataflow. The system introduced important concepts including persistent state
and exactly-once processing semantics. However, MillWheel is an internal Google
system and does not provide the ML-focused abstractions or Python-native
programming model offered by DataFlow.

A comprehensive treatment of distributed data systems architecture and design
patterns is provided by Kleppmann~\cite{Kleppmann17}.

  % ###############################################################################
  \bibliography{references}
  \bibliographystyle{amsplain}
\end{document}